\documentclass[12pt]{article}

\usepackage[utf8]{inputenc} 
\usepackage[T1]{fontenc}    
\usepackage{hyperref}       
\usepackage{url}            
\usepackage{booktabs}       
\usepackage{amsfonts}       
\usepackage{nicefrac}       
\usepackage{microtype}      

\usepackage{amsmath}
\usepackage{amsthm}
\usepackage{amssymb}
\usepackage{algorithm}
\usepackage{algorithmic}
\usepackage{graphicx}
\graphicspath{{/Users/tangch/Documents/Oct_Nov_2016/covtype_plots/}{/Users/tangch/Documents/Oct_Nov_2016/mnist_plots/}}
\usepackage{subcaption}
\usepackage{rotating}
\newtheorem{thm}{Theorem}
\newtheorem{lm}{Lemma}
\newtheorem{claim}{Claim}
\newtheorem{defn}{Definition}

\newtheorem{prop}{Proposition}

\begin{document}
\title{Convergence rate of stochastic $k$-means}

\author{ Cheng Tang\thanks{Department of Computer Science, George Washington University}\\ 
                       \texttt{tangch@gwu.edu} 
                        \and Claire Monteleoni\footnotemark[1]\\
                        \texttt{cmontel@gwu.edu}
                        }

 \maketitle

%

%

\begin{abstract}
We analyze online \cite{BottouBengio} and mini-batch \cite{Sculley} $k$-means variants. Both scale up the widely used Lloyd's algorithm via stochastic approximation, and have become popular for large-scale clustering and unsupervised feature learning.
We show, for the first time, that they have global convergence towards ``local optima'' at rate $O(\frac{1}{t})$ under general conditions.  
In addition, we show if the dataset is clusterable, with suitable initialization, mini-batch $k$-means converges to an optimal $k$-means solution at rate $O(\frac{1}{t})$ with high probability. 
The $k$-means objective is non-convex and non-differentiable: we exploit ideas from non-convex gradient-based optimization by providing a novel characterization of the trajectory of $k$-means algorithm on its solution space, and circumvent its non-differentiability via geometric insights about $k$-means update.
\end{abstract}

\section{Introduction}
\label{sec:intro}
Stochastic $k$-means, including online \cite{BottouBengio} and mini-batch $k$-means \cite{Sculley},
has gained increasing attention for large-scale clustering and is included in widely used machine learning packages, such as \texttt{Sofia-ML} \cite{Sculley} and \texttt{scikit-learn} \cite{scikit-learn}.
Figure \ref{fig:demo} demonstrates the efficiency of stochastic $k$-means against batch $k$-means on the \texttt{RCV1} dataset \cite{data:rcv1}. The advantage is clear, and the results raise some natural questions:
Can we characterize the convergence rate of stochastic $k$-means?
Why do the algorithms appear to converge to different ``local optima''?
Why and how does mini-batch size affect the quality of the final solution?
Our goal is to address these questions rigorously.

Given a discrete dataset of $n$ input data points, denoted by $X:=\{x, x\in \mathbb{R}^d\}$, a common way to cluster $X$ is to first select a set of $k$ centroids, denoted by
$C = \{c_r\in \mathbb{R}^d, r\in [k]\}$, and assign each $x$ to a centroid. 
Data points assigned to the same centroid $c_r$, form
a cluster $A_r\subset X$, and $\cup_{r\in [k]}A_r=X$; we let $A:=\{A_r, r\in [k]\}$ denote the resulting clustering. 
This is center-based clustering, where each clustering is specified by a tuple $(C,A)$.
The $k$-means cost of $(C,A)$ is defined as:
$\phi_{X}(C,A):=\sum_{r\in [k]}\sum_{x\in A_r} \|x-c_r\|^2$.
The $k$-means clustering problem is cast as the optimization problem:
$
\min_{C,A} \phi_{X}(C,A)
$.
This is an NP-hard problem \cite{NPhard}.

However, if either $C$ or $A$ is fixed, the problem can be easily solved.
Fixing the set of $k$ centroids $C$, the smallest $k$-means cost is achieved by choosing the clustering that assigns each point $x$ to it closest center, which we denote by $C(x):=\min_{c_r\in C}\|x-c_r\|$. That is,
\vspace{-0.15cm}
\begin{eqnarray}\label{kmobj1}
\phi_{X}(C):= \min_{A} \phi_{X}(C,A) = \sum_{r\in [k]}\sum_{C(x)=c_r }\|x-c_r\|^2
\end{eqnarray}
This clustering can also be induced by the Voronoi diagram of $C$, denoted by $V(C):=\{V(c_r), r\in [k]\}$, where
$$
V(c_r):=\{x\in \mathbb{R}^d, \|x-c_r\|\le\|x-c_s\|,\forall s\ne r\}
$$
Clustering $A$ induced by $V(C)$ is such that $\forall A_r\in A, A_r = V(c_r)\cap X$. Subsequently, we will use $V(C)\cap X$ to denote this induced clustering.
Likewise, fixing a $k$-clustering $A$ of $X$, the smallest $k$-means cost is achieved by setting the new centers as the mean of each cluster, denoted by $m(A_r)$,
\vspace{-0.2cm}
\begin{eqnarray}\label{kmobj2}
\phi_{X}(A):=\min_{C} \phi_{X}(C,A)= \sum_{r\in [k]}\sum_{x\in A_r}\|x-m(A_r)\|^2
\end{eqnarray}
Batch $k$-means or Lloyd's algorithm \cite{lloyd}, one of the most popular clustering heuristics \cite{jain:kmeans},
essentially proceeds by finding the solution to \eqref{kmobj1} and \eqref{kmobj2} alternatively:
At $t=0$,
it initializes the position of $k$ centroids, $C^0$, via a seeding algorithm. 
$\forall t\ge 1$, it alternates between two steps,
\begin{description}
\item[Step 1]
\vspace{-0.2cm}
Fix $C^{t-1}$,
find $A^t$ such that 
$$ A^t = \arg\min_{A}\phi_X(C^{t-1}, A)=V(C^{t-1})\cap X$$
\item[Step 2]
\vspace{-0.2cm}
Fix $A^t$, 
find $C^{t}$ such that 
$$
C^t = \arg\min_{C}\phi_X(C, A^t)=m(A^t)
$$
\vspace{-0.4cm}
\end{description}
\vspace{-0.6cm}
However, Step 1 requires computation of the closest centroid to every point in the dataset. Even with fast implementations such as \cite{elkan:kmeans_fast}, which reduces the computation for finding the closest centroid of each point, the per-iteration running time is still $O(n)$, making it a computational bottleneck for large datasets.  

To scale up batch $k$-means, ``stochastic approximation'' was proposed \cite{BottouBengio, Sculley}, which we present as Algorithm \ref{alg:MBKM}.
\footnote{In Claim \ref{claim:equivAlg} of the Appendix, we formally show Algorithm \ref{alg:MBKM} subsumes both online and mini-batch $k$-means.} The main idea is, at each iteration, the centroids are updated using one (online \cite{BottouBengio}) or a few (mini-batch \cite{Sculley}) randomly sampled points, denoted by $S^t$, instead of the entire dataset $X$.
At every iteration, Algorithm \ref{alg:MBKM} first approximates Step 1 and 2 using a random sample $S^t$ of constant size,
then it computes $C^t$ by interpolating between $C^{t-1}$ and $\hat C^t$.
Thus, stochastic $k$-means never directly clusters $X$ but keeps updating its set of centroids using constant sized random samples $S^t$, so the per-iteration computation cost is reduced from $O(n)$ in the batch case to $O(1)$.
\paragraph{Notation:}
In the paper, superscripts index a particular clustering, e.g., $A^t$ denotes the clustering at the $t$-th iteration in Algorithm \ref{alg:MBKM} (or batch $k$-means);
subscripts index individual clusters or centroids: $c_r$ denotes the $r$-th centroid in $C$ corresponding to the $r$-th cluster $A_r$.
We use letter $n$ to denote cardinality, $n = |X|$, $n_r=|A_r|$, etc.
$conv(X)$ denotes the convex hull of set $X$.
We let $\phi(C,A)$ denote the $k$-means cost of $(C, A)$;
we let $\phi(C)$, $\phi(A)$ denote the induced (optimal) $k$-means cost of a fixed $C$ or $A$;
as a shorthand, we often move the superscript (subscript) on the input of $\phi(\cdot)$ to $\phi$, e.g., 
we use $\phi^t$ to denote $\phi(C^t)$, and $\phi^t_r$ to denote the cost of the $r$-th cluster at $t$.
We denote the largest $k$-means cost on $X$ as $\phi_{\max}$ and the smallest $k$-means cost as $\phi_{opt}$.
Finally, we let $\pi(\cdot)$ to denote a permutation on $[k]$.
\begin{figure}
\centering
\includegraphics[height = \linewidth, width = 0.5\textwidth, angle = -90]{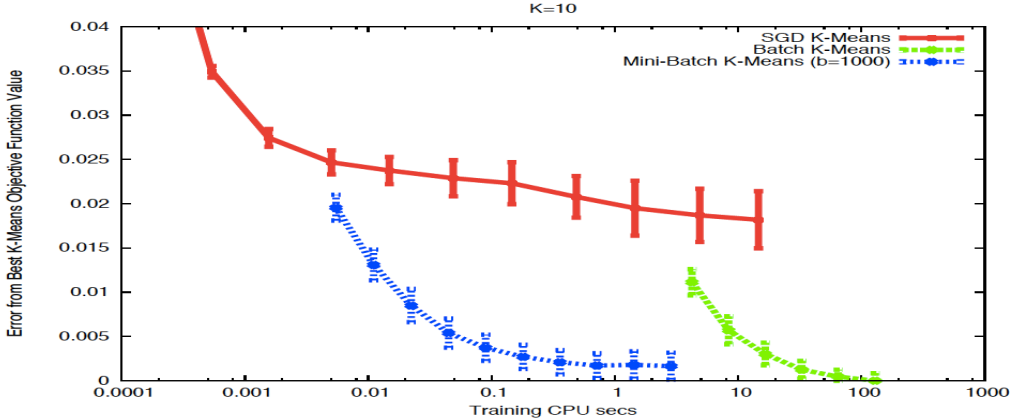}
\caption{Figure from \cite{Sculley}, demonstrating the relative performance
of online, mini-batch, and batch $k$-means.
}
\label{fig:demo}
\end{figure}
\section{Related work and our contributions}
Our work builds on analysis of batch $k$-means, which is notoriously difficult \cite{dasguptaOpen}. 
A great breakthrough was made recently by Kumar and Kannan \cite{kumar}, where they showed the \textit{$k$-SVD + constant $k$-means approximation + $k$-means update} scheme efficiently and correctly clusters most data points on well-clusterable instances, and that the algorithm converges to an approximately optimal solution at geometric rate until reaching a plateau.
Subsequent progress has been made on relaxing the assumptions \cite{awasthi:improved} and simplifying the seeding \cite{tang_montel:aistats16}.
However, these works only study local convergence of batch $k$-means in the sense that they require that the initial centroids are already close to the optimal solution.
\vspace{-0.31cm}
\paragraph{Alternating Minimization and EM}
Temporarily abusing notation, we let $X\in \mathbb{R}^{d\times n}$ be a matrix whose columns represent data points; similarly, let columns of $C\in\mathbb{R}^{d\times k}$ represent centroids, and let columns of $A\in\mathbb{R}^{k\times n}$ indicate cluster membership, i.e., $A_{ij}=1$ if and only if $x_i$ belongs to cluster $j$. The $k$-means problem can be equivalently formulated as 
\begin{eqnarray*}
\min_{C,A}\|X - CA\|_F^2
\mbox{~subject to~} A_{ij}\in \{0,1\}, \|A_i\|_2=1
\end{eqnarray*}
This is a non-convex low-rank matrix factorization problem.
In particular, with the discrete constraint on $A$, the $k$-means problem is close to the sparse coding problem (also known as dictionary learning or basis pursuit), where $C$ corresponds to a set of dictionary items and $A$ the coding matrix, and batch $k$-means can be viewed as a variant of Alternating Minimization (AM) for sparsity-constrained matrix factorization. 
In statistics, the algorithm is also well known to relate to EM for gaussians mixture models.
Both AM and EM are popular heuristics for non-convex unsupervised learning problems, ranging from matrix completion to latent variable models, where growing interest and exciting progress towards understanding their convergence behavior are emerging \cite{JNS:AM, hardt:AM, arora:sparsecode, balakrishnan:EM}.

However, existing analyses of AM or EM do not apply to $k$-means.
For one thing, the prevalent assumptions for AM-style algorithms is \textit{incoherence+sparsity}, while EM applies to generative models. Both are different from our deterministic assumption on the geometry of the dataset.
For another, the \cite{JNS:AM, hardt:AM, arora:sparsecode, balakrishnan:EM} all rely on closed form update expression; usually partial gradients are easily derivable in these problems.
In contrast, we cannot obtain a closed form update rule for Step 1 of batch $k$-means, due to the discrete constraint on $A$.
Instead, we work around this problem using geometric insights about $k$-means update developed from the clustering literature \cite{kumar, tang_montel:aistats16}.
\vspace{-0.38cm}
\paragraph{Non-convex stochastic optimization}
For convex problems, stochastic optimization methods are well studied \cite{rakhlin}. Much less can be said about non-convex problems. 
Our analysis of stochastic $k$-means is influenced by two recent works in non-convex stochastic optimization \cite{ge:sgd_tensor, balsubramani13}.
The first \cite{ge:sgd_tensor} studies the convergence of stochastic gradient descent (SGD) for 
tensor decomposition problems, which amounts to finding a local optimum of a non-convex objective function with only saddle points and local optima. 
Inspired by their analysis framework, we divide our analysis of Algorithm \ref{alg:MBKM} into two phases, those of global and local convergence, according to the distance from the current solution to the set of all ``local optima''.
At the local convergence phase, since multiple local optima are present, stochastic noise may drive the algorithm's iterate off the neighborhood of attraction, and the algorithm may fail to converge locally. 
To deal with this, we adapted techniques that bound martingale large deviation from \cite{balsubramani13} to our local convergence analysis; the work in \cite{balsubramani13} studies the convergence of stochastic PCA algorithms, where the objective function is the non-convex Rayleigh quotient.

\paragraph{Our contributions}
Our contributions are two-fold. 
For users of stochastic $k$-means, Theorem \ref{thm:km} guarantees that it converges to a local optimum with any reasonable seeding (it only requires the seeds be in the convex hull of the dataset) and a properly chosen learning rate, with $O(\frac{1}{t})$ expected convergence rate, where the convergence is with respect to $k$-means objective.
In contrast to recent batch $k$-means analysis \cite{kumar, awasthi:improved, tang_montel:aistats16}, it establishes a global convergence result for stochastic $k$-means, since it applies to almost any initialization $C^0$; it also applies to a wide range of datasets, without requiring a strong clusterability assumption. 

Theoretically, we have three major contributions.
First, our analysis provides a novel analysis framework for $k$-means algorithms, by connecting the discrete optimization approach to that of gradient-based continuous optimization.
Under this framework, we identify a ``Lipschitz'' condition of a local optimum, under which stochastic $k$-means converges locally. This approach to establish local convergence is similar in-spirit to the local convergence analysis of \cite{arora:sparsecode, balakrishnan:EM}.
Second, we show this ``Lipschitz'' condition relates to clusterability assumptions of the dataset.
Consequently, Theorem \ref{thm:solution} 
shows, just as batch $k$-means, stochastic $k$-means can also converge to an optimal $k$-means solution, under a clusterability assumption similar to \cite{kumar}, with a scalable seeding algorithm developed in \cite{tang_montel:aistats16}.
This result 
extends the batch $k$-means results on well-clusterable instances \cite{kumar, awasthi:improved, tang_montel:aistats16} to stochastic $k$-means, and shows the two are equally powerful at finding an optimal $k$-means solution with strong enough clusterability assumption. 
Finally, the martingale technique, which we modified from \cite{balsubramani13}, can be applied to future analysis of non-convex stochastic optimization problems.
\begin{algorithm}[t]
   \caption{Stochastic $k$-means}
   \label{alg:MBKM}
\begin{algorithmic}
   \STATE {\bfseries Input:} dataset $X$, number of clusters $k$, mini-batch size $m$, learning rate $\eta_r^t, r\in [k]$,
     \texttt{convergence\_criterion}
   \STATE{\bfseries Seeding:} Apply seeding algorithm $\mathcal{T}$ on $X$ and obtain seeds $C^0=\{c_1^0,\dots,c_k^0\}$; 
   \REPEAT
   \STATE At iteration $t$ ($t\ge 1$), obtain sample $S^t\subset X$ of size $m$ uniformly at random with replacement;
     set count $\hat{n}^t_r \leftarrow 0$ and set $S^t_{r}\leftarrow \emptyset$, $\forall r\in [k]$
   \FOR{$s\in S^t$} 
   	   \STATE Find $I(s)$ s.t. $c_{I(s)}=C(s)$
        \STATE $S^t_{I(s)}\leftarrow S^t_{I(s)}\cup s$; $\hat{n}_{I(s)}^t \leftarrow \hat{n}_{I(s)}^t + 1$
   \ENDFOR
   \FOR{$c_r^{t-1}\in C^{t-1}$}
   \IF{$\hat{n}_r^t\ne 0$}
     \STATE $c_r^t \leftarrow (1-\eta^t_r)c_r^{t-1} + \eta^t_r\hat{c}_r^t$ with 
     $\hat{c}_r^t:=\frac{\sum_{s\in S^t_r} s}{\hat{n}_r^t}$
   \ENDIF
   \ENDFOR
   \UNTIL{\texttt{convergence\_criterion} is satisfied}
\end{algorithmic}
\end{algorithm}
\section{Revisiting batch $k$-means, local convergence, and clusterability}
This section introduces the framework we construct to study $k$-means updates, under which we characterize the ``local optima'' of batch $k$-means and identify a sufficient condition for a local optimum to be a locally stable attractor\footnote{By ``locally stable'', we mean within a radius, batch $k$-means always converges to this local optimum.} to the algorithm.
Then we show how the clusterability of the dataset in fact determines the strength of a local optimum as an attractor.
\subsection{Batch $k$-means as an alternating mapping}
\label{sec:framework}
Corresponding to the two steps in an iteration of batch $k$-means, it alternates between two solution spaces: the \textit{continuous space} of sets of $k$ centroids, which we denote by $\{C\}$, and the \textit{finite set} of all $k$-clusterings, which we denote by $\{A\}$. 
Our key observation is that $\{C\}$ can be partitioned into equivalence classes by the clustering they induce on $X$, and the algorithm stops if and only if two consecutive iterations stay within the same equivalence class in \{C\}:
for any $C$, let 
$
v(C)
$ 
denote the clustering induced by its Voronoi diagram, i.e., 
$v$ is the mapping
$v(C):= V(C)\cap X$, where $V(C)\cap X$ is as defined in Section \ref{sec:intro}.
It can be shown that $v$ is a well-defined function if and only if $C$ is not a boundary point.
\begin{defn}[Boundary points]
$C$ is a boundary point if 
$\exists A\in V(C)\cap X$ s.t. for some $r\in [k], s\ne r$ and $x\in A_r\cup A_s$, 
$
\|x-c_r\|=\|x-c_s\|
$.
\end{defn}
We used ``$A\in V(C)\cap X$'' instead of $A= V(C)\cap X$, since $V(C)\cap X$ may induce more than one clustering if $C$ is a boundary point (see Lemma \ref{lm:equivalent_boundary}). 
So we abuse notation $V(C)\cap X$ to let it be the set of all possible clusterings of $C$.
For now, we ignore boundary points.
Then we say $C_1,C_2$ are equivalent if they induce the same clustering, i.e., $C_1\sim C_2$ if $v(C_1)=v(C_2)$.
This construction reveals that $\{C\}$ can be partitioned into a finite number of equivalence classes; each corresponds to a unique clustering $A\in \{A\}$.

An iteration of batch $k$-means can be viewed as applying the composite mapping $m\circ v:\{C\}\rightarrow\{C\}$, where Step 1 goes from $\{C\}$ to $\{A\}$ via mapping $v$, and Step 2 goes from $\{A\}$ to $\{C\}$ via the mean operation $m$.
\paragraph{``Local optima''}
It is well known that batch $k$-means stops at $t$ when  
$A^{t+1}=A^t$. Since 
$A^{t+1}=v(C^{t})= v\circ m(A^t)$
and $A^t=v(C^{t-1})$, 
this implies it stops at $t$ if and only if
$
v(C^t)=v(C^{t-1})
$, or equivalently,
$
v\circ m(A^t)=A^t
$.
Similarly, we can derive the algorithm stops if and only if
$
m(A^{t+1})=C^t
$, or
$
m\circ v(C^t)=C^t
$.

We can thus visualize batch $k$-means as an iterative mapping $m\circ v$ on $\{C\}$ that jumps from one equivalence class to another until it stays in the same equivalence class in two consecutive iterations, i.e., $v(C^{t+1})=v(C^t)$, and stops at some $C^*\in \{C^*\}$ 
(see Figure \ref{fig:solution_spaces} in the Appendix).
This stopping condition also provides a natural way of defining the ``local optima'' of $k$-means cost, i.e., they are the fixed points of mapping $m\circ v$ and $v\circ m$, which we call stationary solutions.
\begin{defn}[Stationary solutions]
We define $C^*\in \{C\}$ such that $m\circ v(C^*)=C^*$, and call it a \textbf{stationary point} of batch $k$-means. We let $\{C^*\}$ denote the set of all stationary points. Similarly, we call $A^*\in \{A\}$ a \textbf{stationary clustering} if $v\circ m (A^*)=A^*$, and we let $\{A^*\}$ denote the set of all stationary clusterings.
\end{defn}
Finally, to characterize local convergence we need to define a suitable measure of distance on each space.
\begin{defn}[Centroidal distance]
\label{defn:centroidal_dist}
For $C^{\prime}$ and $C$, we define centroidal distance 
$
\Delta(C^{\prime},C):=\min_{\pi: [k]\rightarrow [k]}\sum_r n_r\|c^{\prime}_{\pi(r)}-c_{r}\|^2
$, where $n_r=|A_r|$. 
\end{defn}
\begin{defn}[Clustering distance]
For $v(C^{\prime})$ and $v(C)$, we define the clustering distance
$
\mbox{ClustDist}(v(C^{\prime}),v(C)):=\max_r\frac{|A^{\prime}_{\pi(r)}\triangle A_r|}{n_r}
$, where $A^{\prime}:=v(C^{\prime})$, $A=v(C)$, $\triangle$ denotes set difference, and $\pi$ is the permutation attaining $\Delta(C^{\prime},C)$.
\end{defn}
Both distances are asymmetric, non-negative, and evaluates to zero if and only if two sets of centroids (clusterings) coincide.
If $C^*$ is a stationary point, then for any solution $C$, $\Delta(C,C^*)$ upper bounds the difference of $k$-means objective, $\phi(C)-\phi(C^*)$ (Lemma \ref{lm:kmdist_cdist}).
\paragraph{Remark}
For clarity of presentation, we have ignored the fact that $k$-means may produce degenerate solutions, where one or more clusters may be empty;
similarly, the definitions of stationary solutions and centroidal distance here ignore the possible existence of boundary points. 
In our analysis, we used more general definitions to handle these issues, whose details are provided in the Appendix. 
\subsection{A sufficient condition for the local convergence of batch $k$-means}
When is a stationary point also a locally stable attractor on $\{C\}$? 
We propose to characterize stability as a local Lipschitz condition on mapping $v(\cdot)$: we require $\mbox{ClustDist}(v(C), v(C^*))$ to be upper bounded by $\Delta(C,C^*)$ locally.
\begin{defn}\label{defn:stable}
We call $C^{*}$ a $(b_0,\alpha)$-stable stationary point if for any $C\in \{C\}$ such that 
$\Delta(C,C^*)\le b^{\prime}\phi^*$, $b^{\prime}\le b_0$, 
we have
$\mbox{ClustDist}(v(C),v(C^*))\le \frac{b}{5b+4(1+\phi(C)\slash\phi^*)}$,
with $b\le \alpha b^{\prime}$ for some $\alpha\in [0,1)$.
\end{defn}
Generalizing combinatorial arguments about batch $k$-means update in \cite{kumar, tang_montel:aistats16}, we show that the definition above is indeed a sufficient condition for local convergence of batch $k$-means.
\begin{lm}
\label{lm:stable_stat}
Let $C^*$ be a $(b_0,\alpha)$-stable stationary point. For any $C$ such that 
$\Delta(C,C^{*})\le b^{\prime}\phi^*$, $b^{\prime}\le b_0$,
apply one step of batch $k$-means update on $C$ results in a new solution $C^1$ such that
$\Delta(C^1,C^*)\le \alpha b^{\prime}\phi^*$.
\end{lm}
\vspace{-0.3cm}
\paragraph{Neighborhood of attraction}
By Lemma \ref{lm:stable_stat}, we can view $b_0$ as the radius of the neighborhood of attraction and $\alpha$ the strength of the attractor, since it determines the convergence rate.
A special case of $(b_0, \alpha)$-stability is when $\alpha=0$, which implies
$v(C)=v(C^*)$ if $C$ is within radius $b_0$ to $C^*$. In this case, batch $k$-means converges in one iteration.
Per our construction in Section \ref{sec:framework}, $b_0$ in this case is the radius of the equivalence class that maps to clustering $A^*=v(C^*)$.
In general, when $\alpha>0$, we expect the radius $b_0$ to be much larger.

Our characterization of local convergence of batch $k$-means does not depend on a specific clusterability assumption, unlike previous work \cite{kumar,awasthi:improved,tang_montel:aistats16}.
Instead, we will see that clusterability implies local Lipschitzness of mapping $v$.
\subsection{Local Lipschitzness and clusterability}
As discussed in Section \ref{sec:framework}, boundary points are problematic to the definition of mapping $v$, but how likely do they arise in practice? We answer this question by revealing the geometric implication of a boundary point, which will lead us to discover the connection between local Lipschitzness of $v$ and clusterability.

Consider any $C\in \{C\}$, and
let $A^{\prime}\in V(C)\cap X$: for a point $x\in A_r^{\prime}\cup A_s^{\prime}$, $s\ne r$,
let $\bar{x}$ denote the projection of $x$ onto the line joining $c_r, c_s$, we define
$$
\Delta_{rs}(C):=\min_{x\in A_r^{\prime}\cup A_s^{\prime}} |\|\bar{x}-c_r\| - \|\bar{x}-c_s\||
$$
\begin{defn}[$\delta$-margin]
For any $C$, we say $V(C)$ has a $\delta$-margin with respect to $X$
if $\exists A \in V(C)\cap X$ such that 
$
\min_{r,s\ne r}\Delta_{rs}(C)=\delta
$.
\end{defn}
\begin{lm}\label{lm:equivalent_boundary}
The following are equivalent
\begin{enumerate}
\vspace{-.4cm}
\item
$C$ is a boundary point 
\vspace{-.2cm}
\item
$V(C)$ has a zero margin with respect to $X$
\vspace{-.2cm}
\item
$|V(C)\cap X|>1$, i.e., the clustering determined by $V(C)$ is not unique.
\vspace{-.3cm}
\end{enumerate}
\end{lm} 
Thus, a set of $k$ centroids $C$ is a boundary point in space $\{C\}$ if and only if there is a data point $x\in X$ that sits exactly on the bisector of two centroids in $C$. 
We believe a symmetric configuration like this has a low probability to arise in practice, due to random perturbations in the real world, e.g., computational error.
With this insight, we define a general dataset to be one that is free of boundary stationary points. 
\paragraph{Assumption A}[\textbf{General dataset}]
$X$ is a general dataset if
$\forall C^*\in \{C^*\}$, $C^*$ has $\delta$-margin with $\delta>0$.

Note $\{C^*\}$ is a finite set, since $\{A^*\}$ is finite and we have the relation $C^*=m(A^*)$ (Lemma \ref{lm:stat_sols}). Thus, Assumption A is a mild condition, as it only requires that a finite subset of the continuous space $\{C\}$ to be free of boundary points, hence the name ``general''.

We show that for a general dataset, every stationary point is a locally stable attractor (in fact its neighborhood of attraction is exactly its equivalence class induced by $v$). In other words, mapping $v$ is always locally Lipschitz on a sufficiently small neighborhood of any stationary point. 
Moreover, on a general dataset, we can lower bound the centroidal distance between two consecutive $k$-means iteration, provided the algorithm has not converged.
Both results, summarized in Lemma \ref{lm:stat_stab2}, are important building blocks for our proof of Theorem \ref{thm:km}.
\begin{lm}
\label{lm:stat_stab2} 
If $X$ is a general dataset, then $\exists r_{\min}>0$ s.t.
\begin{enumerate}
\item
$\forall C^*\in \{C^*\}$, $C^*$ is a $(r_{\min},0)$-stable stationary point.
\item
Let $m(A^{\prime})\notin \{C^*\}$ for some $A^{\prime}\in \{A\}$ and let 
$A^{\prime}\in V(C^{\prime})\cap X$, then 
$\Delta(C^{\prime},m(A^{\prime}))\ge r_{\min}\phi(m(A^{\prime}))$.
\end{enumerate}
\end{lm}
In the lemma, $r_{\min}$ is a lower bound on the radius of attraction for points in $\{C^*\}$.
As discussed below Lemma \ref{lm:stable_stat}, this radius, although positive, can be very small.
The radius of an attractor, as we show next, can be related to the strength of margin $\delta$.
\paragraph{Assumption B}[\textbf{$f(\alpha)$-clusterability}]
We say a dataset-solution pair $(X,C^*)$ is $f(\alpha)$-clusterable, if $C^{*}\in \{C^*\}$ and $C^*$ has $\delta$-margin s.t.
$\forall r\in [k]$, $s\ne r$, 
$$
\delta\ge f(\alpha)\sqrt{\phi^{*}}(\frac{1}{\sqrt{n_r^{*}}}+\frac{1}{\sqrt{n_s^{*}}})
\mbox{~~for~} \alpha \in (0,1)
$$ with $f(\alpha)>\max\{64^2,\frac{5\alpha+5}{256\alpha},\max_{r\in [k],s\ne r} \frac{n_r^{*}}{n_s^{*}}\}$. 
\begin{prop}
\label{prop:geom}
Suppose $(X, C^*)$ satisfies Assumption (B).
Then, for any $C$ such that 
$\Delta(C,C^*)\le b\phi^*$ for some $b\le \frac{f(\alpha)^2}{16^2}$,
we have
$
\max_{r\in [k]}\frac{|A_r\triangle A_r^*|}{n_r^*}
\le
\frac{b}{f(\alpha)^3}
$.
That is, $C^*$ is 
$
(\frac{f(\alpha)^2}{16^2}, \alpha)
$ 
-stable.
\end{prop}
$f(\alpha)$-clusterability assumption is a simplified version of the proximity assumption in \cite{kumar}. It essentially requires that 
$\delta = \Omega(\sqrt{k}\sigma_{\max})$ for a stationary point $C^*$, where $\sigma_{\max}$ is the maximal standard deviation of an individual cluster.
It serves as an example showing how clusterability implies local Lipschitzness of $v$. 
Furthermore, Proposition \ref{prop:geom} reveals that the larger $f(\alpha)$ is, the larger the radius of attraction.
\section{Convergence analysis of stochastic $k$-means}
This section has three components. We first describe the technique we developed that helps us establish the local convergence of stochastic $k$-means, which will be used in the proof of both our main theorems.
Then, we provide proof outlines of Theorem \ref{thm:km} and Theorem \ref{thm:solution}, respectively.
Throughout our analysis, we consider learning rate of the form:
\vspace{-0.35cm}
\begin{align}\label{learning_rate}
\eta_r^t=\eta^t=\frac{c^{\prime}}{t_o+t} \mbox{~,~~}\forall r\in [k]
\end{align}
\subsection{Local convergence in the presence of stochastic noise}
Unlike a convex problem, the difficulty of establishing local convergence in our case is, if the algorithm's solution is driven off the current neighborhood of attraction by stochastic noise at any iteration, it may be drawn to a different attractor.
Fixing a $(b_0,\alpha)$-stable stationary point $C^*$, suppose the algorithm is within the neighborhood of attraction of $C^*$ at time $\tau$.
The event ``the algorithm's iterate is within radius $b_0$ to $C^*$ up to $t-1$'' can be formalized as:
\vspace{-0.2cm}
\begin{eqnarray}
\Omega_t:=\{\Delta(C^i,C^*)\le b_0\phi^*, \forall \tau\le i<t\}
\end{eqnarray}
Letting $t\rightarrow\infty$ leads to the following definition:
\vspace{-0.15cm}
\begin{eqnarray}
\Omega_{\infty}:=\{\Delta(C^i,C^*)\le b_0\phi^*, \forall i\ge \tau\}
\end{eqnarray}
Clearly, local convergence to $C^*$ implies $\Omega_{\infty}$; Lemma \ref{lm:stable_stat} also requires $\Omega_{\infty}$ as a prerequisite. So we need to show
$Pr(\Omega_{\infty})\approx 1$, even with stochastic noise.
\subsubsection{Inequality for a martingale-like process}
\label{sec:martingale}
We use $\Delta^t:=\Delta(C^t,C^{*})$ as a shorthand and we let $E_{\Omega_t}[\cdot]$ denote the expectation conditioning on $\Omega_t$.
Let $\Omega$ represent the sample space starting from $\tau$, then 
$\Omega_{t+1}\subset\Omega_{t}\subset \Omega, \forall t> \tau$. 
Conditioning on $\Omega_t$, we can apply Lemma \ref{lm:stable_stat} to get
\vspace{-0.3cm}
\begin{eqnarray*}
\Delta^t\le
\Delta^{t-1}(1-\frac{\beta}{t_o+t})
+[\frac{c^{\prime}}{t_o+t}]^2\epsilon_1^t 
+\frac{2c^{\prime}}{t_o+t}\epsilon_2^t 
~~| \Omega_t
\end{eqnarray*}
where with probability $1$, $\beta\ge 2$, and the stochastic noise terms $\epsilon_1^t, \epsilon_2^t$ are of order $O(\phi^{t-1})$.
Therefore, $(\Delta^t)$ is a supermartingale-like process with bounded stochastic noise, conditioning on $\Omega_t$.

To exploit this conditional structure, we partition the failure event $\Omega\setminus \Omega_{\infty}$, i.e., the event that the algorithm eventually escapes this neighborhood, as a disjoint union of events $\Omega_{t}\setminus\Omega_{t+1}$, and then our task becomes upper bounding $Pr(\Omega_{t}\setminus\Omega_{t+1})$ for all $t$.
To achieve this, we first derive an upper bound on the conditional moment generating function $E_{\Omega_t}[\exp\lambda\Delta^t]$ as a function of $b_0\phi^*$ and the noise terms, using ideas in \cite{balsubramani13}.
Then applying conditional Markov's inequality, we get
\vspace{-0.2cm}
\begin{eqnarray*}
Pr(\Omega_{t}\setminus\Omega_{t+1})
=
Pr\{\Delta^t>b_0\phi^*|\Omega_t\}
\le \frac{E_{\Omega_t}[\exp\lambda\Delta^t]}{\exp{\lambda b_0\phi^*}}
\end{eqnarray*}
Since the inequality holds for all $\lambda>0$, we can choose $\lambda$ as a function of $\ln t$, which enables us to bound 
$Pr(\Omega_{t}\setminus\Omega_{t+1})$ by $\frac{\delta}{(t+1)^2}$, for all $t\ge 1$, $\delta>0$, 
with sufficiently large $c^{\prime}$ and $t_o$ in \eqref{learning_rate}. 
This implies 
\vspace{-0.2cm}
$$
P(\Omega_{\infty})=1-\sum_{t\ge \tau}Pr(\Omega_{t}\setminus\Omega_{t+1})\ge 1-\delta
$$
Essentially, this is our variant of martingale large deviation bound. Comparing to related work \cite{balsubramani13, ge:sgd_tensor}, our technique yields a tighter bound on the failure probability than \cite{ge:sgd_tensor}, which uses Azuma's inequality, and is much simpler than \cite{balsubramani13};
the latter constructs a complex nested sample space and applies Doob's inequality, 
whereas ours simply uses Markov's inequality. 
In addition, our technique allows us to explore the noise dependence on $\Omega_t$, which leads to a weaker dependence of parameter $t_o$ on the initial condition $b_o\phi^*$.

We believe this technique can be useful for other non-convex analysis of stochastic methods. 
We provide one example here. Our current analysis considers the flat learning rate in \eqref{learning_rate}. 
However, in practice the following adaptive learning rate is commonly used:
\vspace{-0.265cm}
\begin{eqnarray}
\label{defn:adaptive_rate}
\eta^t_r:=\frac{\hat{n}_r^t}{\sum_{i\le t}\hat{n}_r^i}
\end{eqnarray}
We conjecture that stochastic $k$-means with the above learning rate also has $O(\frac{1}{t})$ convergence, as supported by our experiments (see Section \ref{sec:experiments}).
However, it is difficult to incorporate \eqref{defn:adaptive_rate} into our analysis: $\hat n_r^i$ is a random quantity whose probability depends on the clustering configuration $v(C^{i-1})$. 
To establish $O(\frac{1}{t})$ convergence, we need to show $E\eta_r^t\approx \Theta(\frac{1}{t})$. 
Without additional information, this is hopeless, as $\eta_r^t$ depends on information of the entire history of the process. 
But conditioning on $\Omega_t$, we can show that $n_r^i \approx n_r^*$, for all $r\in [k], i\ge \tau$.
Using this relation, we may approximate $E\eta_r^t$. 
Since our technique allows this conditional dependence, we may extend our local convergence analysis to incorporate the case where $\eta_r^t$ is adaptive.

Finally, conditioning on $\Omega_t$, we can combine Lemma \ref{lm:stable_stat} with the standard arguments in stochastic gradient descent \cite{rakhlin, arora:sparsecode, balakrishnan:EM} to obtain the $O(\frac{1}{t})$ local convergence rate (Theorem \ref{thm:mbkm_local}):
$
E_{\Omega_t}[\Delta(C^t,C^*)] = O(\frac{1}{t})
$.
\subsubsection{Proof sketch of main theorems}
Equipped with the necessary ingredients, we are ready to explain the analysis that leads to our main theorems.

Theorem \ref{thm:km} is a global convergence result. To prove it, we divide our analysis of Algorithm \ref{alg:MBKM} into two phases, that of global convergence and local convergence, indicated by the distance from the current solution to stationary points, $\Delta(C^t,C^*)$.

We define global convergence phase as a time interval of random length $\tau$ such that
$\forall t <\tau$, 
$\forall C^*\in\{C^*\}$, $\Delta(C^t,C^*)> \frac{1}{2}r_{\min}\phi^*$ ($r_{\min}$ as defined in Lemma \ref{lm:stat_stab2}).
During this phase, we obtain a lower bound on the expected decrease in $k$-means objective (Lemma \ref{lm:kmeans}): 
\begin{eqnarray*}
E[\phi^{t+1}-\phi^t|F_t]
\le
-2\eta^{t+1}p_{\min}^{t+1} (\phi^t - \tilde\phi^t) 
+(\eta^{t+1})^26\phi^t
\end{eqnarray*}
Here,
$\tilde\phi^t:=\sum_r\sum_{x\in v(c_r^{t})}\|x-m(v(c_r^{t}))\|^2$ 
and $p_{\min}^{t}:=\min_{r,p_r^{t}(m)>0}p_r^{t}(m)$, with
$
p_r^t(m)
=Pr\{c_r^{t-1}\mbox{~is updated at~}t\mbox{~with sample size~}m\}
= 1- (1-\frac{n_r^{t}}{n})^m
$.
Thus, the term $p_{\min}^{t+1} (\phi^t - \tilde\phi^t)$
lower bounds the drop in $k$-means objective. 
For $p_r^{t+1}(m)>0$, by the discrete nature of cluster assignment, $n_r^t\ge 1$.
So
$p_{\min}^{t+1}\ge 1- (1-\frac{1}{n})^m\ge 1- e^{-\frac{m}{n}}$.

On the other hand, $\phi^t-\tilde\phi^t=\Delta(C^t, m(v(C^t)))$ by Lemma \ref{centroidal}. 
Thus, to lower bound the decrease by zero, we only need to lower bound $\Delta(C^t, m(v(C^t)))$.
The idea is, in case $m(v(C^t)))$ is a non-stationary point, by part 2 of Lemma \ref{lm:stat_stab2},
$\Delta(C^t, m(v(C^t)))> \frac{1}{2}r_{\min}\phi(m(v(C^t)))$.
Otherwise, $m(v(C^t)))$ is a stationary point, and by definition of the global convergence phase, the same lower bound applies, which implies
$p_{\min}^t (\phi^t - \tilde\phi^t)$ is lower bounded by a positive constant in the global convergence phase. 
Since we choose $\eta^t:=\Theta(\frac{1}{t})$, 
the expected per iteration drop of cost is of order $\Omega(\frac{1}{t})$, which forms a divergent series; after a sufficient number of iterations the expected drop can be arbitrarily large. We conclude that $\Delta(C^t,C^*)$ cannot be bounded away from zero asymptotically, since the $k$-means cost of any clustering is positive (Lemma \ref{lm:limit_cluster}).
Hence, starting from any initial point $C^0$, the algorithm will always be drawn to a stationary point, ending its global convergence phase after a finite number of iterations, i.e., $Pr(\tau<\infty)=1$. 

At the beginning of the local convergence phase,
$\Delta(C^{\tau},C^*)\le \frac{1}{2}r_{\min}\phi^*$ for some $C^*\in \{C^*\}$.
Again by Lemma \ref{lm:stat_stab2}, the algorithm is within the neighborhood of attraction of $C^*$, and thus we can apply the local convergence result in Theorem \ref{thm:mbkm_local}. Combining both phases leads us to Theorem \ref{thm:km}.
\begin{thm}\label{thm:km}
Suppose $X$ satisfies Assumption (A). Fix any $0<\delta <\frac{1}{e}$, if we run Algorithm \ref{alg:MBKM} with arbitrary
$C^0$ such that $C^0\subset conv(X)$, and any mini-batch size $m\ge 1$, and choose learning rate $\eta^t=\frac{c^{\prime}}{t+t_o}$ 
such that 
$$
c^{\prime}
> 
\max\{
\frac{\phi_{\max}}{(1-e^{-\frac{m}{n}})r_{\min}\phi_{opt}},
\frac{1}{(1-e^{\frac{4m}{5n}})}
\}
$$
$$
t_o\ge 768(c^{\prime})^2(1+\frac{1}{r_{\min}})^2n^2\ln^2\frac{1}{\delta}
$$
Then there exists events $G(A^{*})$, parametrized by $A^{*}$, such that 
$$
Pr\{\cup_{A^{*}\in \{A^*\}_{[k]}} G(A^{*})\}\ge 1-\delta
$$
For any stationary clustering $A^*$, we have
$\forall t\ge 0$,
$$
E\{\phi^t-\phi^{*}|G(A^{*})\} = O(\frac{1}{t})
$$
\end{thm}
\paragraph{Remark:}
$\cup_{A^{*}\in \{A^*\}} G(A^{*})$ is contained in the event that Algorithm \ref{alg:MBKM} converges to a stationary point. Thus, Theorem \ref{thm:km} implies that, with any reasonable initialization and sufficiently large $c^{\prime}$, $t_o$, stochastic $k$-means converges globally almost surely; conditioning on global convergence to a stationary point $A^*$, the convergence rate is $O(\frac{1}{t})$ in expectation. 
Also note $\phi_{\max}$ is upper bounded, since $C^0\subset conv(X)$ implies $C^t\subset conv(X)$, $\forall t\ge 1$ (see Claim \ref{claim:conv_hull}). 
Finally, note that Theorem \ref{thm:km} establishes global convergence to a local optimum, but it does not guarantee that stochastic $k$-means converges to the same local optimum as its batch counterpart, even with the same initialization.
 
Our next theorem complements Theorem \ref{thm:km} in the sense that it provides local convergence to a global optimum of $k$-means problem: it shows that if we use Algorithm \ref{alg:seeding} as our seeding algorithm and the optimal clustering satisfies $f(\alpha)$-clusterability, then stochastic $k$-means converges to a global optimum of $k$-means objective at rate $O(\frac{1}{t})$. 
Its proof has three parts: First, we show $f(\alpha)$-clusterability implies $(b_0,\alpha)$-stability, as stated in Proposition \ref{prop:geom}. 
Second, we show $C^0$ found by Algorithm \ref{alg:seeding} is within the neighborhood of attraction of the optimal solution with high probability, by adapting Theorem of \cite{tang_montel:aistats16} with the additional assumption that
\vspace{-0.2cm}
\begin{align}\label{eqn:add_assumption}
f(\alpha)\ge 5\sqrt{\frac{1}{2w_{\min}}\ln (\frac{2}{\xi p^*_{\min}}\ln\frac{2k}{\xi})}
\end{align}
where $w_{\min}$ and $p^*_{\min}$ are geometric properties of clustering $v(C^*)$ (see definitions in Appendix).
Finally, combining these with Theorem \ref{thm:mbkm_local} completes the proof.
\begin{algorithm}[t]
\caption{Buckshot seeding \cite{tang_montel:aistats16}}
\label{alg:seeding}
\begin{algorithmic}
\STATE $\{\nu_i,i\in [m_o]\} \leftarrow $ sample $m_o$ points from $X$ uniformly at random with replacement
\STATE $\{S_1,\dots,S_k\} \leftarrow $run Single-Linkage on $\{\nu_i,i\in [m_o]\}$ until there are only $k$ connected components left
\STATE $C^0=\{\nu_r^{*},r\in [k]\}\leftarrow$ take the mean of the points in each connected component $S_r,r\in [k]$
\end{algorithmic}
\end{algorithm}
\begin{thm}
\label{thm:solution}
Suppose $(X, C^*)$ satisfies Assumption (B) and $f(\alpha)$ in addition satisfies
\eqref{eqn:add_assumption} for any $0<\alpha<1$, $\xi>0$.
Fix $\beta\ge 2$, and $0<\delta<\frac{1}{e}$.
If we initialize $C^o$ in Algorithm \ref{alg:MBKM} by Algorithm \ref{alg:seeding}, with $m_o$  satisfying
$
\frac{\log\frac{2k}{\xi}}{p^*_{\min}}
<m_o
<\frac{\xi}{2}\exp\{2(\frac{f(\alpha)}{4}-1)^2w_{\min}^2\}
$, 
and running Algorithm \ref{alg:MBKM} with learning rate of the form 
$\eta^t = \frac{c^{\prime}}{t+t_o}$ and mini-batch size $m$ so that
$$
m > \frac{\ln (1-\sqrt{\alpha})}{\ln (1-\frac{4}{5}p_{\min}^*)}
$$
\vspace{-0.32cm}
$$
c^{\prime}>\frac{\beta}{2[1-\sqrt{\alpha}-e^{-\frac{4}{5}mp_{\min}^*}]}
\mbox{~and~~}
t_o\ge
867(c^{\prime})^2n^2\ln^2\frac{1}{\delta}
$$
Then $\forall t\ge 1$, there exists event $G_t\subset \Omega$ s.t.
$$
Pr\{G_t\}\ge (1-\delta)(1-\xi) \mbox{~~and}
$$ 
$$
E[\phi^t|G_t]-\phi^{*}\le E[\Delta(C^t,C^{*})|G_t] = O(\frac{1}{t})
$$
\end{thm}
\paragraph{Remark:}
Theorem \ref{thm:solution} in fact applies to any stationary point satisfying $f(\alpha)$-clusterability, which includes the optimal $k$-means solution. 
Interestingly, we cannot provide guarantee for online $k$-means ($m=1$) here.
Our intuition is, instead of allowing stochastic $k$-means to converge to any stationary point as in Theorem \ref{thm:km}, it studies convergence to a target stationary point; a larger $m$ provides more stability to the algorithm and prevents it from straying away from the target.
\vspace{-0.3cm}
\begin{figure*}[t]
\begin{subfigure}{.55\textwidth}
  \centering
  \includegraphics[width=\linewidth, height = 0.5\textwidth]{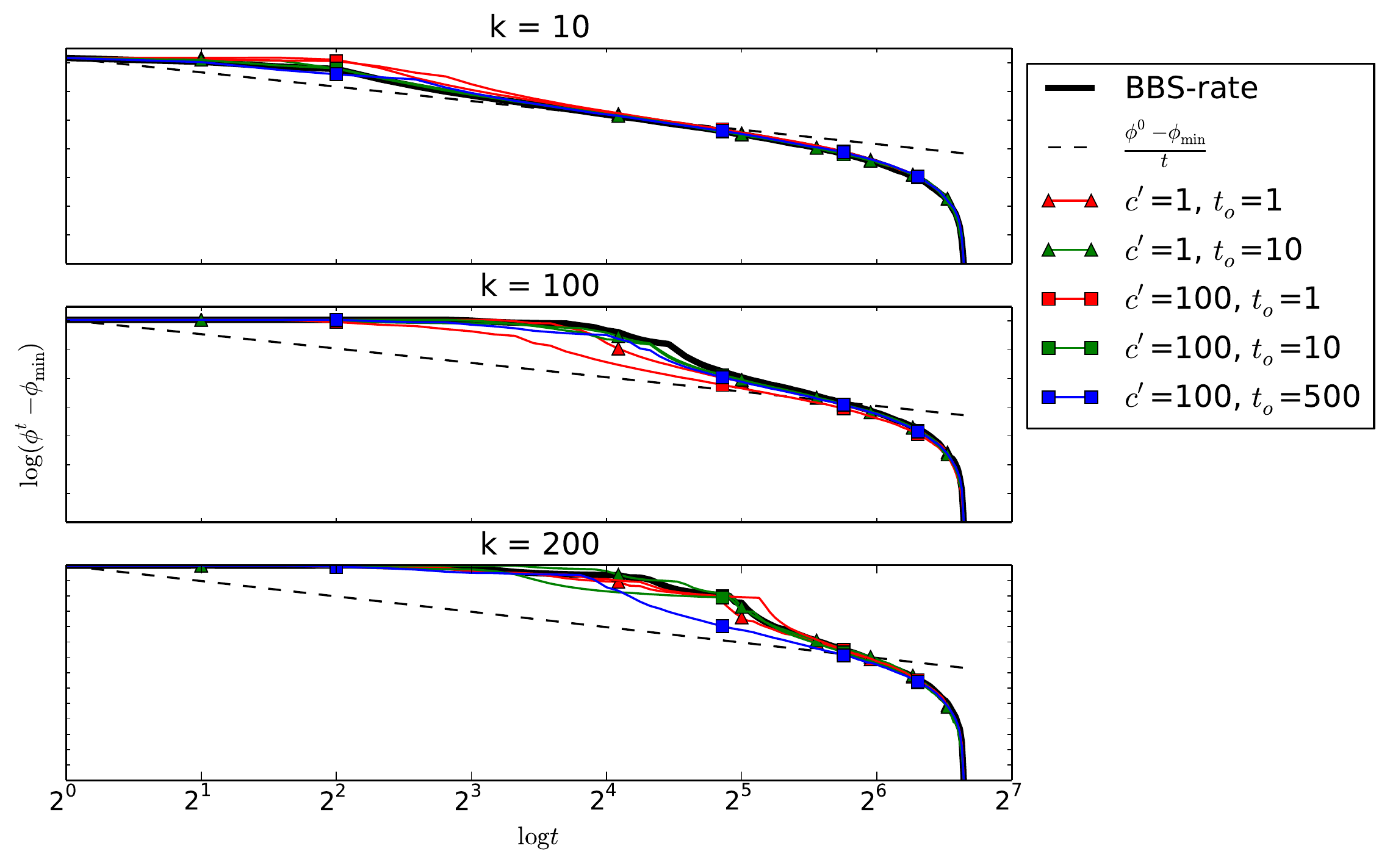}
  \caption{$m=1$, log-log plot}
  \label{fig:mb1_log}
\end{subfigure}%
\begin{subfigure}{.4\textwidth}
  \centering
  \includegraphics[width=\linewidth, height = 0.68\textwidth]{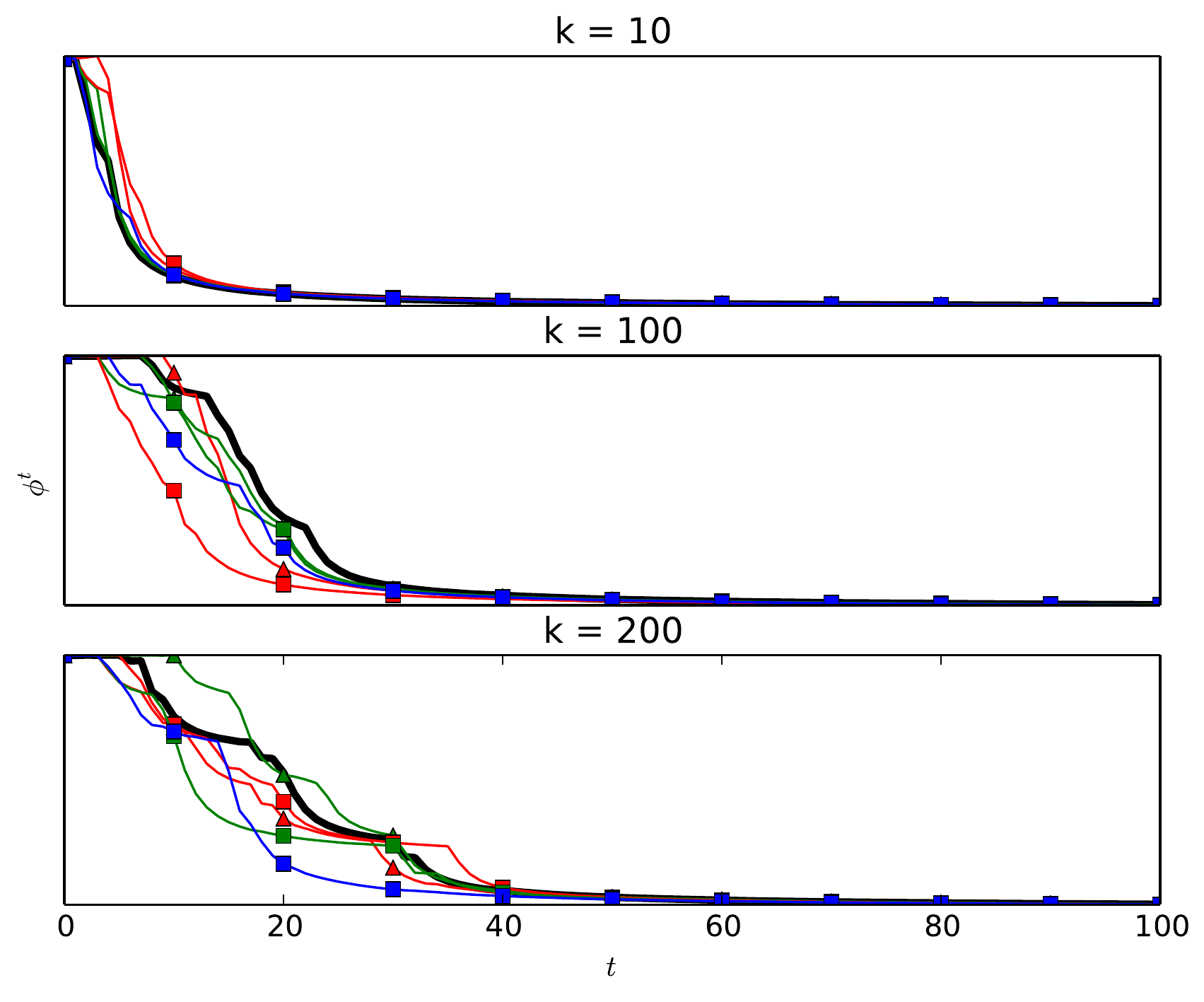}
  \caption{$m=1$, original plot}
  \label{fig:mb1_origin}
\end{subfigure}
\vspace{-0.1cm}
\vspace{-0.3cm}
\end{figure*}
\setcounter{figure}{1}
\begin{figure*}
\begin{subfigure}{.33\linewidth}
  \centering
  \includegraphics[width=\linewidth, height = 0.8\textwidth]{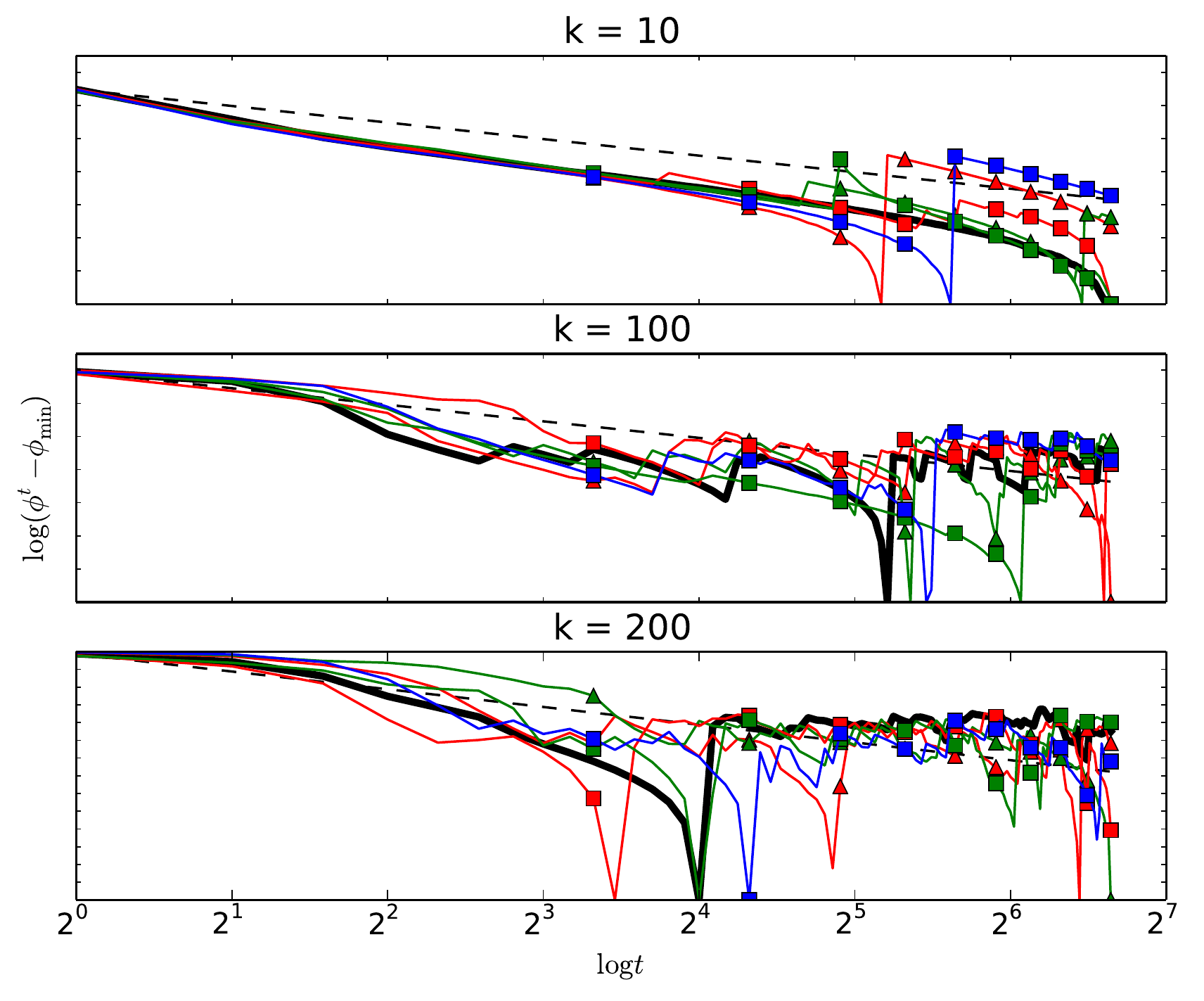}
  \caption{$m=10$}
  \label{fig:mb10_log}
\end{subfigure}%
\begin{subfigure}{.33\linewidth}
  \centering
  \includegraphics[width=\linewidth, height = 0.8\textwidth]{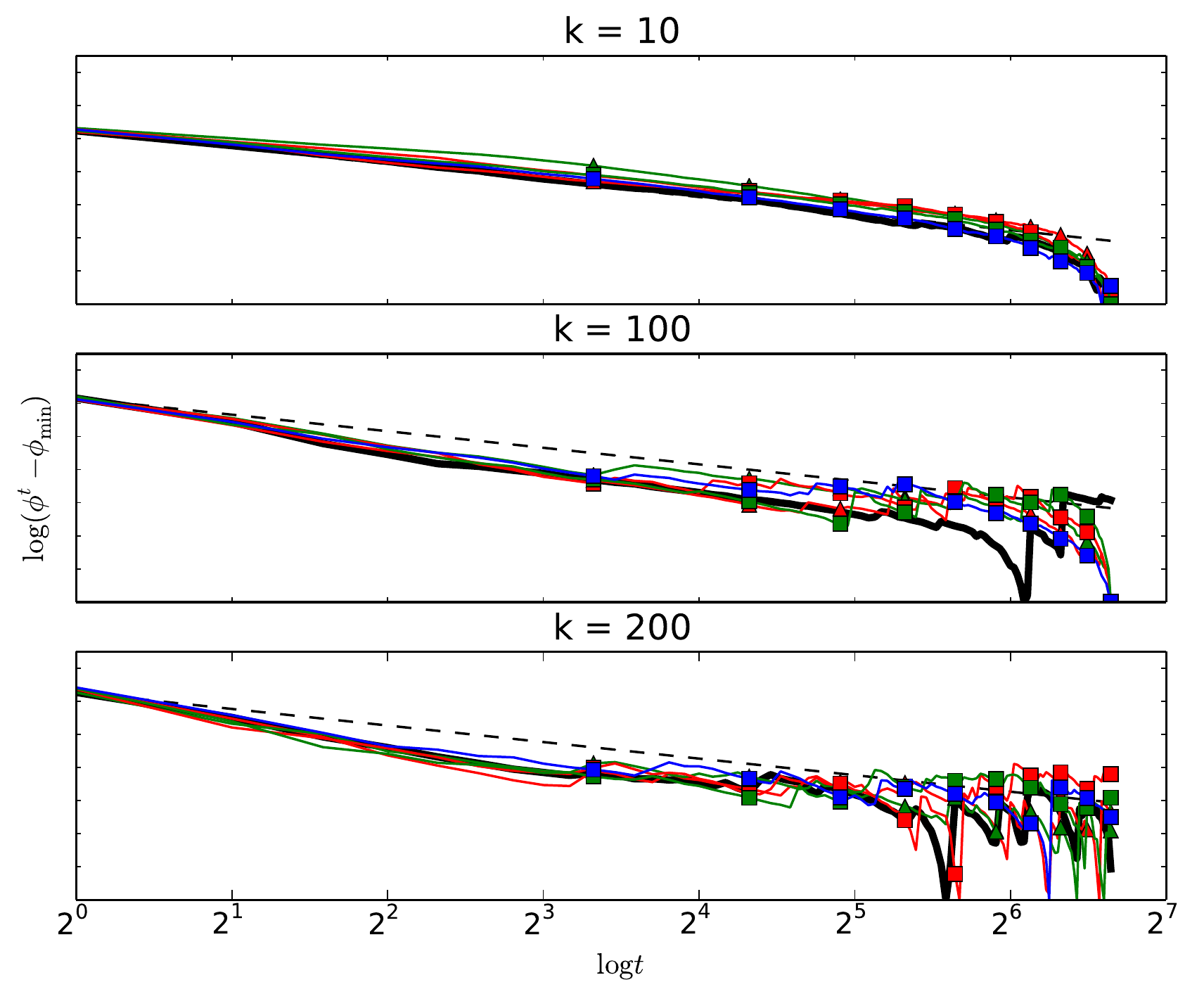}
  \caption{$m=100$}
  \label{fig:mb100_log}
\end{subfigure}
\begin{subfigure}{.33\linewidth}
  \centering
  \includegraphics[width=\linewidth, height = 0.8\textwidth]{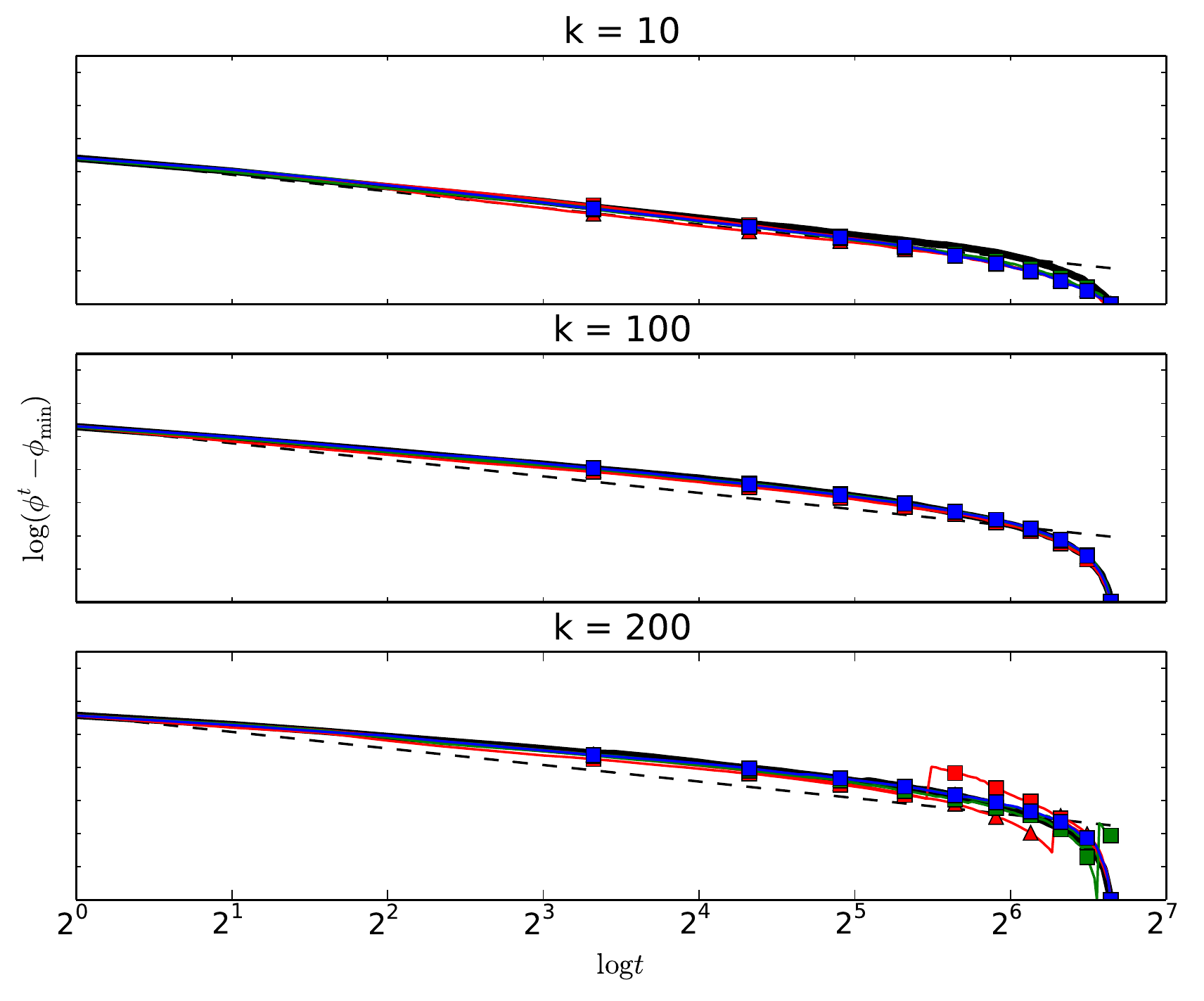}
  \caption{$m=1000$}
  \label{fig:mb1000_log}
\end{subfigure}
\caption{Convergence graphs of stochastic $k$-means}
\label{fig:log}
\vspace{-0.4cm}
\end{figure*}
\section{Experiments}
\label{sec:experiments}
We use \texttt{Python} and its \texttt{scikit-learn} package \cite{scikit-learn} for our experiments, which has stochastic $k$-means implemented. We disabled centroid relocation and modified their source code to allow a user-defined learning rate (their learning rate is fixed as $\eta^t_r:=\frac{\hat{n}_r^t}{\sum_{i\le t}\hat{n}_r^i}$, as in \cite{BottouBengio, Sculley}, which we refer to as \textbf{BBS-rate} subsequently). 
\subsection{Verification of the linear learning rate}\label{sec:exp1}
To verify the $O(\frac{1}{t})$ global convergence rate of Theorem \ref{thm:km}, we first run stochastic $k$-means with varying learning rate, mini-batch size, and $k$ on \texttt{RCV1} \cite{data:rcv1}. The dataset has manually categorized $804414$ newswire stories with $103$ topics, where each story is a $47236$-dimensional sparse vector; it was used in \cite{Sculley} for empirical evaluation of mini-batch $k$-means.
We experiment with both the flat learning rate in \eqref{learning_rate} and the adaptive learning rate in \eqref{defn:adaptive_rate}.
Figure \ref{fig:log} shows the convergence in $k$-means cost of stochastic $k$-means algorithm over $100$ iterations with different choices of $m$ and $k$; fix each pair $(m, k)$, we initialize Algorithm \ref{alg:MBKM} with a same set of $k$ randomly chosen data points and run stochastic $k$-means with varying learning rate parameters $(c^{\prime}, t_o)$, and we average the performance of each learning rate setup over $5$ runs to obtain the original convergence plot. 
Figure \ref{fig:mb1_origin} is an example of a convergence plot before transformation. The dashed black line in each log-log figure is $\frac{\phi^0-\phi_{\min}}{t}$ ($\phi_{\min}$ is the lowest empirical $k$-means cost), a function of order $\Theta(\frac{1}{t})$. To compare the performance of stochastic $k$-means with this baseline, we first transform the original $\phi^t$ vs $t$ plot to that of $\phi^t-\phi_{\min}$ vs $t$. By Theorem \ref{thm:km}, $E[\phi^t-\phi^{*}|G(A^{*})]=O(\frac{1}{t})$, so we expect the slope of the log-log plot of $\phi^t-\phi^{*}$ vs $t$ to be at least as large as that of $\Theta(\frac{1}{t})$. Although we do not know the exact cost of the stationary point, we use $\phi_{\min}$ as a rough estimate of $\phi^{*}$. 
\paragraph{Interpretation}
First, we observe that most log-log convergence graphs fluctuate around a line with a slope at least as steep as that of $\Theta(\frac{1}{t})$, verifying the linear convergence rate (one exception is that this only happens towards the end of convergence in Figure \ref{fig:mb1_log}; we discuss this behavior in the next section). 
Interestingly, the convergence does not seem to be sensitive to the learning rate in our experiment: the adaptive \textbf{BBS-rate} behave similarly to our flat learning rate with different parameters $(c^{\prime}, t_o)$.
On the other hand, the convergence rate of stochastic $k$-means seems sensitive to the ratio $\frac{m}{k}$; when the ratio is higher, faster and more stable convergence is observed. 
\subsection{Further exploration of convergence}\label{sec:exp2}
Our second set of experiments serves to corroborate our observations from previous experiments, and to further explore the convergence behavior subject to different factors.
To this end, we include two more benchmark datasets, \texttt{mnist} and \texttt{covtype}, a simulated dataset \texttt{gauss}, and add stochastic $k$-means with a constant learning rate.
Instead of a running the algorithm for only $100$ iterations, we adopt a setup that is more akin to what is commonly used in practice --- we divide the convergence in to 20 epochs, where the epoch length are chosen from $60$, $600$, $6000$. 
\paragraph{Two-phased convergence explained by $t_o$}
From our previous experiment, we observe that the initial phase of convergence is sometimes slower than $\Theta(\frac{1}{t})$ (e.g., in Figure \ref{fig:mb1_log}). 
This phenomenon also shows up, and in fact more frequently, when we turn to other datasets.
Here is our explanation: the $\frac{b}{t}$ (let $b$ be some constant) model of convergence is not exactly what we derive in our theorems:
the exact form of convergence rate we obtained in Theorem \ref{thm:km}, \ref{thm:mbkm_local}, and \ref{thm:solution}, which we hide behind the Big-$O$ notation, is of the form
$
\frac{b}{t+t_o}
$. After taking into account $t_o$, our theoretical convergence rate well-matches the two-phased convergence we observe in practice. 
For example, in Figure \ref{exp2:covtype}, when $t_o$ is set to $60$ or higher, the actual convergence can be simulated by the proxy to our theoretical bound\footnote{The difference between their intercepts at the $y$-axis is caused by a constant factor.}, $\frac{\phi^o-\phi_{\min}}{t+t_o}$. 
Note the practical requirement on $t_o$ is much more optimistic than what Theorem \ref{thm:km} needs, that is,
$$
t_o\ge 768(c^{\prime})^2(1+\frac{1}{r_{\min}})^2n^2\ln^2\frac{1}{\delta}
$$
Again, we observe that the convergence rate of stochastic $k$-means is not sensitive to the choice of $t_o$, despite the fact that the latter plays a role in explaining the convergence rate.  
\paragraph{Runtime vs final $k$-means cost}
We compare the $k$-means cost achieved by stochastic $k$-means with different learning rate and epoch length to that achieved by batch $k$-means after 20 iterations. 
Each entry in the table is computed as $\frac{\phi^{T}}{\phi_{batch}}$, where $T$ is $20\times E$, where $E$ is the epoch length; $\phi_{batch}$ is the final $k$-means cost after running batch $k$-means for $20$ iterations. \textbf{flat} and \textbf{const} stands for our analyzed learning rate in \eqref{learning_rate} and a constant learning rate, which we set to be $\frac{1}{\sqrt{E}}$. For the flat learning rate, we arbitrarily choose $c^{\prime}=4$, and $t_o$ to be one of $\{10,60,600,6000\}$, which ever gives the lowest $k$-means cost.

As shown in Table \ref{table:cost}, the final $k$-means cost of stochastic $k$-means, using epoch length of 600, is already comparable to its batch counterpart. 
On the other hand, the data sizes of \texttt{mnist, covtype, gauss, rcv1} are $60k$, $500k$, $600k$, and $800k$, respectively. So even using the largest epoch length, $6k$, stochastic $k$-means would save at least one-tenth of the computation, due to distance evaluation, in comparison to batch $k$-means.

From the convergence plots (Figure \ref{exp2:covtype} and \ref{exp2:mnist}), we see that the convergence behavior of stochastic $k$-means is not sensitive to the choice of learning rate. 
Here, we observe that learning rate does not affect the final $k$-means cost too much either; even a constant learning rate works! 
\begin{table}
\caption{Final $k$-means cost relative to batch $k$-means}
\centering
\begin{tabular}{rlll}
\hline
& \texttt{covtype} &&\\
\hline
   k & E=60,flat,BBS,const   & E=600,flat,BBS,const   & E=6k,flat,BBS,const   \\
\hline
  10 & 0.93,0.92,0.93        & 0.99,0.93,0.99         & 1.03,1.01,1.03          \\
  50 & 1.13,1.12,1.13        & 1.02,1.03,1.02         & 1.01,1.01,1.01          \\
 100 & 1.15,1.10,1.15        & 1.05,1.07,1.05         & 1.02,1.01,1.02          \\
\hline
\hline
 &\texttt{mnist}&&\\
\hline
  10 & 1.07,1.07,1.07        & 1.02,1.02,1.02         & 1.03,1.02,1.03          \\
  50 & 1.15,1.15,1.15        & 1.06,1.07,1.06         & 1.02,1.02,1.02          \\
 100 & 1.18,1.18,1.18        & 1.07,1.06,1.07         & 1.02,1.02,1.02          \\
\hline
\end{tabular}
\begin{subtable}{0.46\textwidth}
\begin{tabular}{rll}
\hline
 \texttt{rcv1} & E=60& E=600 \\
\hline
  10 & 1.03,1.03,1.03        & 1.02,1.02,1.02         \\
  50 & 1.06,1.06,1.06        & 1.06,1.06,1.06         \\
 100 & 1.09,1.09,1.09        & 1.07,1.07,1.07         \\
\hline
\end{tabular}
\end{subtable}
\begin{subtable}{0.46\textwidth}
\begin{tabular}{rll}
\hline
 &\texttt{gauss} E=60& E=600 \\
\hline
 & 1.05,1.07,1.05        & 1.03,1.03,1.03         \\
   & 1.16,1.14,1.16        & 1.07,1.05,1.07         \\
 & 1.11,1.11,1.11        & 1.02,1.02,1.02         \\
\hline
\end{tabular}
\end{subtable}
\label{table:cost}
\end{table}
\paragraph{Significance of different factors to convergence}
Finally, from our experiments we summarize the relative importance of different factors to the convergence behavior of stochastic $k$-means as below:
\begin{itemize}
\item
mini-batch size $m$: the larger $m$ is, the convergence becomes more stable and faster.
\item
number of clusters $k$: the smaller $k$ is, the convergence becomes more stable and faster.
\item
dataset: although $\frac{b}{t+t_o}$ is observed for all datasets, stochastic $k$-means seems to favor certain datasets to others. For example, on \texttt{rcv1}, almost $\frac{b}{t}$ (and sub-linear when $m$ is larger) convergence rate is observed. 
\item
learning rate: the algorithm is not sensitive to the choice of learning rate.
\end{itemize}
\begin{figure*}[h]
\begin{subfigure}{.33\textwidth}
  \centering
  \includegraphics[width=\linewidth]{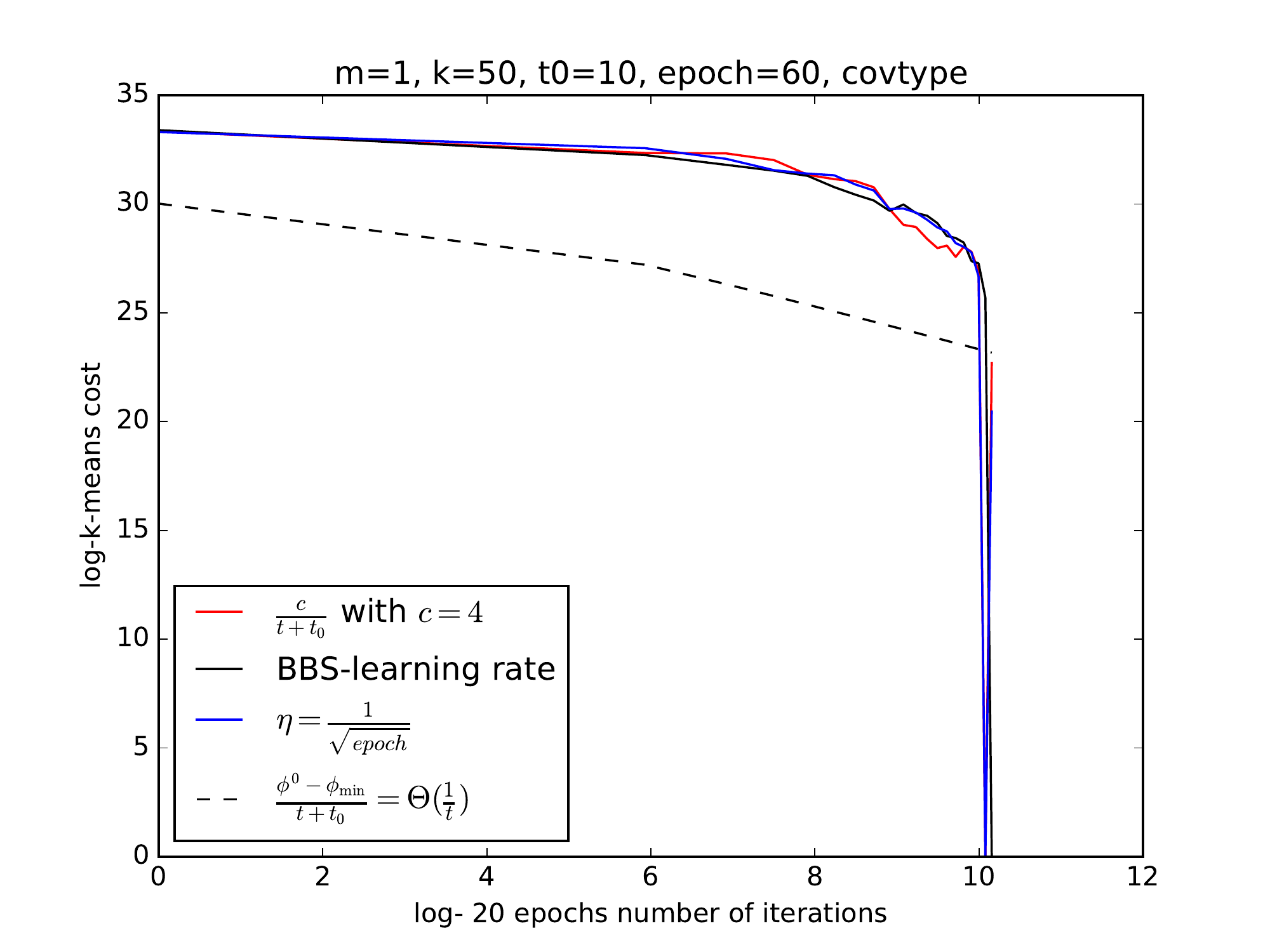}
\end{subfigure}%
\begin{subfigure}{.33\textwidth}
  \centering
  \includegraphics[width=\linewidth,]{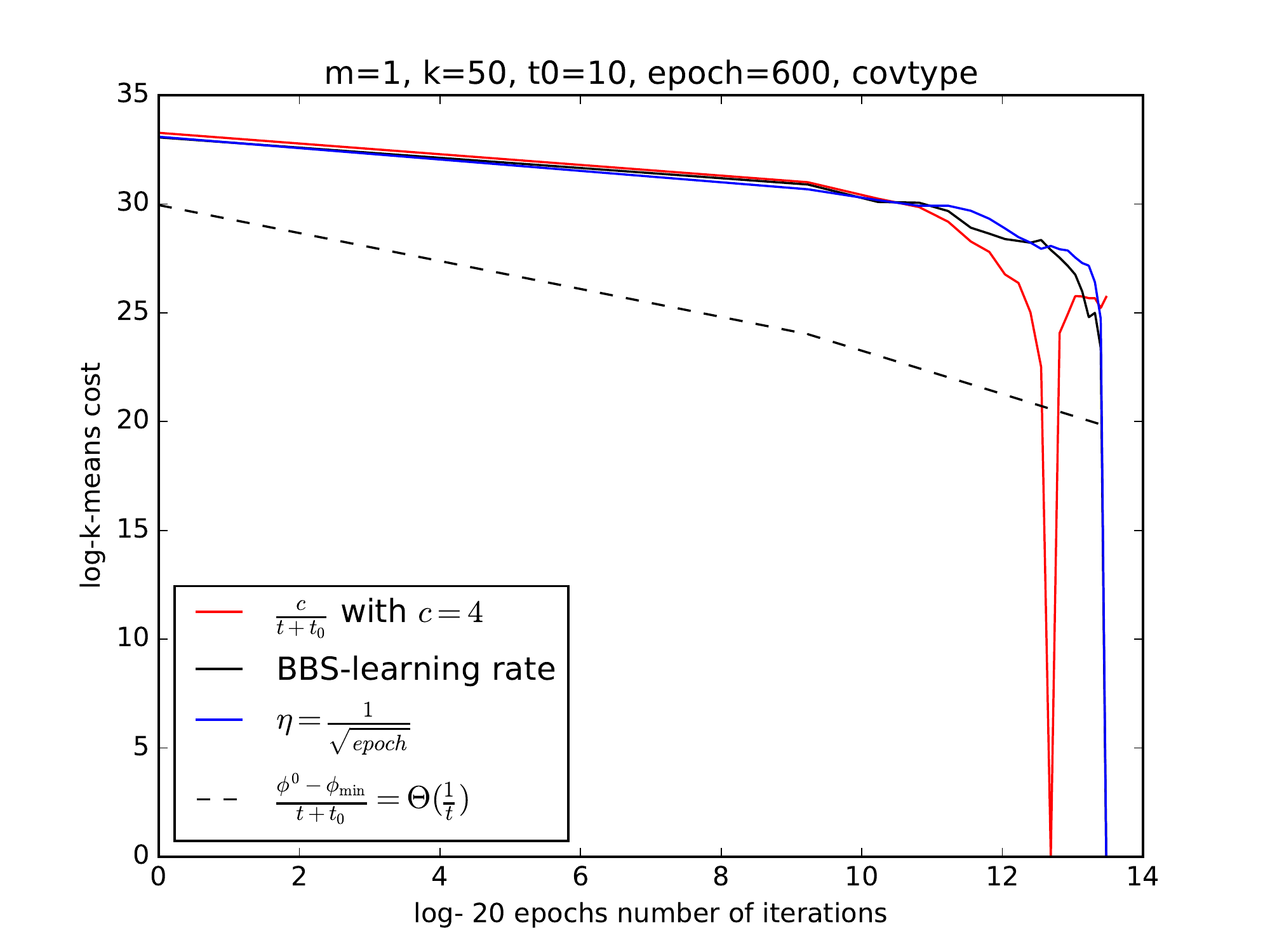}
\end{subfigure}
\begin{subfigure}{.33\textwidth}
  \centering
  \includegraphics[width=\linewidth]{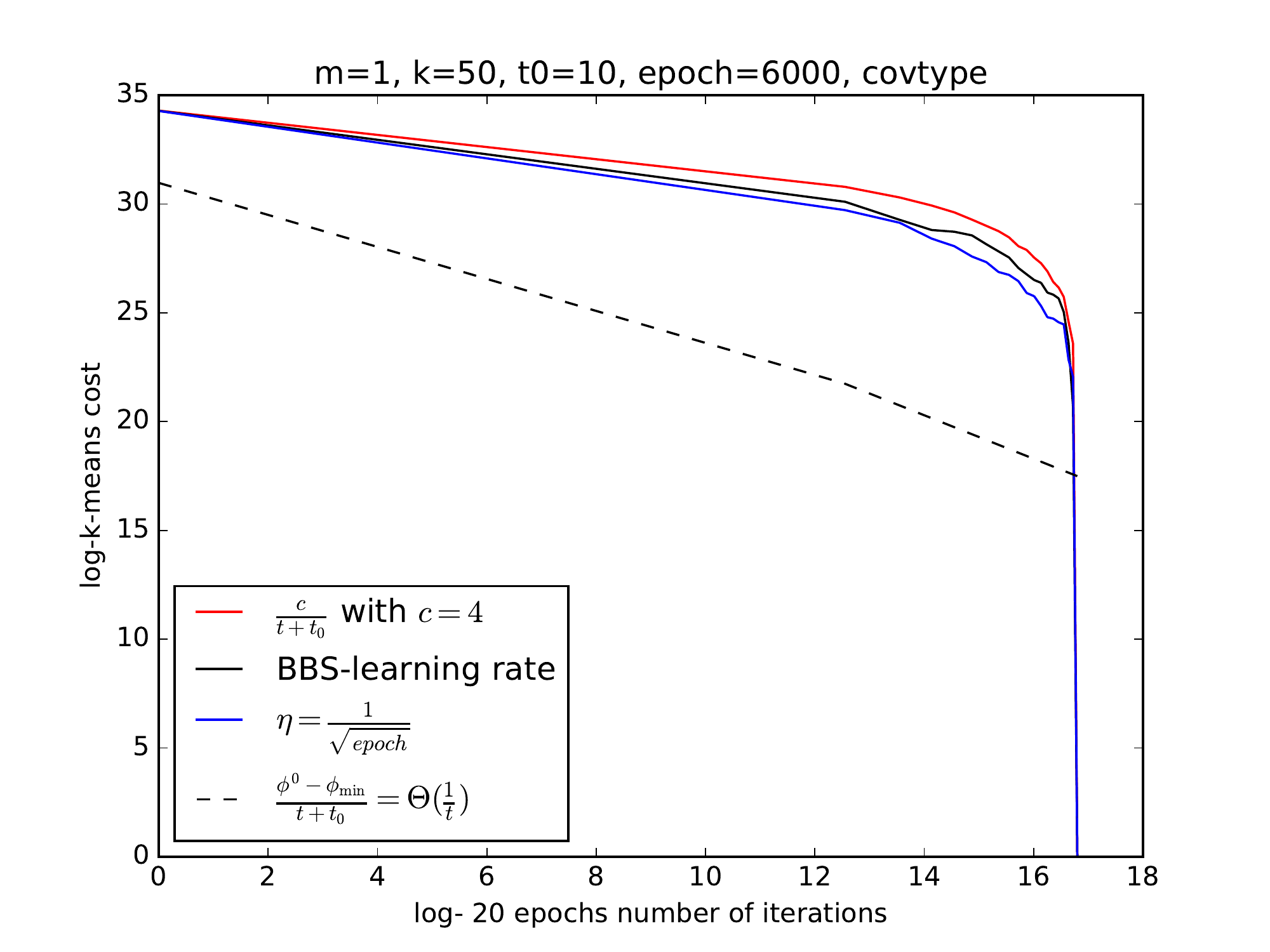}
\end{subfigure}
\begin{subfigure}{.33\textwidth}
  \centering
  \includegraphics[width=\linewidth]{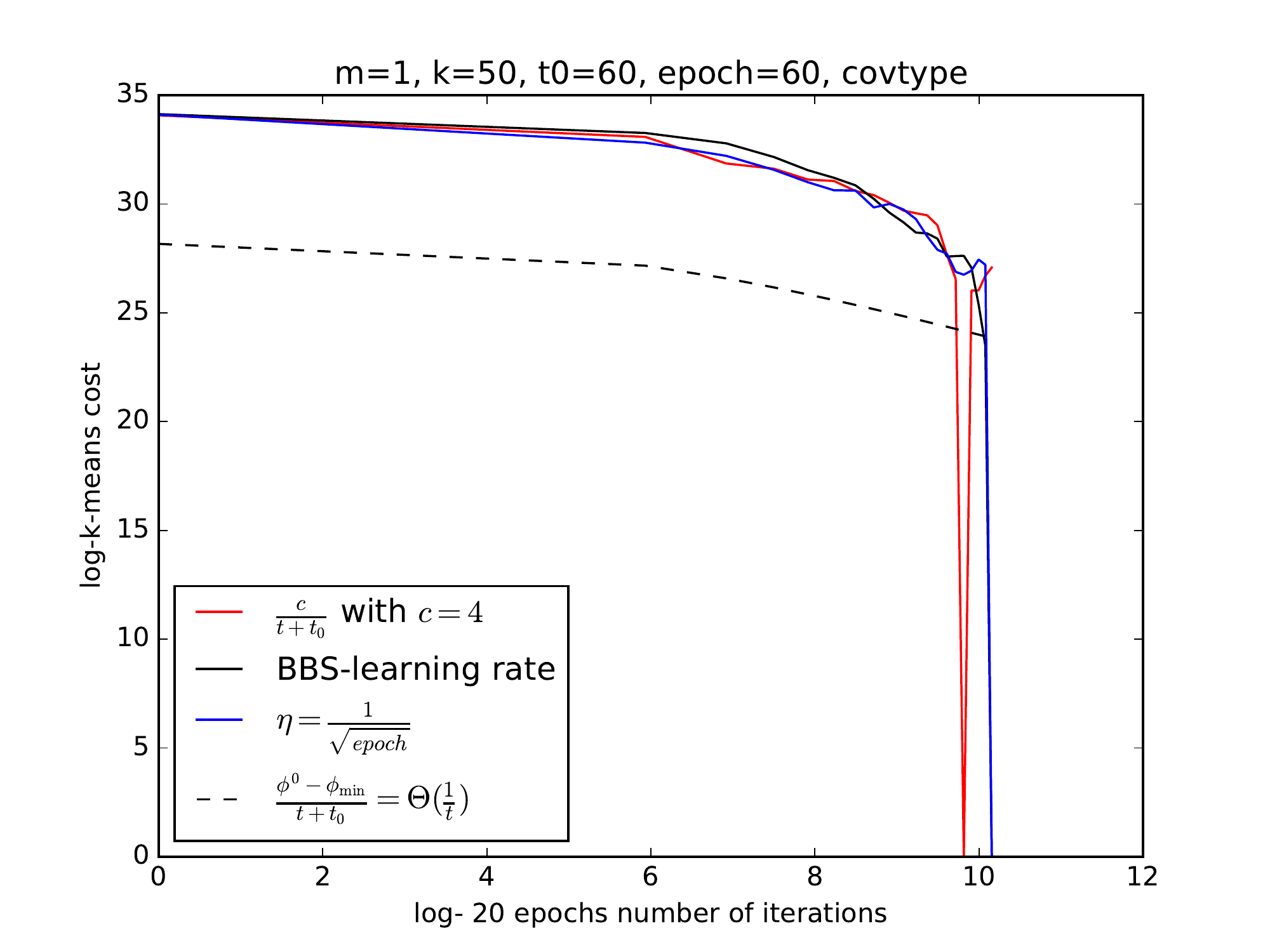}
\end{subfigure}%
\begin{subfigure}{.33\textwidth}
  \centering
  \includegraphics[width=\linewidth,]{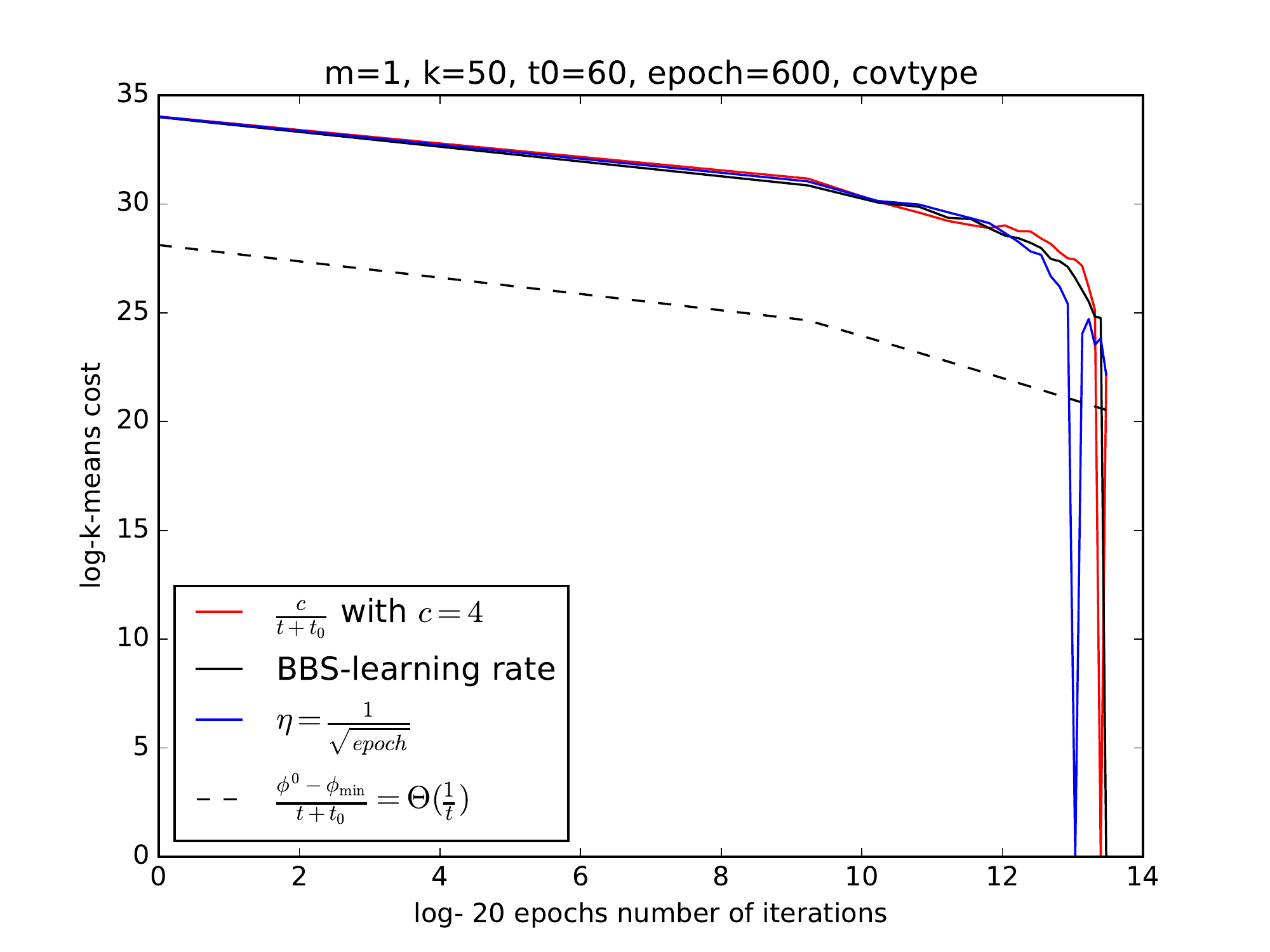}
\end{subfigure}
\begin{subfigure}{.33\textwidth}
  \centering
  \includegraphics[width=\linewidth]{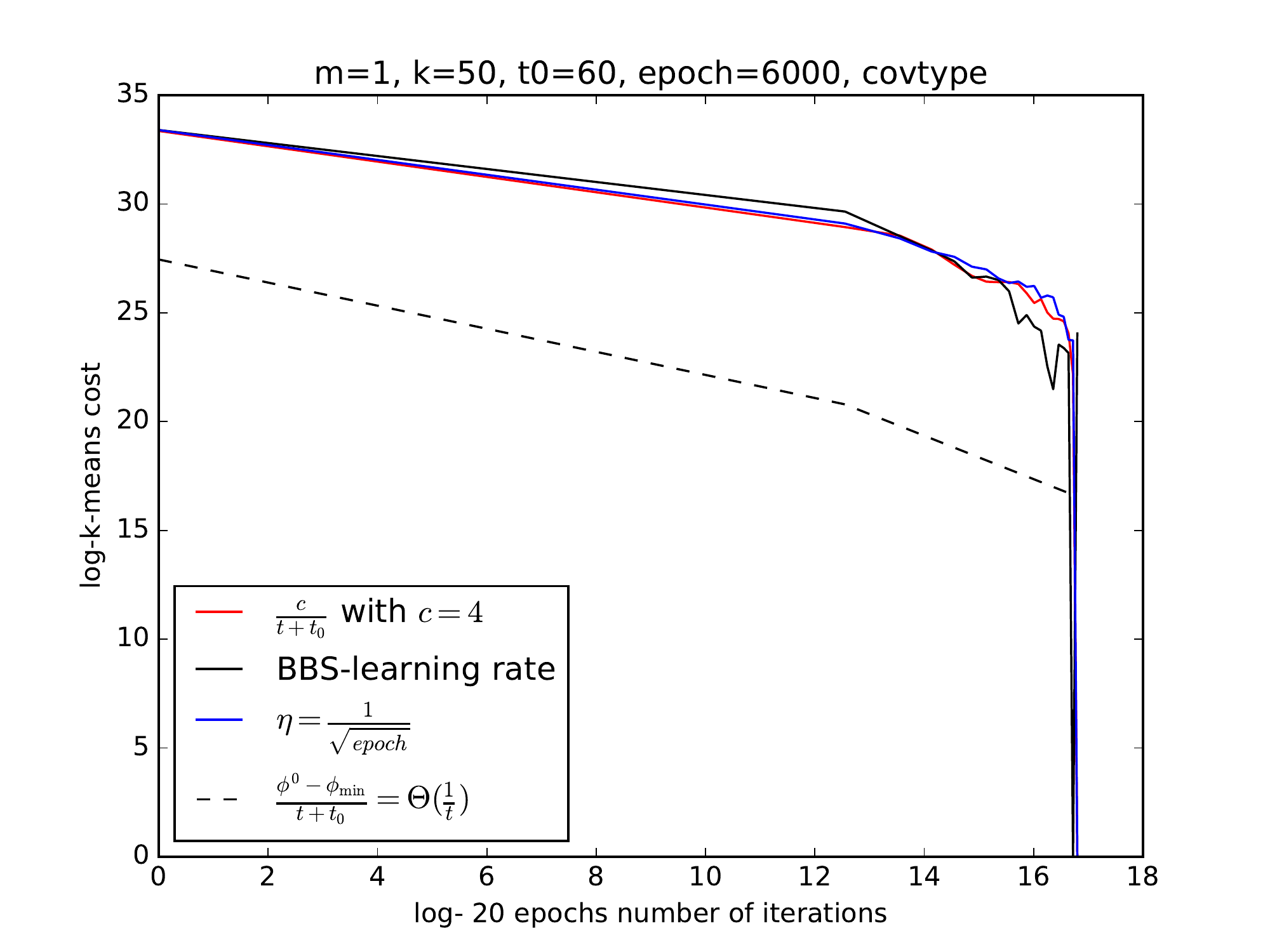}
\end{subfigure}
\begin{subfigure}{.33\textwidth}
  \centering
  \includegraphics[width=\linewidth]{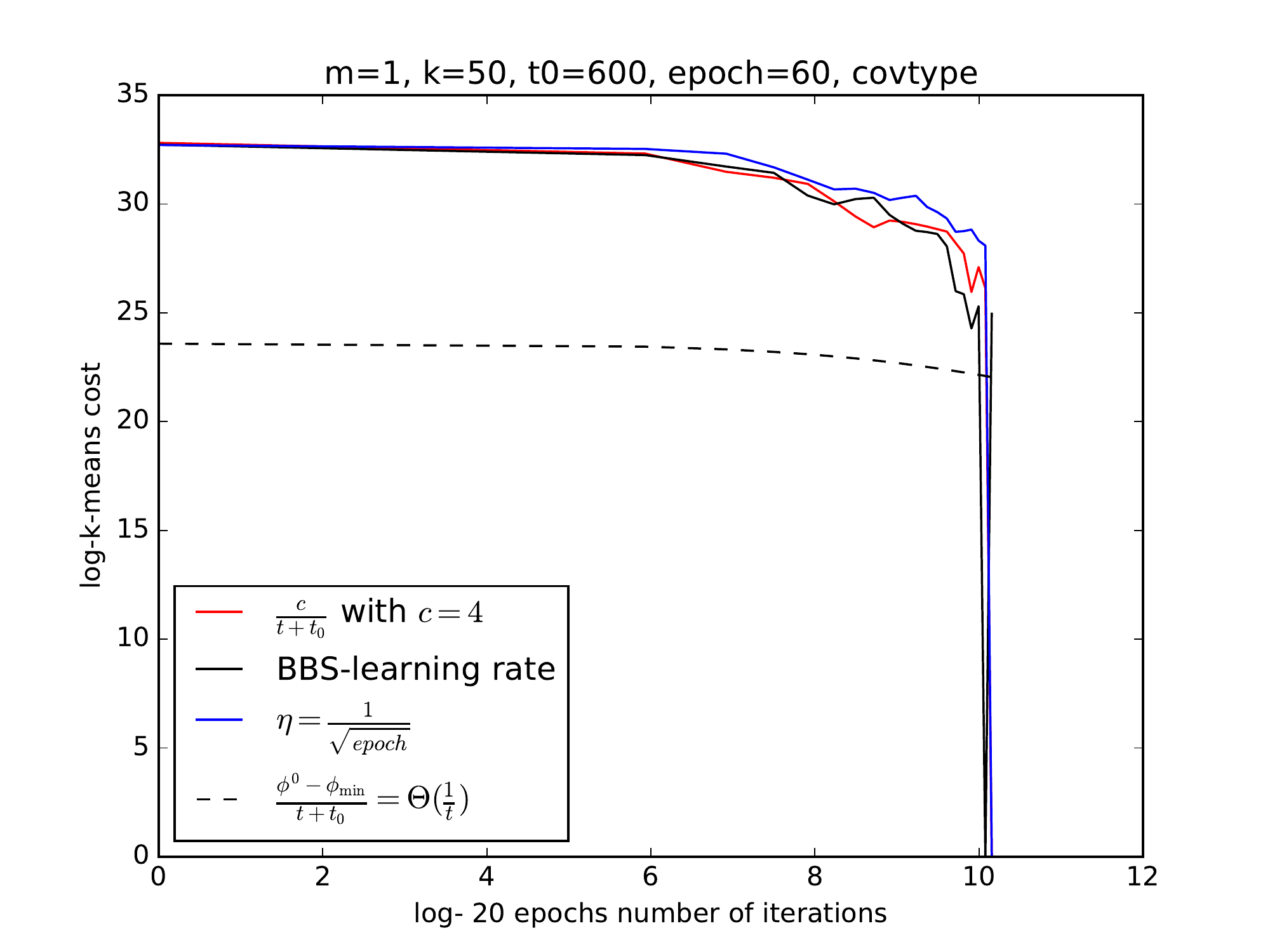}
\end{subfigure}%
\begin{subfigure}{.33\textwidth}
  \centering
  \includegraphics[width=\linewidth,]{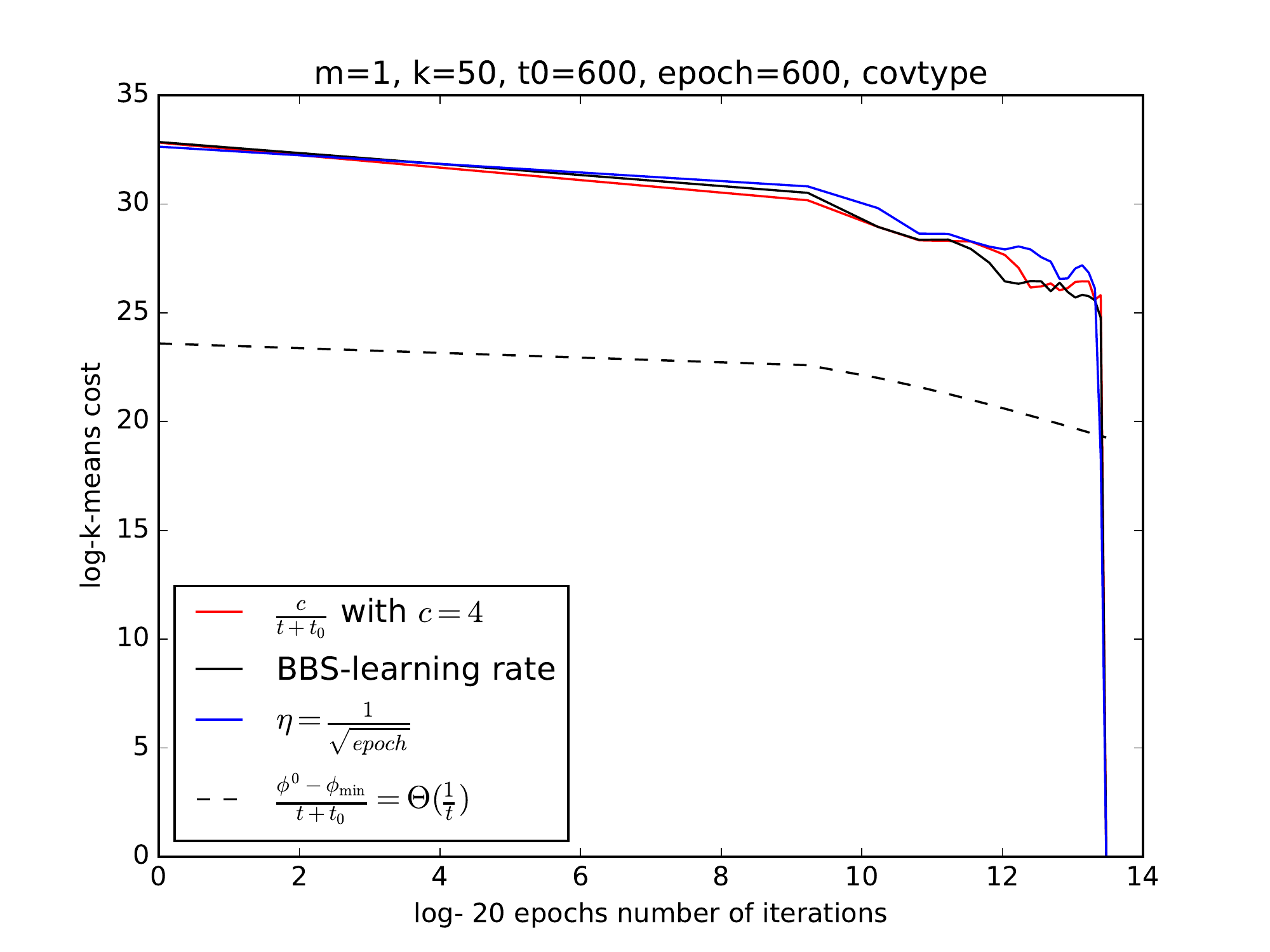}
\end{subfigure}
\begin{subfigure}{.33\textwidth}
  \centering
  \includegraphics[width=\linewidth]{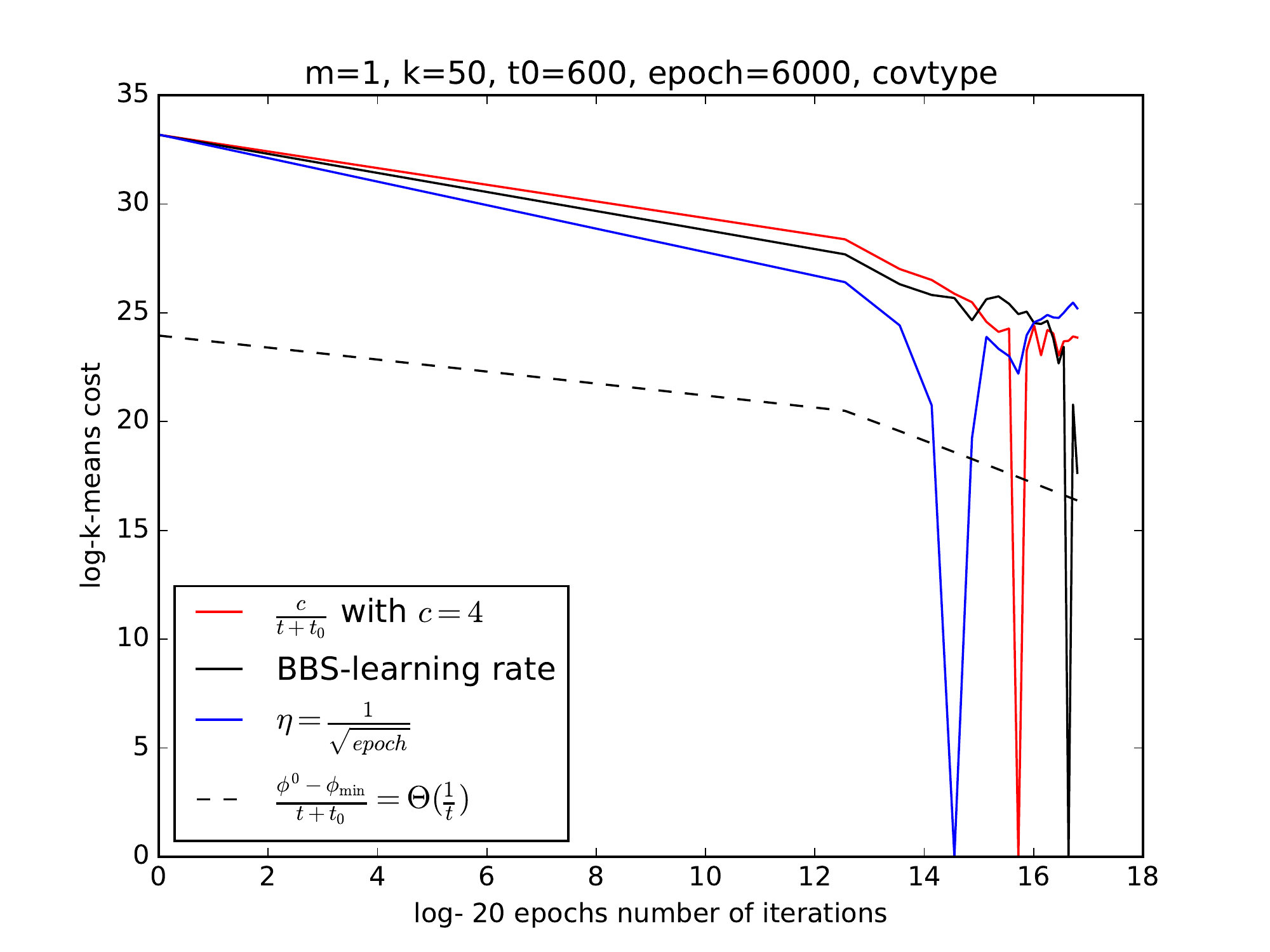}
\end{subfigure}
\begin{subfigure}{.33\textwidth}
  \centering
  \includegraphics[width=\linewidth]{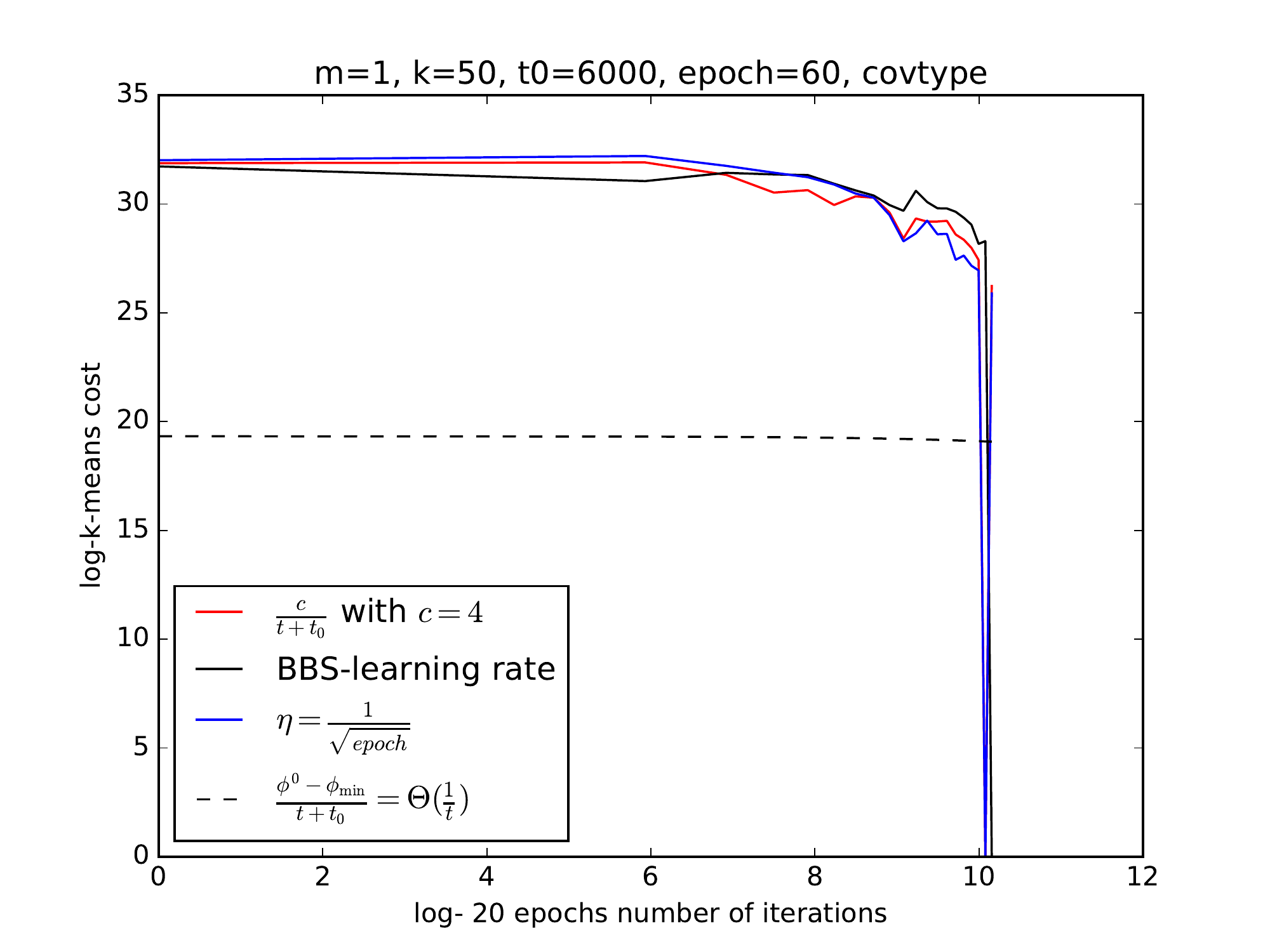}
\end{subfigure}%
\begin{subfigure}{.33\textwidth}
  \centering
  \includegraphics[width=\linewidth,]{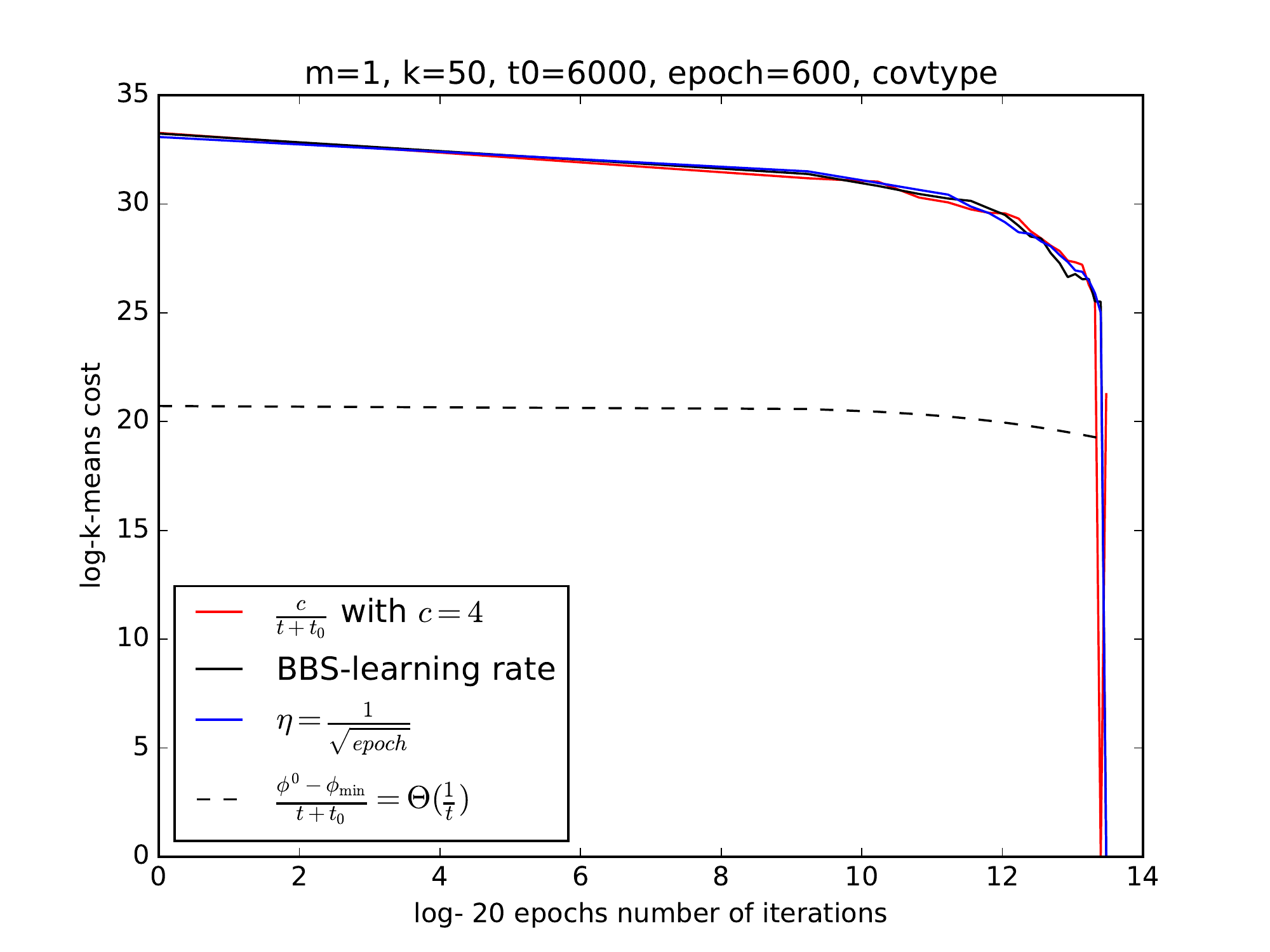}
\end{subfigure}
\begin{subfigure}{.33\textwidth}
  \centering
  \includegraphics[width=\linewidth]{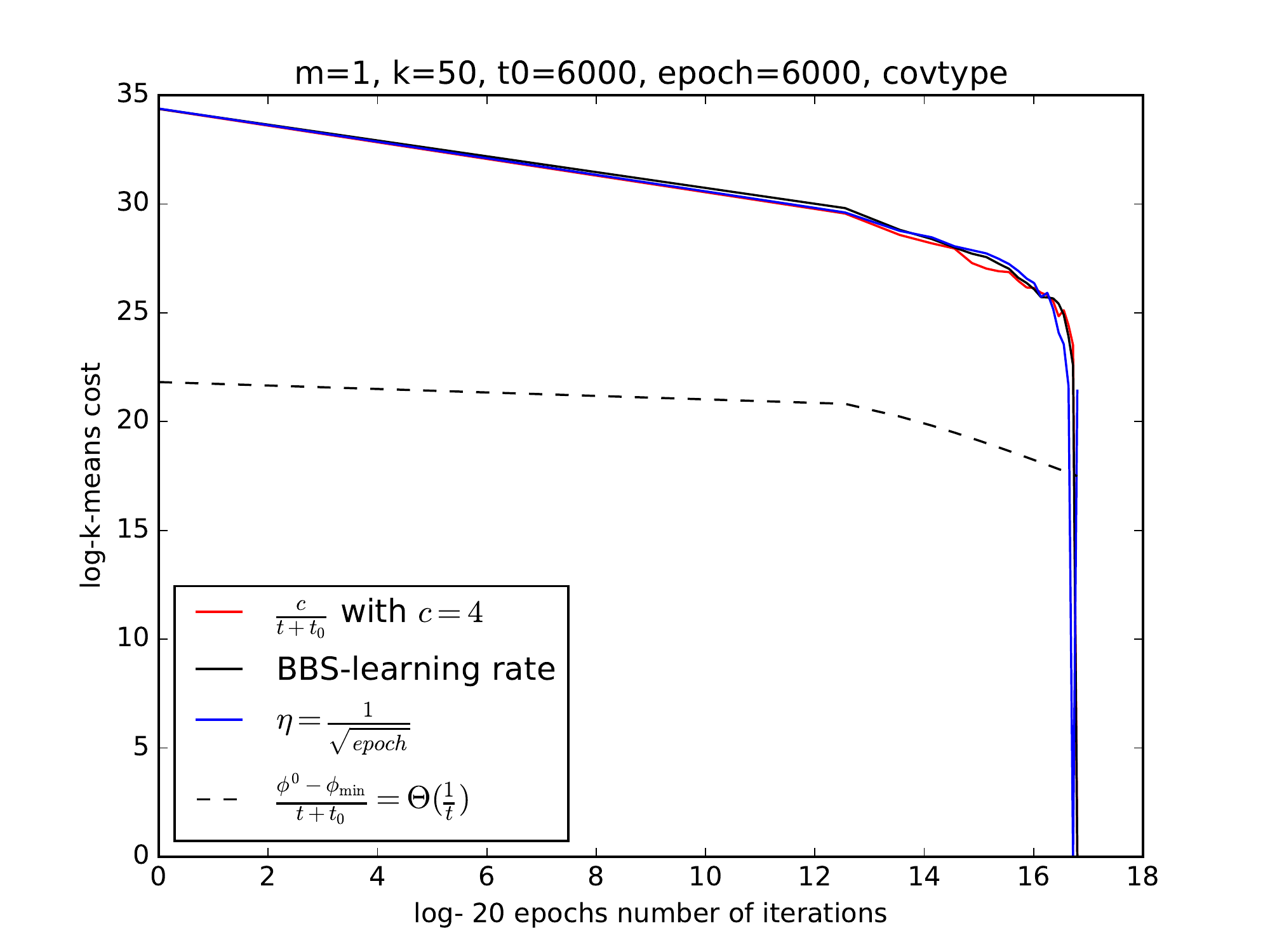}
\end{subfigure}
\vspace{-0.1cm}
\caption{Experiments on \texttt{covtype}}
\label{exp2:covtype}
\end{figure*}
\begin{figure*}[t]
\vspace{-1cm}
\begin{subfigure}{.33\textwidth}
  \centering
  \includegraphics[width=\linewidth]{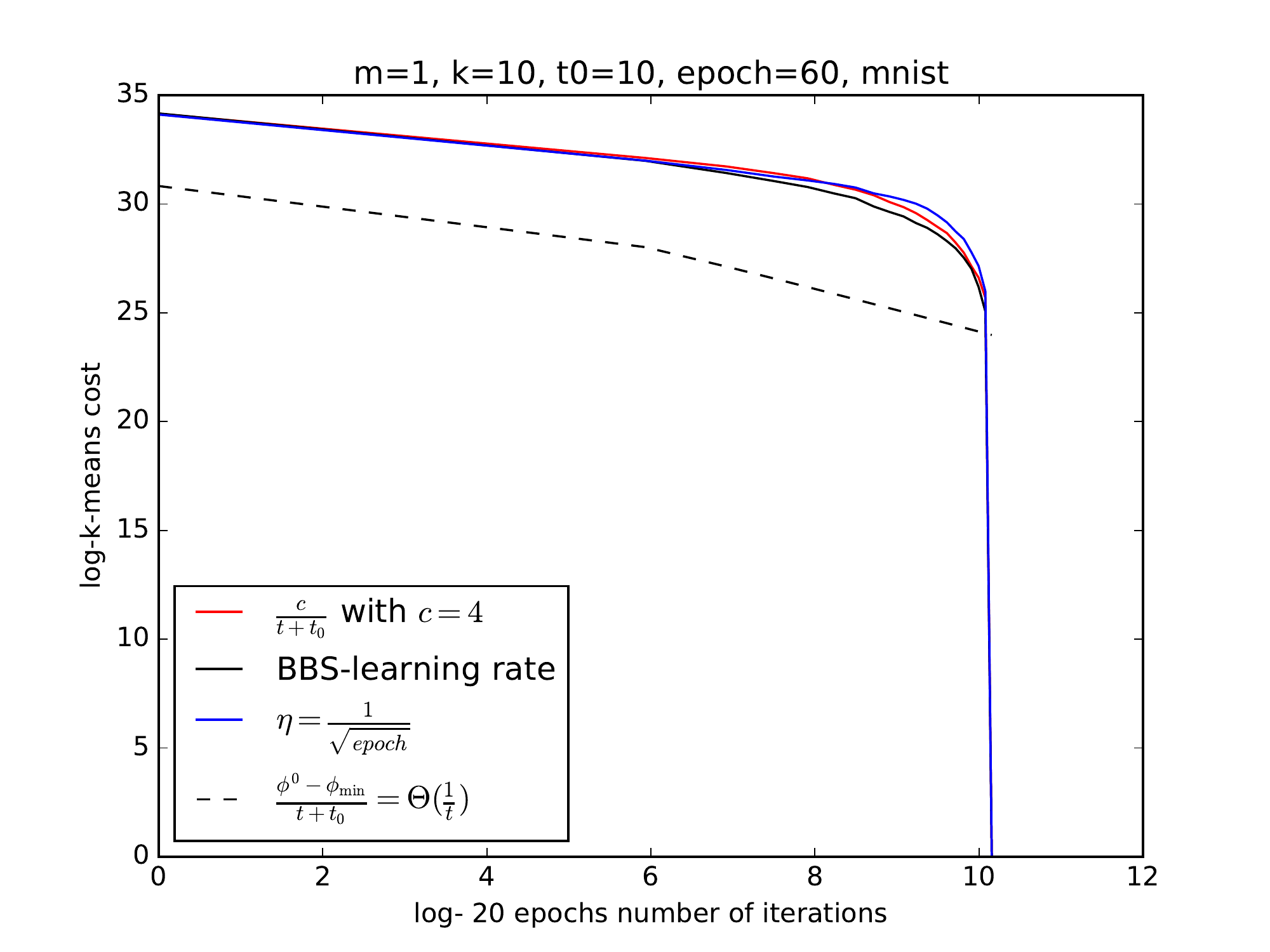}
\end{subfigure}%
\begin{subfigure}{.33\textwidth}
  \centering
  \includegraphics[width=\linewidth,]{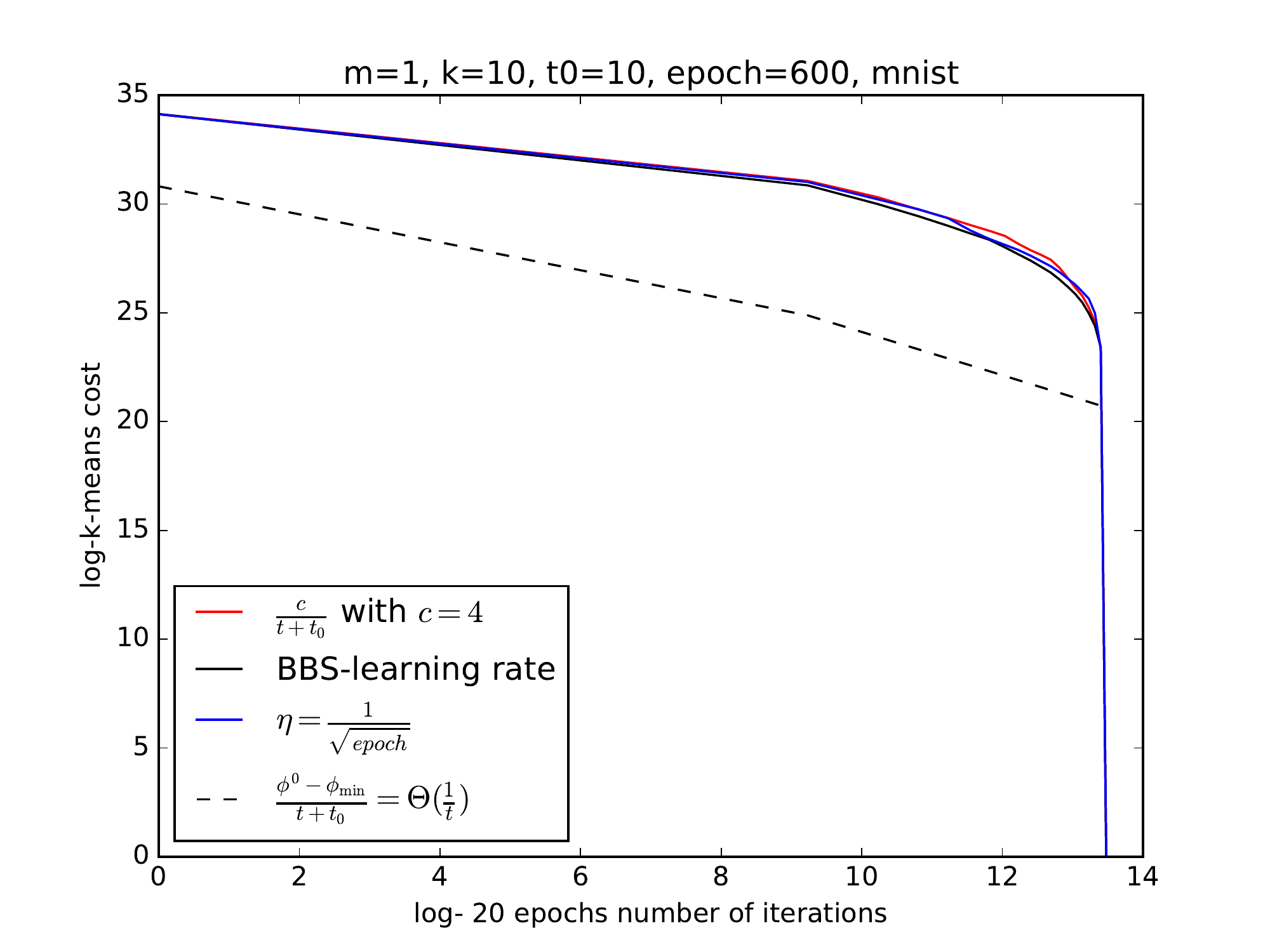}
\end{subfigure}
\begin{subfigure}{.33\textwidth}
  \centering
  \includegraphics[width=\linewidth]{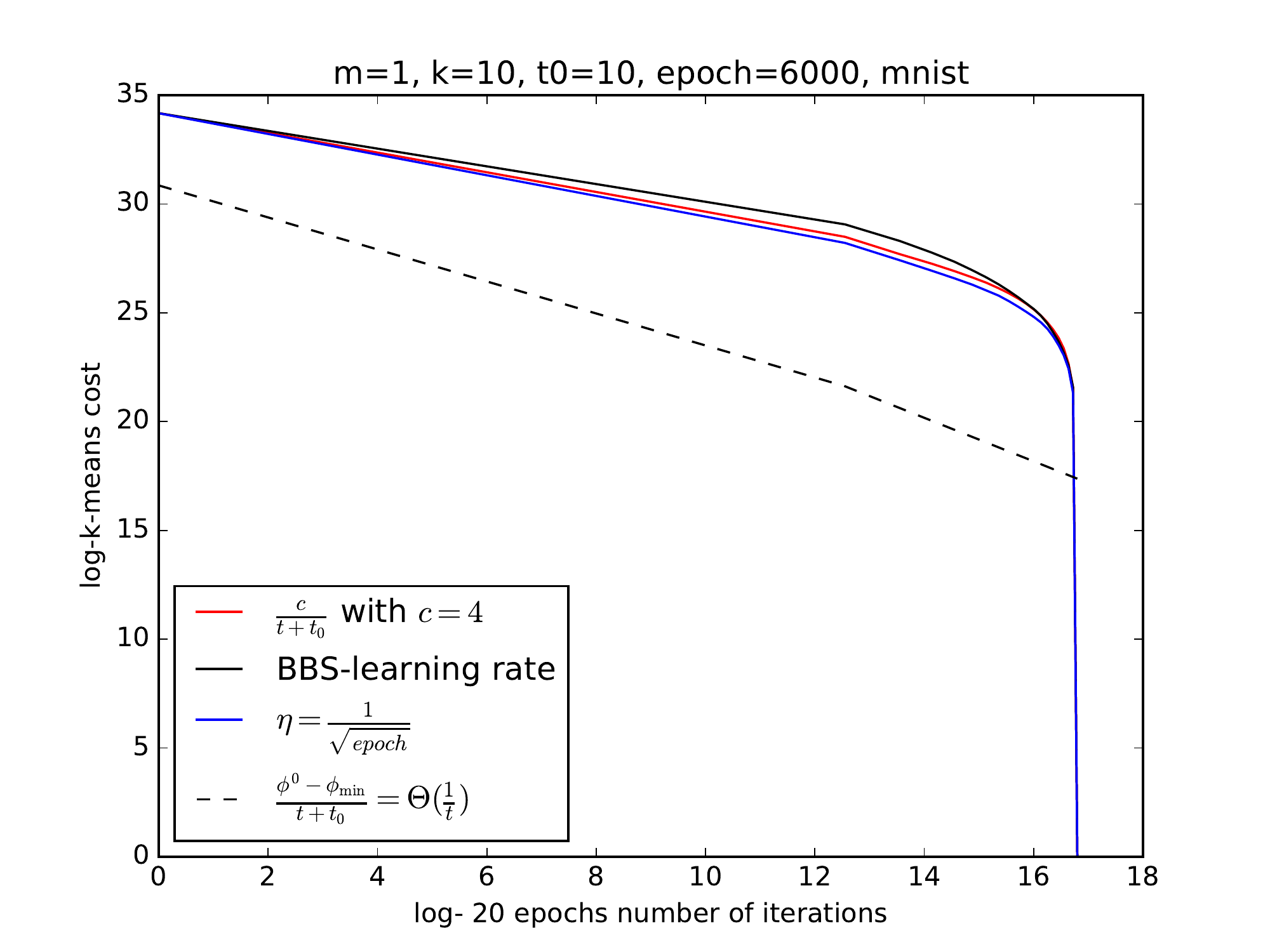}
\end{subfigure}
\begin{subfigure}{.33\textwidth}
  \centering
  \includegraphics[width=\linewidth]{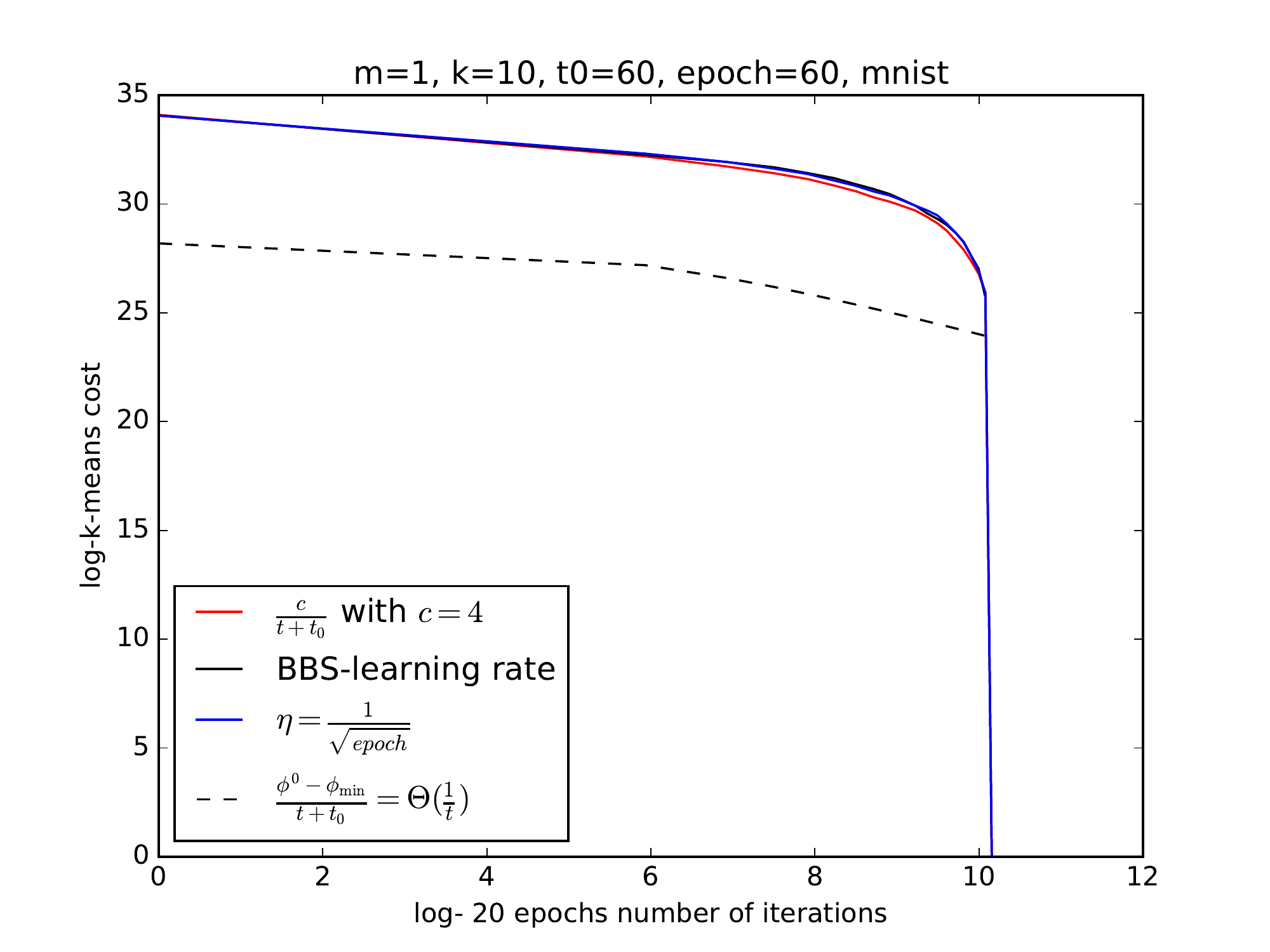}
\end{subfigure}%
\begin{subfigure}{.33\textwidth}
  \centering
  \includegraphics[width=\linewidth,]{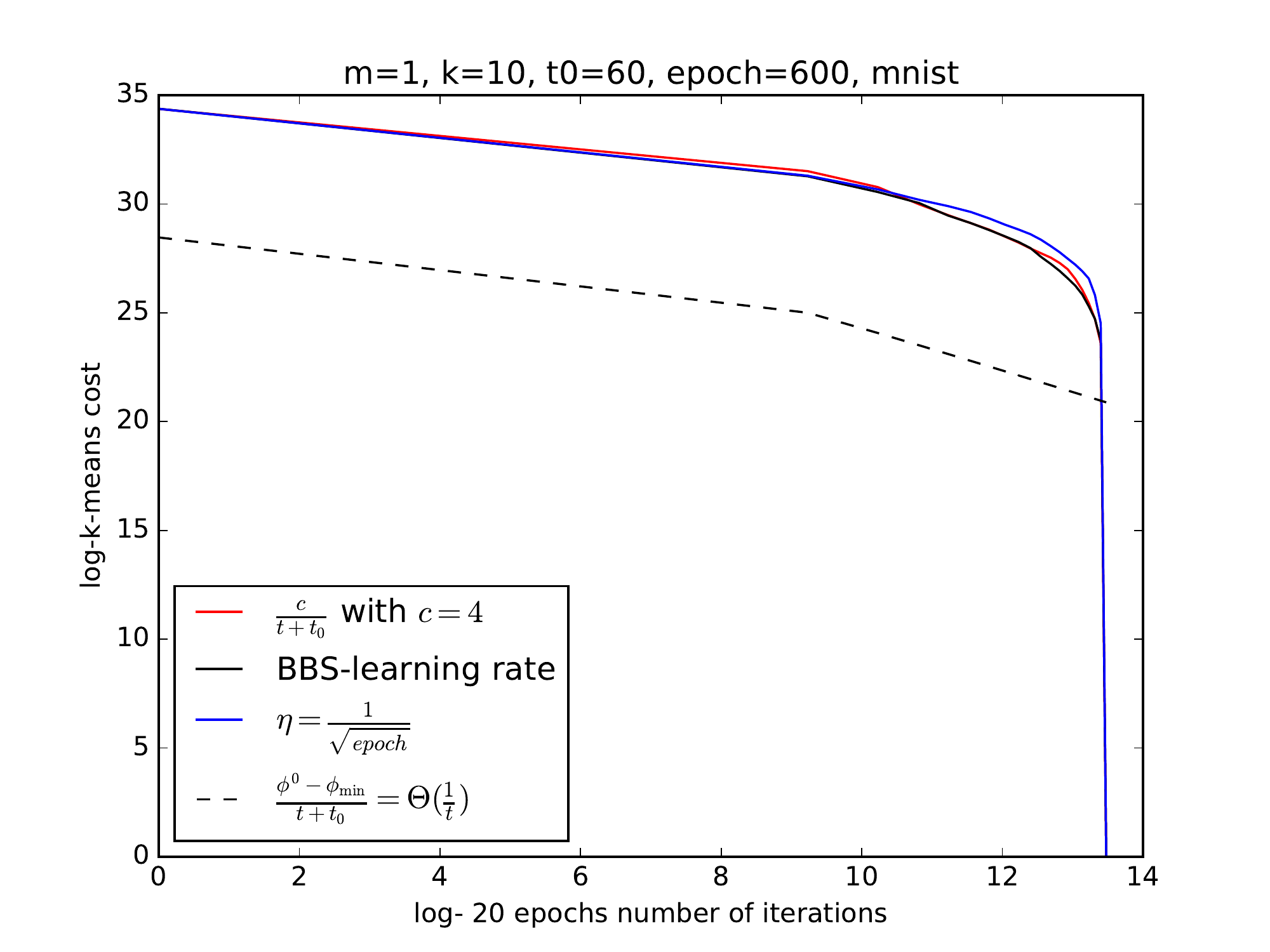}
\end{subfigure}
\begin{subfigure}{.33\textwidth}
  \centering
  \includegraphics[width=\linewidth]{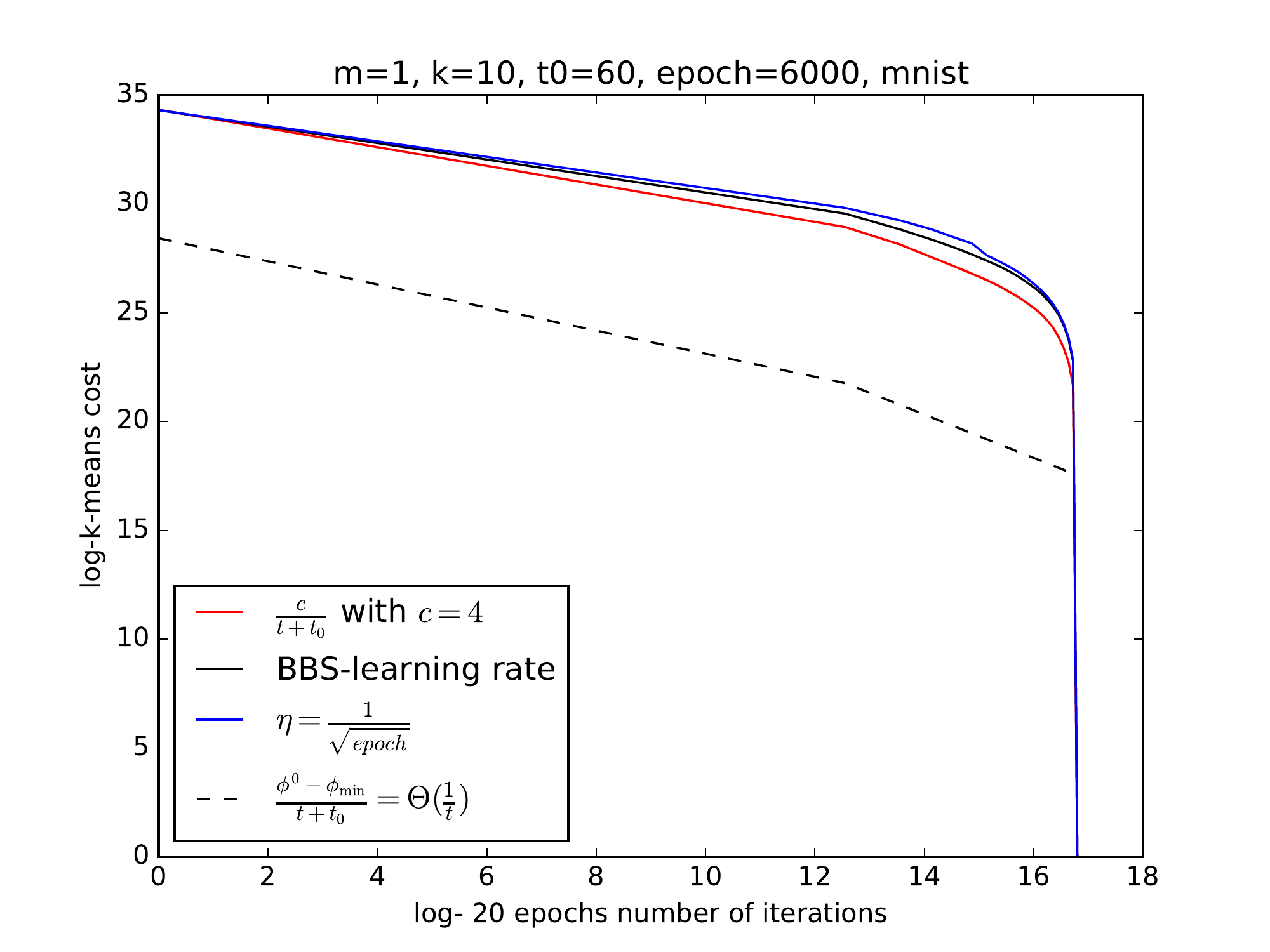}
\end{subfigure}
\begin{subfigure}{.33\textwidth}
  \centering
  \includegraphics[width=\linewidth]{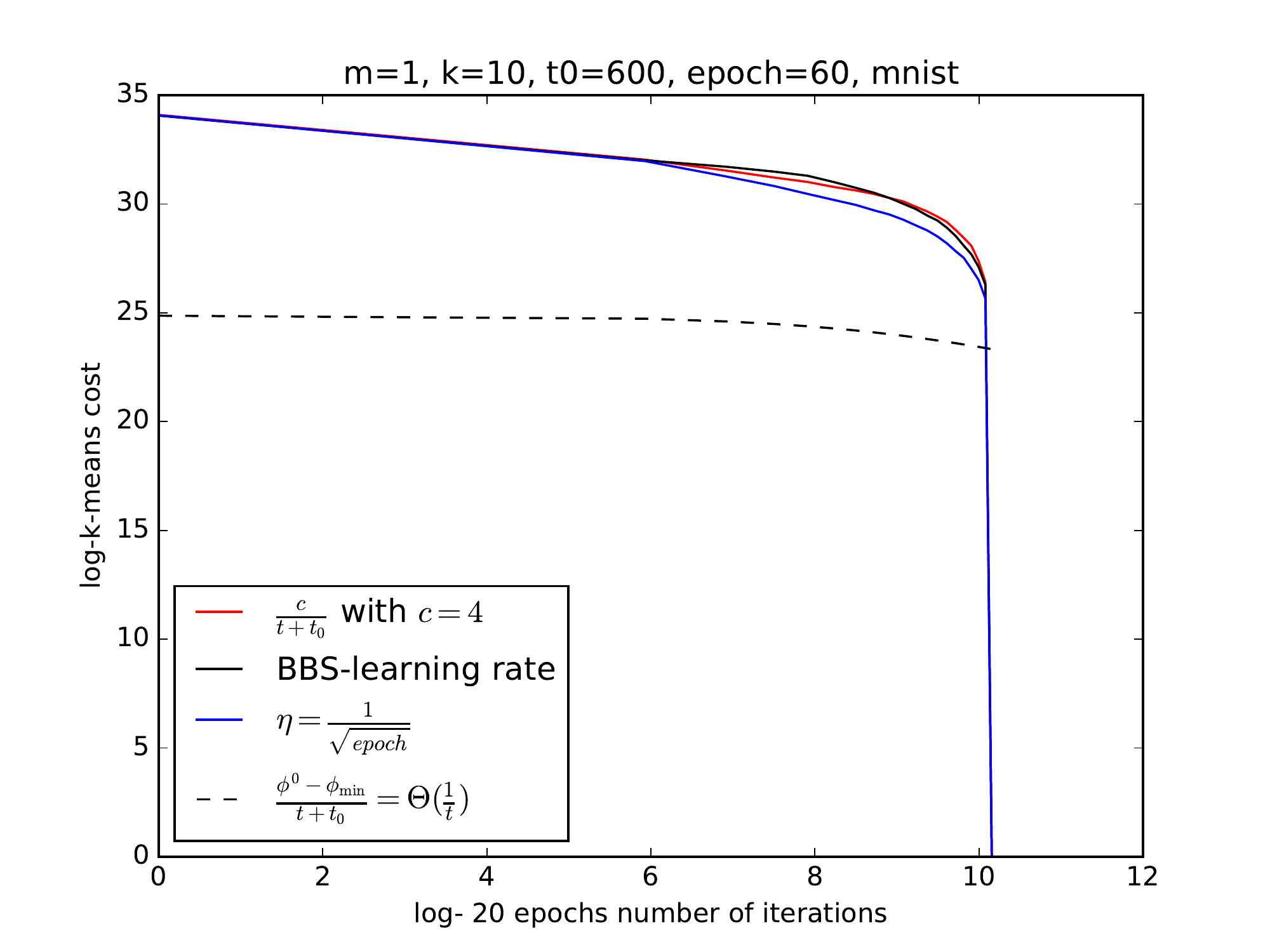}
\end{subfigure}%
\begin{subfigure}{.33\textwidth}
  \centering
  \includegraphics[width=\linewidth,]{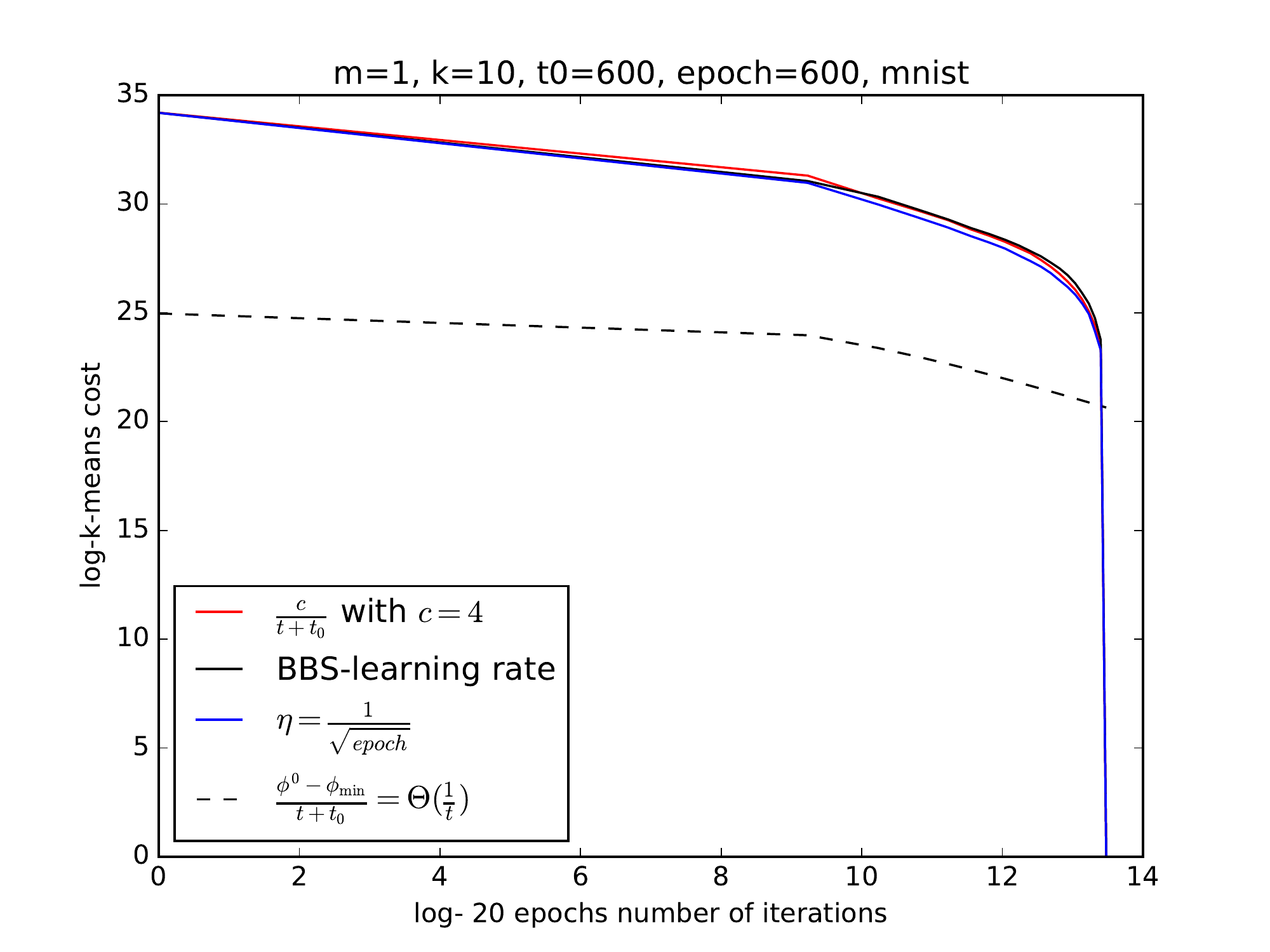}
\end{subfigure}
\begin{subfigure}{.33\textwidth}
  \centering
  \includegraphics[width=\linewidth]{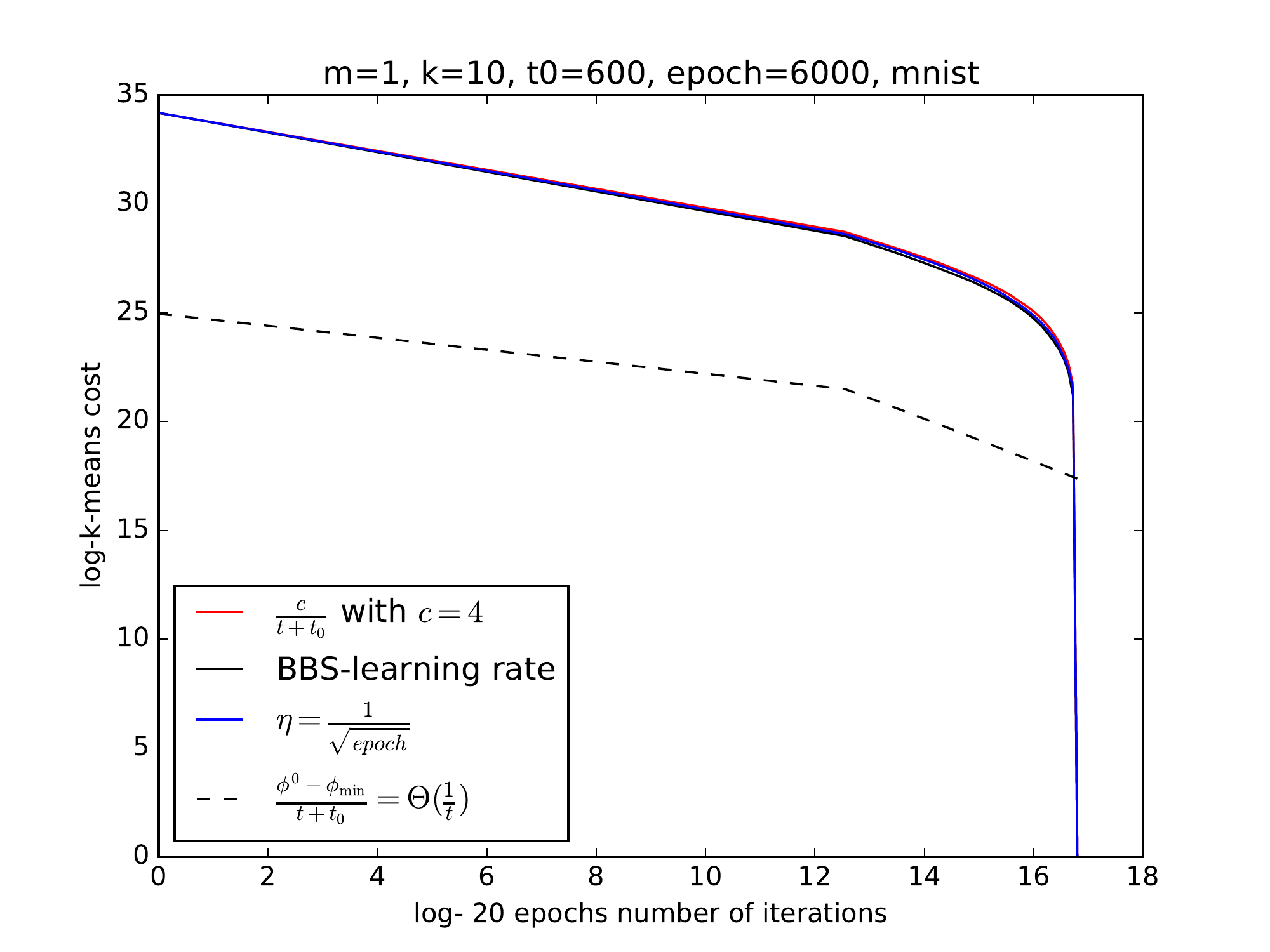}
\end{subfigure}
\caption{Experiments on \texttt{mnist}}
\label{exp2:mnist}
\end{figure*}
\vspace{-0.3cm}
\section{Discussion and open problems}
This work provides the first analysis of the convergence rate of stochastic $k$-means, but several questions remain unanswered.
First, our analysis applies to the flat learning rate in \eqref{learning_rate} while adaptive learning rate in \eqref{defn:adaptive_rate} is more common in practice. From our experiments, we conjecture that $O(\frac{1}{t})$ convergence can also be attained in the latter case. As discussed in Section \ref{sec:martingale}, the key question is whether we can show the adaptive rate, a random quantity that depends on all information prior $t$, is of order $\Theta(\frac{1}{t})$.
Second, we provide two examples of assumptions that imply Lipschitzness of $v$. Can we find other assumptions? In particular, is there a spectrum of assumptions in-between our Assumption (A) and (B) that imply different strength of Lipschitzness? 
We also believe further study of batch $k$-means can be made using our framework in Section \ref{sec:framework}.
For example, we observed that the radius of the neighborhood of an attractor (stationary point) is related to clusterability. Can we use this to relate the number of local attractors to clusterability of the dataset? 
In addition, if a stationary point has a large radius of attraction in $\{C\}$, then intuitively, two different random initializations will likely fall into this same neighborhood. Does this provide another angle to the clustering stability analysis \cite{luxburg:stability, bendavid_luxburg:stability}?
\small{
\bibliography{mbkm_nips16}
\bibliographystyle{plain}
}
\newpage
\section{Appendix A: supplementary materials to Section \ref{sec:framework}}
\begin{figure}
\includegraphics[height=\textwidth, width = 0.8\linewidth, angle = -90]{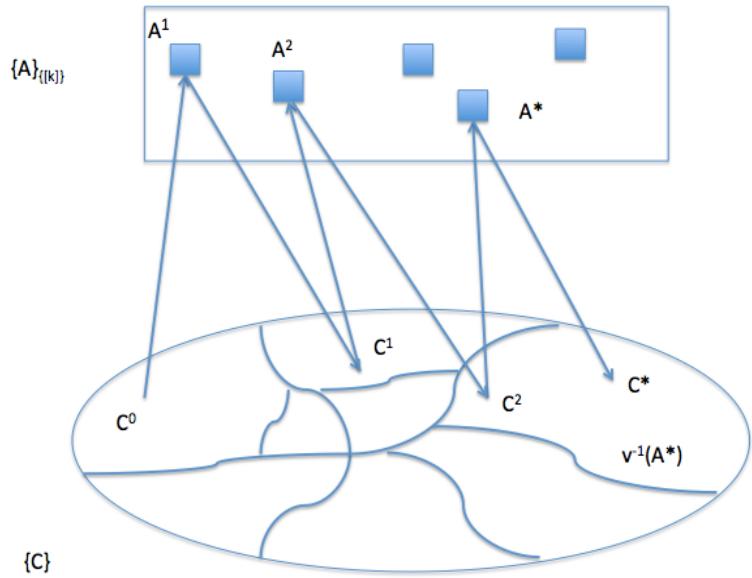}
\caption{An illustration of one run of batch $k$-means in the solution spaces: the rectangle represent the enlarged space of clusterings $\{A\}_{[k]}$ and the ellipse represent the centroidal space $\{C\}$, which is partitioned into equivalences classes.
The arrows represent k-means updates as mappings $v: \{C\}\rightarrow \{A\}_{[k]}$ 
and $m: \{A\}_{[k]}\rightarrow\{C\}$.
The algorithm starts at $C^0$ and stops at $C^*$ after three iterations, where $C^*=m(A^*)\in Cl(v^{-1}(A^*))$ .
}
\label{fig:solution_spaces}
\end{figure}
In this part of the Appendix, we provide details on the construction of our framework that are not included in Section \ref{sec:framework} due to space constraints.
\subsection*{Handling non-degenerate and boundary points}
One problem with $k$-means is it may produce degenerate solutions: if the solution $C^t$ has $k$ centroids, it is possible that data points are mapped to only $k^{\prime}<k$ centroids. To handle degenerate cases, starting with $|C^0|=k$, we
consider an enlarged clustering space $\{A\}_{[k]}$, which is the union of all $k^{\prime}$-clusterings 
with $1\le k^{\prime}\le k$. 
We use the pre-image $v^{-1}(A)\in \{C\}$ to denote the non-boundary points $C$ such that $v(C) = A$, i.e., these are the set of non-boundary points in the equivalence class induced by clustering $A$.
To include boundary points as well, we devise the operator $Cl(\cdot)$ as the ``closure'' of an equivalence class $v^{-1}(A)$, which includes all boundary points $C^{\prime}$ such that $A\in V(C^{\prime})\cap X$.

Using the above two extensions, we give the robust definition of stationary clusterings and stationary points, which we use in our analysis.
 \begin{defn}[Stationary clusterings]
We call $A^*$ a stationary clustering of $X$, if $m(A^*)\in Cl(v^{-1}(A^*))$.
We let $\{A^*\}_{[k]}\subset \{A\}_{[k]}$ denote the set of all stationary clusterings of $X$
with number of clusters $k^{\prime}\in [k]$.
\end{defn}
For each $A^*$, we define a matching centroidal solution $C^*$.
\begin{defn}[Stationary points]
For a stationary clustering $A^*$ with $k^{\prime}$ clusters, we define $C^{*}=\{c_r^*, r\in [k^{\prime}]\}$ to be a stationary point corresponding to $A^*$, so that $\forall A_{r}^{*}\in A^{*}$, $c_r^*:=m(A_r^{*})$.  
We let $\{C^*\}_{[k]}$ denote the corresponding set of all stationary points of $X$ with $k^{\prime}\in [k]$.
\end{defn}
With the robust definitions, Figure \ref{fig:solution_spaces} provides a visualization of batch $k$-means walking on $\{C\}$ (and $\{A\}_{[k]}$) as an iterative mapping $m\circ v$ ($v\circ m$, resp.).
In $\{C\}$, it jumps from one equivalence class to another until it stays in the same equivalence class in two consecutive iterations.

Now we extend $\Delta(\cdot,\cdot)$ to include the degenerate cases.
Fix a clustering $A$ with its induced $k$ centroids $C:=m(A)$, and another set of $k^{\prime}$-centroids $C^{\prime}$ ($k^{\prime}\ge k$) with its induced clustering $A^{\prime}$, if
$|A^{\prime}|=|A|=k$ (this means if $k^{\prime}>k$, then $C^{\prime}$ has at least one degenerate centroid), then
we can pair the subset of non-degenerate $k$ centroids in $C^{\prime}$ with those in $C$, and ignore the degenerate centroids.
Under this condition, we can extend Definition \ref{defn:centroidal_dist} to include degenerate solutions as well, provided $C=m(A)$ for some clustering $A$, which is always satisfied in our subsequent analysis. 

\subsection*{A sufficient condition for the local convergence of batch $k$-means}
We show batch $k$-means algorithm has geometric convergence in the local neighborhood of a stable stationary point in the solution space.
\begin{proof}[\textbf{Proof of Lemma \ref{lm:stable_stat}}]
Without loss of generality, we let $\pi(r)=r,\forall r\in [k]$.
Let 
$
\rho_{out}^r:=\frac{|\cup_{s\ne r}(A_s\cap A_r^*)|}{n^*_r}
$, and
$
\rho_{in}^r:=\frac{|\cup_{s\ne r}(A_r\cap A_s^*)|}{n^*_r}
$;
let $\rho_{\max}:=\max_r\frac{|A_{r}\triangle A_r^*|}{n^*_r}$.
Clearly,
$
(\rho_{out}^r+\rho_{in}^r)= \frac{|A_{r}\triangle A_r^*|}{n_r^*}
\le \rho_{\max},
$ by our definition. 
Now, similar to \cite{tang_montel:aistats16}, we can get
$
\|m(A_r)-c_r^*\|=\|\frac{(1-\rho_{out}^r)n^*_r m(A_r\cap A_r^*)+\sum_{s\ne r}\sum_{x\in A_r\cap A_s^*}x}{(1-\rho_{out}^r+\rho_{in}^r)n^*_r}-c_r^*\|
\le
\frac{1-\rho_{out}}{1-\rho_{out}^r+\rho_{in}^r}\|m(A_r\cap A_r^{*})-c_r^*\|
+\frac{\|\sum_{s\ne r}\sum_{x\in A_r\cap A_s^*}x-c_r^*\|}{(1-\rho_{out}^r+\rho_{in}^r)n^*_r}
$
And as in \cite{tang_montel:aistats16}, we get
$
(1-\rho_{out})\|m(A_r\cap A_r^{*})-c_r^*\| 
\le
\frac{\sqrt{\rho_{out}^r \phi_r^*}}{\sqrt{n^*_r}}
$.
Now we bound the second term:
by Cauchy-Schwarz inequality,
$
\|\sum_{s\ne r}\sum_{x\in A_r\cap A_s^*}x-c_r^*\|^2 
\le
(\sum_{s\ne r}\sum_{x\in A_r\cap A_s^*}1^2)
(\sum_{s\ne r}\sum_{x\in A_r\cap A_s^*}\|x-c_r^*\|^2) 
=\rho_{in}^rn_r^*\sum_{s\ne r}\sum_{x\in A_r\cap A_s^*}\|x-c_r^*\|^2
$.
Thus, $\forall r\in [k]$,
$
\|m(A_r)-c_r^*\|^2
\le 4\frac{\rho_{out}^r\phi_r^*}{n^*_r}+4\frac{\rho^r_{in}\sum_{s\ne r}\sum_{x\in A_r\cap A_s^*}\|x-c_r^*\|^2}{n^*_r}
$, where we use the assumption that $\rho_{\max}< \frac{1}{4}< 1-\frac{1}{\sqrt{2}}$.
Summing over all $r$,
$
\sum_r n^*_r\|m(A_r)-c_r^*\|^2
\le 
4\rho_{\max}\sum_r (\phi_r^*+\sum_{s\ne r}\sum_{x\in A_r\cap A_s^*}\|x-c_r^*\|^2)
$.
By Lemma \ref{lm:diff}, 
$\sum_r\sum_{s\ne r}\sum_{x\in A_r\cap A_s^*}\|x-c_r^*\|^2$ can be upper bounded
by 
$
\phi(C^{\prime})+\sum_rn_{r}\|m(A_{r})-c_r^{*}\|^2
=\phi(C^{\prime})+\sum_r(1-\rho_{out}^r+\rho_{in}^r)n^*_{r}\|m(A_{r})-c_r^{*}\|^2
\le
\phi(C^{\prime})+(1+\rho_{\max})\sum_rn^*_{r}\|m(A_{r})-c_r^{*}\|^2 
$. Substituting this into the previous inequality, we have
$
(1-4\rho_{\max}(1+\rho_{\max}))\sum_r n^*_r\|m(A_r)-c_r^*\|^2
\le 
4\rho_{\max}(\phi^*+\phi(C^{\prime}))
$.
Thus,
$
\sum_r n^*_r\|m(A_r)-c_r^*\|^2
\le 
\frac{\rho_{\max}}{1-4\rho_{\max}(1+\rho_{\max})}[\phi^*+\phi(C^{\prime})]
$. 
By our assumption, $\rho_{\max}\le \frac{b}{5b+4(1+\frac{\phi(C)}{\phi^*})}<\frac{1}{4}$, so
$
\frac{\rho_{\max}}{1-4\rho_{\max}(1+\rho_{\max})}
\le
\frac{\rho_{\max}}{1-5\rho_{\max}}
\le
\frac{b}{1+\frac{\phi(C)}{\phi^*}}
$, and
$
\frac{\rho_{\max}}{1-4\rho_{\max}(1+\rho_{\max})}[\phi^*+\phi(C^{\prime})]
\le b\phi^*
$, since $\phi(C^{\prime})\le \phi(C)$ (equality holds if $C$ is a stationary point).
\end{proof}
\begin{lm}
\label{lm:diff}
Fix any target clustering $C^*$, and another clustering $C$ with a matching
$\pi: [k]\rightarrow [k]$.
Let $C^{\prime}:=\{m(A_r),r\in [k]\}$. Then
\begin{align*}
&\sum_r\sum_{s\ne r}\sum_{x\in A_{\pi(r)}\cap A_s^*}\|x-c_r^*\|^2\\
&\le 
\phi(C^{\prime})-\sum_r\phi(c_r^*;A_{\pi(r)}\cap A_r^*)+\sum_rn_{r}\|m(A_r)-c_r^{*}\|^2
\end{align*}
\end{lm}
\begin{proof}
Without loss of generality, we let $\pi(r)=r$.
\begin{eqnarray*}
\phi(C^{\prime})-\phi(C^{*})
=\sum_r\sum_{x\in A_r}\|x-c_r^{*}\|^2 - \sum_r\sum_{x\in A_r^{*}}\|x-c_r^{*}\|^2\\
+ \sum_r\sum_{x\in A_r}\|x-m(A_r)\|^2-\sum_r\sum_{x\in A_r}\|x-c_r^{*}\|^2
\end{eqnarray*} 
So
$
\sum_r\sum_{x\in A_r}\|x-c_r^{*}\|^2 - \sum_r\sum_{x\in A_r^{*}}\|x-c_r^{*}\|^2
=\phi(C^{\prime})-\phi(C^{*})
-\sum_r\sum_{x\in A_r}\|x-m(A_r)\|^2+\sum_r\sum_{x\in A_r}\|x-c_r^{*}\|^2
\le
\phi(C)-\phi(C^{*})
+\sum_rn_r\|m(A_r)-c_r^{*}\|^2
$. 
Now, we claim
$
\sum_r\sum_{x\in A_r}\|x-c_r^{*}\|^2 - \sum_r\sum_{x\in A_r^{*}}\|x-c_r^{*}\|^2
=
\sum_r\sum_{s\ne r}\sum_{x\in A_r\cap A_s^*}\{\|x-c_r^*\|^2 - \|x-c_s^*\|^2\}
$.
This is because
we can enumerate $x$ using clustering $\cup_r A_r$:
for each $x\in A_r$, either $x\in A_r\cap A_r^*$, then 
$\|x-c_r^*\|^2-\|x-c_r^*\|^2=0$, or
$x\in A_r\cap A_s^*$ for some $s\ne r$, which means the difference is
$\|x-c_r^*\|^2-\|x-c_s^*\|^2$ (and this term is positive by optimality of clustering $\cup_r A_r^*$ fixing $\{c_r^*\}$).
Thus,
$
\sum_r\sum_{s\ne r}\sum_{x\in A_r\cap A_s^*}\|x-c_r^*\|^2
=\sum_r\sum_{x\in A_r}\|x-c_r^{*}\|^2 - \sum_r\sum_{x\in A_r^{*}}\|x-c_r^{*}\|^2
+\sum_r\sum_{s\ne r}\sum_{x\in A_r\cap A_s^*}\|x-c_s^*\|^2
\le
\phi(C^{\prime})-\phi(C^{*})
+\sum_rn_r\|m(A_r)-c_r^{*}\|^2
+\sum_r\sum_{s\ne r}\sum_{x\in A_r\cap A_s^*}\|x-c_s^*\|^2
=
\phi(C^{\prime})-\sum_r\phi(c_r^*;A_{r}\cap A_r^*)+\sum_rn_{r}\|m(A_r)-c_r^{*}\|^2
$, where the last equality is by observing that
$
\phi(C^*)=\sum_r\sum_{A_r\cap A_r^*}\|x-c_r^*\|^2+\sum_r\sum_{s\ne r}\sum_{x\in A_r\cap A_s^*}\|x-c_s^*\|^2
$.
\end{proof}
\subsection{Local Lipschitzness and clusterability}
\begin{proof}[\textbf{Proof of Lemma \ref{lm:equivalent_boundary}}]
``$1\implies 2$'' obviously holds since $\|x-c_r\|=\|x-c_s\|$ if and only if $\|\bar{x}-c_r\| - \|\bar{x}-c_s\|$.
``$2\implies 3$'': let $A\in V(C)\cap X$ be the clustering achieving the zero margin, and consider $x\in A_r\cup A_s$ s.t. $\|\bar{x}-c_r\| - \|\bar{x}-c_s\|$; without loss of generality, assume $x\in A_r$ according to clustering $A$, and define $A^{\prime}$ to be the same clustering as $A$ for all points in $X$ but $x$, where it assigns $x$ to $A_s$. Then $A^{\prime}\in V(C)\cap X$ and $|V(C)\cap X|\ge 2>1$.
``$3\implies 1$'': Suppose otherwise. Then every point $x$ has a unique center that minimizes its distance to it, which means the clustering determined by $V(C)\cap A$ is unique. A contradiction.
\end{proof}
\begin{lm}\label{lm:stat_sols}
If $C^*\in \{C^*\}$, then  
$C^*=m(A^*)$, where $A^*\in \{A^*\}$ and $A^*=v(C^*)$. 
\end{lm}
\begin{proof}
By definition of stationary points, $C^*=m\circ v(C^*)$. Let $A=v(C^*)$, then
$m(A)=C^*$ and $v\circ m(A)=v(C^*)=A$. Thus $A\in \{A^*\}$ by definition of a stationary clustering.
\end{proof}
\begin{lm}
\label{lm:closure}
Fix a clustering $A=\{A_1,\dots, A_k\}$, and let $C\in v^{-1}(A)$. Then 
$\exists \delta >0$ such that the following statement holds: 
\begin{eqnarray}
\label{statement}
\mbox{For~} C^{\prime} s.t.~ \Delta(\cdot,\cdot) \mbox{~is defined~},
\Delta(C^{\prime},C)<\delta \implies C^{\prime}\in v^{-1}(A)
\end{eqnarray}
\end{lm}
\begin{proof}
Since $C$ is not a boundary point, $\forall x\in A_r, r\in [k]$, 
$$
\|x - c_{r}\|<\|x-c_{s}\|, \forall s\ne r
$$
So we can choose $\delta > 0$ s.t. $\forall x\in A_r$, $\forall r\in [k], s\ne r$,
$$
\|x - c_{r}\|<\|x-c_{s}\|-2\sqrt{\delta}
$$
Let $\pi^*$ be a permutation such that $\Delta(C^{\prime},C)$ is defined.
We have $\forall x\in A_r, r\in[k], s\ne r$,
\begin{eqnarray*}
\|x-c^{\prime}_{\pi^{*}(s)}\| - \|x-c^{\prime}_{\pi^{*}(r)}\|
\ge
\|x-c_{s}\|-\|c^{\prime}_{\pi^{*}(s)}-c_{s}\|\\
-(\|x-c_{r}\|+\|c_{r}-c^{\prime}_{\pi^*(r)}\|)
>
\|x-c_s\|-\|x-c_r\|-2\sqrt{\delta} 
\ge 0
\end{eqnarray*}
where the second inequality is by the fact that
$$
\max_{r}\|c^{\prime}_{\pi^{*}(r)}-c_r\|^2\le \Delta(C^{\prime},C)<\delta
$$
Therefore,
$V(C^{\prime})\cap X = A$, i.e., $C^{\prime}\in v^{-1}(A)$.
\end{proof}
\begin{lm}
\label{lm:stat_stab1} 
Suppose $\forall C^*\in \{C^*\}_{[k]}$, $C^*$ is not a boundary point (i.e., suppose Assumption (A) holds). 
Let $C=m(A^{\prime})\notin \{C^*\}_{[k]}$ for some $A^{\prime}\in \{A\}$ and let $C^{\prime}\in Cl(v^{-1}(A^{\prime}))$, then 
$\exists \delta>0$ s.t. 
$\Delta(C^{\prime},C)\ge \delta$.
\end{lm}
\begin{proof}
We prove the lemma by contradiction: suppose $\forall \delta>0$, $\exists C^{\prime}$ s.t. $C^{\prime}\in Cl(v^{-1}(A^{\prime}))$ and $\Delta(C^{\prime},C)<\delta$.
First, we claim that for $\delta$ sufficiently small, $C$ must be a boundary point: suppose otherwise, then by Lemma \ref{lm:closure}, $v(C^{\prime}) = v(C) = A^{\prime}$, contradicting the fact that $C\notin \{C^*\}_{[k]}$. 
Let $A\in V(C)\cap X$.
Since $C$ is a boundary point, $\exists r, s$ and $x\in A_r\cup A_s$ s.t.  
$$
\|x-c_r\|=\|x-c_s\|
$$
Now, we choose $\delta>0$ to be sufficiently small so that for any $A^{\prime}\in V(C^{\prime})\cap X$, clustering $A^{\prime}$ only differs from $A$ on the assignment of these points sitting on the bisector.
This implies $C \in Cl(v^{-1}(A^{\prime}))$, which implies $C$ is a boundary stationary point, a contradiction.
\end{proof}
\begin{lm}
\label{lm:stat_stab}
If $\forall C^*\in \{C^*\}_{[k]}$, $C^*$ is a non-boundary stationary point, that is, $C^*:=m(A^*)\in v^{-1}(A^*)$. 
Then $\exists r_{\min}>0$ such that $\forall C^*\in \{C^*\}_{[k]}$, $C^*$ is a $(r_{\min},0)$-stable stationary point.
\end{lm}
\begin{proof}
Fix any $k$ in the range of $[k]$ (we abuse the notation with the same $k$ here).
For any $C$ such that $\Delta(C, C^*)$ exists (i.e., $|C|=k^{\prime}\ge k=|C^*|$), we first show $\exists r^*>0$, such that the following statement holds:
$$
\Delta(C, C^*) < r^*\phi^* \implies C \in v^{-1}(A^*)
$$
Since $C^*$ is a non-boundary point,
there is a permutation $\pi_{o}$ of $[k]$ such that
$\forall x\in A_r, \forall r\in [k]$ and $\forall s\ne r$,
$$
\|x - c^*_{\pi_{o}(r)}\|<\|x-c^*_{\pi_{o}(s)}\|
$$
We choose $r^* > 0$ so that $\forall x\in A_r, \forall r\in [k], \forall s\ne r$,
$$
\|x - c^{*}_{\pi_{o}(r)}\|\le\|x-c^{*}_{\pi_{o}(s)}\|-2\sqrt{r^*\phi^*}, ~\forall r\in [k], s\ne r
$$
with equality holds for at least one triple of $(x, r, s)$.
Let $\pi^*$ be a permutation satisfying
$$
\pi^* = \arg\min_{\pi}\sum_{r\in [k]}n_r^*\|c_{\pi(r)}-c_r^*\|^2
$$
Let $\pi^{\prime}:=\pi^{*}\circ\pi_{o}$. We have $\forall (x,r,s)$ triples,
\begin{eqnarray*}
\|x-c_{\pi^{\prime}(s)}\| - \|x-c_{\pi^{\prime}(r)}\|\\
\ge
\|x-c^{*}_{\pi_{o}(s)}\|-\|c^{*}_{\pi_{o}(s)}-c_{\pi^{\prime}(s)}\|\\
-(\|x-c^{*}_{\pi_{o}(r)}\|+\|c^{*}_{\pi_{o}(r)}-c_{\pi^{\prime}(r)}\|)\\
>
\|x-c^*_{\pi_{o}(s)}\|-\|x-c^*_{\pi_{o}(r)}\|-2\sqrt{r^*\phi^*} \ge 0
\end{eqnarray*}
where the second inequality is by the fact that
\begin{eqnarray*}
\max_{r}\|c_{\pi^*(r)}-c^*_{r}\|^2\le \Delta(C,C^*)<r^*\phi^*\\
\implies
\max_{r}\|c_{\pi^*(r)}-c^*_{r}\| <\sqrt{r^*\phi^*}
\end{eqnarray*}
Since $\pi^{\prime}$ is the composition of two permutations of $[k]$, it is also a permutation of $[k]$, and 
$\forall r, s\ne r$, 
$\|x-c_{\pi^{\prime}(r)}\|<\|x-c_{\pi^{\prime}(s)}\|$, so $C\in v^{-1}(A^*)$.
Since by our definition, $r^*$ is unique for each $C^*$. 
Since $\{C^*\}_{[k]}$ is finite, taking the minimum over all such $r^*$, i.e., $r_{\min}:=\min_{C^*\in \{C^*\}_{[k]}}r^*$ completes the proof.
\end{proof}
The following is a restatement of Lemma \ref{lm:stat_stab2}, which is robust to degeneracy and boundary points.
\begin{lm}[\textbf{Restatement of Lemma \ref{lm:stat_stab2}}]
If $X$ is a general dataset, then $\exists r_{\min}>0$ s.t.
\begin{enumerate}
\item
$\forall C^*\in \{C^*\}_{[k]}$, $C^*$ is a $(r_{\min},0)$-stable stationary point.
\item
Let $m(A^{\prime})\notin \{C^*\}_{[k]}$ for some $A^{\prime}\in \{A\}$ and let $C^{\prime}\in Cl(v^{-1}(A^{\prime}))$, then 
$\Delta(C^{\prime},m(A))\ge r_{\min}\phi(m(A))$.
\end{enumerate}
\end{lm}
\begin{proof}
By Lemma \ref{lm:stat_stab}, $\exists r_{\min}^*>0$ s.t. $\forall C^*$, $C^*$ is $r_{\min}^*$-stable.
Furthermore, by Lemma \ref{lm:stat_stab1}, $\exists r_{\min}^{\prime}>0$ s.t.
$\forall C^*$, $\Delta(C^{\prime},m(A))\ge r_{\min}^{\prime}\phi(m(A))$.
Let $r_{\min}:=\min\{r_{\min}^*, r_{\min}^{\prime}\}$ completes the proof.
\end{proof}
\begin{proof}[\textbf{Proof of Proposition \ref{prop:geom}}]
For all $r\in [k]$,
$$
n_r^*\|c_r-c_r^*\|^2\le \Delta(C,C^*)\le b\phi^*
$$
so
$
\|c_r-c_r^*\|\le \sqrt{\frac{b\phi^*}{n_r^*}}$. 
Then for all $r\ne s$,
\begin{eqnarray*}
\|c_r-c_r^*\|+\|c_s-c_s^*\|
\le\sqrt{b}\sqrt{\phi^*}(\frac{1}{\sqrt{n_r^*}}+\frac{1}{\sqrt{n_s^*}})\\
=\frac{\sqrt{b}}{f} f\sqrt{\phi^*}(\frac{1}{\sqrt{n_r^*}}+\frac{1}{\sqrt{n_s^*}})
\le
\frac{\sqrt{b}}{f}\Delta_{rs}
\le
\frac{1}{16}\Delta_{rs}
\end{eqnarray*} 
where the second inequality is by (B),
and the last inequality by our assumption on $b$.
Thus, we may apply Lemma \ref{lm:kumar} to get
$
\frac{|A_{r}\triangle A_r^*|}{n_r^*}\le \frac{b}{f^3}
$ for all $r$, proving the first statement.
Now by Lemma \ref{lm:kmdist_cdist},
$\phi(C)\le (b+1)\phi^*$, so
\begin{eqnarray*}
\frac{\alpha b}{5\alpha b+4(1+\frac{\phi(C)}{\phi^*})}
\ge  
\frac{\alpha b}{5\alpha b+4(2+b)}\\
\ge
\frac{\alpha b}{5\alpha f^2/16^2+4(2+f^2/16^2)}\\
\ge \frac{b}{f^3(\alpha)}
\ge \frac{|A_{r}\triangle A_r^*|}{n_r^*}
\end{eqnarray*} 
where the third inequality holds
since
$f\ge \max\{64^2, \frac{5\alpha+5}{16^2\alpha}\}$ by (B). 
This proves the second statement since $C^*$ is then $(\frac{f^2}{16^2}, \alpha)$-stable 
by Definition \ref{defn:stable}.
\end{proof}
\section{Appendix B: Proof of Theorem \ref{thm:mbkm_local}}
\begin{thm}
\label{thm:mbkm_local}
Fix any $0<\delta\le \frac{1}{e}$.
Suppose $C^*$ is $(b_o,\alpha)$-stable.
If we run Algorithm \ref{alg:MBKM} with parameters satisfying
$$
m > \frac{\ln (1-\sqrt{\alpha})}{\ln (1-\frac{4}{5}p_{\min}^*)}
$$
$$
c^{\prime}>\frac{\beta}{2[1-\sqrt{\alpha}-(1-\frac{4}{5}p_{\min}^*)^m]}
\mbox{~with~}\beta\ge 2
$$
$$
t_o\ge
768(c^{\prime})^2(1+\frac{1}{b_o})^2n^2\ln^2\frac{1}{\delta}
$$
Then if at some iteration $i$,
$\Delta^i\le \frac{1}{2}b_o\phi^*$,
we have $\forall t> i$, 
$$
Pr(\Omega_t)
\ge 1-\delta
\mbox{~~~and}
$$
\begin{eqnarray*}
E_t[\Delta^t]\le
(\frac{t_o+i+1}{t_o+t+1})^{\beta}\Delta^i \\
+ \frac{(c^{\prime})^2B}{\beta-1}
(\frac{t_o+i+2}{t_o+i+1})^{\beta+1}\frac{1}{t_o+t+1}
\end{eqnarray*}
where
$B:=4(b_o+1)n\phi^*$.
\end{thm}
\subsection{Proofs leading to Theorem \ref{thm:mbkm_local}}
In the subsequent analysis, we let
$$
\beta^t:= 2c^{\prime}\min_{r}p_r^t(m)(1-\frac{\max_{r}p_r^t(m)}{\min_s p_s^t(m)}\sqrt{\alpha})
$$
where
\begin{eqnarray*}
&&p_r^t(m)
:=Pr\{c_r^{t-1}\mbox{~is updated at~}t\mbox{~with sample size~}m\}\\
&&= 1- (1-\frac{n_r^{t-1}}{n})^m
\end{eqnarray*}
So,
$$
\beta^t=2c^{\prime}(\min_rp_r^t(m)-\sqrt{\alpha}\max_s p_s^t(m))
$$
The noise terms appearing in our analysis are:
\begin{eqnarray}
E[\sum_r\sum_{x\in A_r^{t+1}}\|x-\hat{c}_r^{t+1}\|^2+\phi^t|F_t]\label{noise1}\\
\sum_r n_r^*\langle c_r^{t-1}-c_r^*, \hat{c}_r^t-E[\hat{c}_r^t|F_{t-1}]\rangle\label{noise2}\\
\sum_rn_r^*\|\hat{c}_r^t-c_r^*\|^2 \label{noise3}
\end{eqnarray} 
In the analysis of this section, we use $E_t[\cdot]$ as a shorthand notation for $E[\cdot|\Omega_t]$, where $\Omega_t$ is as defined in the main paper.
Let $F_t$ denote the natural filtration of the stochastic process $C^0, C^1, \dots$, up to $t$.

The main idea of the proof is to show that with proper choice with the algorithm's parameters $m$, $c^{\prime}$, and $t_o$, the following holds at every step $t$:
\begin{itemize}
\item
$\beta^t \ge 2$ |$\Omega_t$
\item
Noise terms \eqref{noise2} and \eqref{noise3} are upper bounded by a function of $\phi^*$|$\Omega_t$  
\item
$Pr(\Omega_t\setminus\Omega_{t+1})$ is negligible $|\Omega_t, \beta^t\ge 2$, bounded noise
\item
$E_t[\Delta^t|F_{t-1}]\le (1-\frac{\beta^t}{t_o+t})\Delta^{t-1}+\epsilon^t$ $|\Omega_t$

where $\epsilon^t$, the noise term, decreases of order $O(\frac{1}{t^2})$.
\end{itemize}
\begin{lm}\label{lm:invariant1}
Suppose $C^*$ is $(b_o, \alpha)$-stable.
If
$$
m > \frac{\ln (1-\sqrt{\alpha})}{\ln (1-\frac{4}{5}p_{\min}^*)}
$$
and 
$$
c^{\prime}>\frac{\beta}{2[1-\sqrt{\alpha}-(1-\frac{4}{5}p_{\min}^*)^m]}
$$
Then conditioning on $\Omega_t$, we have 
$\beta^t\ge \beta$. 
\end{lm}
\begin{proof}
Let's first consider $p_r^t(1)=\frac{n_r^{t-1}}{n}$.
Conditioning on $\Omega_t$, using the fact that $C^*$ is $(b_o,\alpha)$-stable,
we have
\begin{eqnarray*}
\frac{n_r^{t-1}}{n}\ge p_{\min}^* (1-\max_{r}\frac{|A_r^t\triangle A_r^*|}{n_r^*})\\
\ge p_{\min}^* (1-\frac{\alpha b_o}{5\alpha b_o+4(1+\frac{\phi^t}{\phi^*})})
\ge \frac{4}{5}p_{\min}^*
\end{eqnarray*}
And hence,
$$
\min_r p_r^t(m)\ge 1-(1-\frac{4}{5}p_{\min}^*)^m
$$
Now,
\begin{eqnarray*}
\beta^t \ge 2c^{\prime} (\min_r p_r^t(m)-\sqrt{\alpha})\\
\ge 2c^{\prime} (1-(1-\frac{4}{5}p_{\min}^*)^m-\sqrt{\alpha})
\ge \beta
\end{eqnarray*}
where the last inequality is by our requirement on $c^{\prime}$ and the fact that 
$1-(1-\frac{4}{5}p_{\min}^*)^m-\sqrt{\alpha}>0$ by our requirement on $m$.
\end{proof}
\begin{lm}
\label{lm:iter_wise}
Suppose $C^*$ is $(b_o, \alpha)$-stable.
Then if we apply one step of Algorithm \ref{alg:MBKM}, with
$m, c^{\prime}$ satisfying conditions in Lemma \ref{lm:invariant1}, then
conditioning on $\Omega_i$, 
\begin{eqnarray*}
\Delta^i\le
\Delta^{i-1}(1-\frac{\beta}{t_o+i})
+[\frac{c^{\prime}}{t_o+i}]^2\sum_r n_r^*\|\hat{c}_r^i-c_r^* \|^2\\
+\frac{2c^{\prime}}{t_o+i}\sum_r n_r^*\langle c_r^{i-1}-c_r^*, \xi_r^i\rangle
\end{eqnarray*}
where $\xi_r^i := \hat{c}_r^i-E[\hat{c}_r^i|F_{i-1}]$.
\end{lm}
\begin{proof}
Let $\Delta^i_r:=n_r^*\|c_r^i-c_r^*\|^2$, so $\Delta^i = \sum_r\Delta^i_r$,
and we use $p_r^t$ as a shorthand for $p_r^t(m)$.
By the update rule of Algorithm \ref{alg:MBKM},
\begin{eqnarray*}
\Delta^i_r
=n_r^*\|(1-\eta^i)(c_r^{i-1}-c_r^*)+\eta^i(\hat{c}_r^i-c_r^*)\|^2\\
\le n_r^*\{(1-2\eta^i)\|c_r^{i-1}-c_r^*\|^2
+2\eta^i\langle c_r^{i-1}-c_r^*,\hat{c}_r^i-c_r^* \rangle\\
+(\eta^i)^2[\|c_r^{i-1}-c_r^*\|^2+\|\hat{c}_r^i-c_r^*\|^2]\}
\end{eqnarray*}
Let $\xi_r^i = \hat{c}_r^i-E[\hat{c}_r^i|F_{i-1}]$, where
$E[\hat{c}_r^i|F_{i-1}] = (1-p_r^i)c_r^{i-1}+p_r^im(A_r^{i})$.
Since
\begin{eqnarray*}
\langle c_r^{i-1}-c_r^*,\hat{c}_r^i-c_r^* \rangle 
=\langle c_r^{i-1}-c_r^*,E[\hat{c}_r^i|F_{i-1}]+\xi_r^i-c_r^* \rangle\\
\le
(1-p_r^i)\|c_r^{i-1}-c_r^*\|^2 \\+ p_r^i\|m(A_r^{i})
-c_r^*\|\|c_r^{i-1}-c_r^*\| 
+ \langle c_r^{i-1}-c_r^*,\xi_r^i \rangle
\end{eqnarray*}
We have
\begin{eqnarray*}
\Delta_r^i 
\le
n_r^*\{
-2\eta^i[\|c_r^{i-1}-c_r^*\|^2-(1-p_r^i)\|c_r^{i-1}-c_r^*\|^2\\
- p_r^i\|c_r^{i-1}-c_r^*\|\|m(A_r^i)-c_r^*\|]
+\|c_r^{i-1}-c_r^*\|^2\\
+2\eta^i\langle \xi_r^i, c_r^{i-1}-c_r^*\rangle
+(\eta^i)^2[\|c_r^{i-1}-c_r^*\|^2+\|\hat{c}_r^i-c_r^*\|^2]
\}\\
\le
n_r^*\{
-\frac{2c^{\prime}}{t_o+i}\min_r p_r^t\|c_r^{i-1}-c_r^*\|^2\\
+\frac{2c^{\prime}}{t_o+i}\max_s p_s^t\|c_r^{i-1}-c_r^*\|\|m(A_r^i)-c_r^*\|\\
+\|c_r^{i-1}-c_r^*\|^2+2\eta^i\langle \xi_r^i, c_r^{i-1}-c_r^*\rangle\\
+(\eta^i)^2[\|c_r^{i-1}-c_r^*\|^2+\|\hat{c}_r^i-c_r^*\|^2]
\}
\end{eqnarray*}
Note
\begin{eqnarray*}
\sum_{r}n_r^*\|c_r^i-c_r^*\|\|m(A_r^i)-c_r^*\|\\
\le
\sqrt{
(\sum_r n_r^* \|c_r^{i-1}-c_r^*\|^2)
(\sum_r n_r^* \|m(A_r^i)-c_r^*\|^2)
}\\
=\sqrt{\Delta^{i-1}\Delta(m(A^i),C^*)}
\le
\sqrt{\alpha}\Delta^{i-1}
\end{eqnarray*}
where the first inequality is by Cauchy-Schwartz and 
the last inequality is by applying Lemma \ref{lm:stable_stat}.
Finally, summing over $\Delta_r^i$, we get
\begin{eqnarray*}
\Delta^i
=\sum_r\Delta^i_r
\le \Delta^{i-1}[1-\frac{2c^{\prime}}{t_o+i}\min_rp_r^t(1-\frac{\max_sp_s^t}{\min_rp_r^t}\sqrt{\alpha})]\\
+ [\frac{c^{\prime}}{(t_o+i)}]^2\sum_r n_r^*\|\hat{c}_r^i-c_r^* \|^2\\
+ \frac{2c^{\prime}}{(t_o+i)p_r^i}\sum_r n_r^*\langle c_r^{i-1}-c_r^*, \xi_r^i\rangle\\
\le
\Delta^{i-1}(1-\frac{\beta}{t_o+i})
+[\frac{c^{\prime}}{t_o+i}]^2\sum_r n_r^*\|\hat{c}_r^i-c_r^* \|^2\\
+\frac{2c^{\prime}}{t_o+i}\sum_r n_r^*\langle c_r^{i-1}-c_r^*, \xi_r^i\rangle\\
\end{eqnarray*}
The second inequality is by $\beta^t\ge \beta$, as proven in Lemma \ref{lm:invariant1}.
\end{proof}
\begin{lm}
\label{lm:conditional}
Suppose $X$ satisfies (A1), $C^o\in conv(X)$, and $C^*$ is $(b_o,\alpha)$-stable.
If we run one step of Algorithm \ref{alg:MBKM}, with
$m, c^{\prime}$ satisfying conditions in Lemma \ref{lm:invariant1},
then conditioning on $\Omega_i$, we have,
for any $\lambda>0$, 
\begin{eqnarray*}
&&E_i\{\exp\{\lambda \Delta^i\}|F_{i-1}\}\\
&&\le \exp\left\{\lambda\{(1-\frac{\beta}{t_0+i})\Delta^{i-1}
+\frac{(c^{\prime})^2 B}{(t_0+i)^2}
+\frac{\lambda (c^{\prime})^2 B^2}{2(t_0+i)^2}\}
\right\}
\end{eqnarray*}
\end{lm}
\begin{proof}
By Lemma \ref{lm:boundB}, we have
\eqref{noise2} and \eqref{noise3} are both upper bounded by $B$.
By Lemma \ref{lm:iter_wise}, we have
\begin{eqnarray*}
E_i\{\exp(\lambda \Delta^i)|F_{i-1}\}
\le \exp \lambda [\Delta^{i-1}(1-\frac{\beta}{t_o+i})+\frac{(c^{\prime})^2B}{(t_o+i)^2}]\\
E_i\{\exp\lambda\frac{2c^{\prime}}{t_o+i}\sum_{r}n_r^*\langle c_r^{i-1}-c_r^*,\xi_r^i\rangle|F_{i-1}\}
\end{eqnarray*}
Since 
$$
\frac{2\lambda c^{\prime}}{i+t_0} \sum_{r}n_r^*\langle \xi_r^i, c_r^{i-1}-c_r^{*}\rangle 
\le \frac{2\lambda c^{\prime}}{i+t_0} B
$$ and
$E_i\{\frac{2\lambda c^{\prime}}{i+t_0} \sum_{r}n_r^*\langle \xi_r^i, c_r^{i-1}-c_r^{*}\rangle|F_{i-1}\} = 0$,
by Hoeffding's lemma
\begin{eqnarray*}
E_i\left\{\exp\{\frac{2\lambda c^{\prime}}{i+t_0}\sum_{r}n_r^*\langle \xi_r^i, c_r^{i-1}-c_r^{*}\rangle|F_{i-1}\}
\right\} \\
\le \exp\{\frac{\lambda^2 (c^{\prime})^2B^2}{2(i+t_0)^2}\}
\end{eqnarray*}
Combining this with the previous bound completes the proof.
\end{proof}
\begin{lm}[adapted from \cite{balsubramani13}]
\label{lm:bal}
For any $\lambda >0$,
$E_i\{e^{\lambda \Delta^{i-1}}\}\le E_{i-1}\{e^{\lambda \Delta^{i-1}}\}$
\end{lm}
\begin{proof}
By our partitioning of the sample space, $\Omega_{i-1}=\Omega_i\cup (\Omega_{i-1}\setminus \Omega_i)$,
and for any $\omega \in \Omega_i$ and  $\omega^{\prime} \in \Omega_{i-1}\setminus \Omega_i$,
$\Delta^{i-1}(\omega) \le b_{o}\phi^*<\Delta^{i-1}(\omega^{\prime})$. 
Taking expectation
over $\Omega_i$ and $\Omega_{i-1}$, we get
$
E_i\{e^{\lambda \Delta^{i-1}}\}\le E_{i-1}\{e^{\lambda \Delta^{i-1}}\}
$.
\end{proof}
\begin{prop}
\label{prop:high_prob}
Fix any $0<\delta\le \frac{1}{e}$.
Suppose $C^*$ is $(b_o,\alpha)$-stable.
If $\Delta^o\le \frac{1}{2}b_o\phi^*$, and if
$$
m > \frac{\ln (1-\sqrt{\alpha})}{\ln (1-\frac{4}{5}p_{\min}^*)}
$$
$$
c^{\prime}>\frac{\beta}{2[1-\sqrt{\alpha}-(1-\frac{4}{5}p_{\min}^*)^m]}
\mbox{~with~}\beta\ge 2
$$
$$
t_o\ge
768(c^{\prime})^2(1+\frac{1}{b_o})^2n^2\ln^2\frac{1}{\delta}
$$ 
Then
$$
P(\Omega_{\infty})\le \delta
$$
(here we used $\Delta^0$ instead of $\Delta^i$ and treat the starting time, the $i$-th iteration in Theorem \ref{thm:mbkm_local} as the zeroth iteration for cleaner presentation).
\end{prop}
\begin{proof}
By Lemma \ref{lm:conditional}, for any $\lambda>0$,
\begin{eqnarray*}
E_i\{e^{\lambda \Delta^i}\}
\le 
E_i\{e^{\lambda\{(1-\frac{\beta}{t_{o}+i})\Delta^{i-1}}\}
\exp\{
\frac{\lambda (c^{\prime})^2 B}{(t_o+i)^2}
+\frac{\lambda^2 (c^{\prime})^2 B^2}{2(t_o+i)^2}\}\\
\le
E_{i-1}\{e^{\lambda^{(1)}\Delta^{i-1}}\}
\exp\{
\frac{\lambda (c^{\prime})^2 B}{(t_o+i)^2}
+\frac{\lambda^2 (c^{\prime})^2 B^2}{2(t_o+i)^2}\}
\end{eqnarray*} 
where $\lambda^{(1)} = \lambda(1-\frac{\beta}{t_{o}+i})$, and the second inequality
is by Lemma \ref{lm:bal}.
Similarly, the following recurrence relation holds for $k=0,\dots,i$:
\begin{eqnarray*}
E_{i-k}\{e^{\lambda^{(k)} \Delta^{i-k}}\}
\le
E_{i-(k+1)}\{e^{\lambda^{(k+1)}\Delta^{i-k-1}}\}\\
\exp\{
\frac{\lambda^{(k)} (c^{\prime})^2 B}{(t_o+i-k)^2}
+\frac{(\lambda^{(k)})^2 (c^{\prime})^2 B^2}{2(t_o+i-k)^2}
\}
\end{eqnarray*}
where $\lambda^{(0)}:=\lambda$, 
and for $k\ge 1$, $\lambda^{(k)}:=\Pi_{t=1}^{k}(1-\frac{\beta}{t_o+(i-t+1)})\lambda^{(0)}$.

Note (see, e.g., \cite{balsubramani13})
$\forall \beta>0, k\ge 1$, 
$$
\lambda^{(k)}
=
\Pi_{t=1}^{k}(1-\frac{\beta}{t_o+(i-t+1)})\le (\frac{t_o+i-k+1}{t_o+i})^{\beta}
$$
Since the bound is shrinking as $\beta$ increases and $\beta\ge 2$,
$$
\frac{\lambda^{(k)}}{(t_0+i-k)^2}
\le  (\frac{t_o+i-k+1}{t_o+i})^{2}\frac{\lambda}{(t_o+i-k)^2}
\le \frac{4\lambda}{(t_o+i)^2}
$$ 
Repeatedly applying the relation, we get
\begin{eqnarray*}
E_i\{e^{\lambda \Delta^i}\}
\le e^{\lambda^{(i)}\Delta^{0}}
\exp\{\sum_{k=0}^{i-1}(\frac{4\lambda (c^{\prime})^2 B}{(t_o+i)^2}
+\frac{4\lambda^2 (c^{\prime})^2 B^2}{2(t_o+i)^2})\}\\
\le 
\exp\{
\lambda(\frac{t_o+1}{t_o+i})^{\beta}\Delta^{0}
+[\lambda(c^{\prime})^2 B+\frac{\lambda^2(c^{\prime})^2B^2}{2}]
\frac{4i}{(t_o+i)^2}
\}\\
\le
\exp\{
\lambda(\frac{t_o+1}{t_o+i})^{\beta}\frac{1}{2}b_{o}\phi^*
+[\lambda(c^{\prime})^2 B+\frac{\lambda^2(c^{\prime})^2B^2}{2}]
\frac{4i}{(t_o+i)^2}
\}
\end{eqnarray*}
Then we can apply the conditional Markov's inequality, for any $\lambda_i>0$,
\begin{eqnarray*}
Pr(\omega \in \Omega_{i}\setminus\Omega_{i+1})
= Pr(\Delta^i>b_o\phi^*|\Omega_{i})\\
=Pr(e^{\lambda_i\Delta^i}>e^{\lambda_i b_o\phi^*}|\Omega_{i})
\le 
\frac{E[e^{\lambda_i\Delta_r^i}|\Omega_{i}]}{e^{\lambda_ib_o\phi^*}}
\end{eqnarray*}
Combining this with the upper bound on $E_{i}e^{\lambda_i\Delta^i}$, we get 
\begin{eqnarray*}
Pr(\omega \in \Omega_{i}\setminus\Omega_{i+1})\\
\le 
\exp\{-\lambda_i
\{
\frac{1}{2}b_{o}[2-(\frac{t_o+1}{t_o+i})^{\beta}]\\
-(B+\frac{\lambda_i B^2}{2})\frac{4(c^{\prime})^2i}{(t_o+i)^{2}}
\}
\}\\
\le
\exp\left\{-\lambda_i
\{
\frac{b_o\phi^*}{2}
-(B+\frac{\lambda_i B^2}{2})\frac{4(c^{\prime})^2i}{(t_o+i)^{2}}
\}
\right\}
\end{eqnarray*} 
since $i\ge 1$. 
We choose $\lambda_i = \frac{1}{\Delta}\ln \frac{(i+1)^2}{\delta}$ with
$\Delta = \frac{b_o\phi^*}{4}$, and show that 
$\frac{b_o\phi^*}{2}-(B+\frac{\lambda_i B^2}{2})\frac{4(c^{\prime})^2i}{(t_o+i)^{2}}$ 
is lower bounded by $\Delta$.
\paragraph{Case 1:}
$B>\frac{\lambda_i B^2}{2}$. We get
$$
\frac{1}{2}b_o\phi^*
-(B+\frac{\lambda_i B^2}{2})\frac{4(c^{\prime})^2i}{(t_o+i)^{2}}
\ge \Delta
$$ 
since 
$
t_o \ge 
\frac{128(c^{\prime})^2(b_o+1)n}{b_o}
=
\frac{64(c^{\prime})^2(b_o+1)n\phi^*}{\frac{1}{2}b_o\phi^*}
=\frac{16(c^{\prime})^2B}{\frac{1}{2}b_o\phi^*}
$.
\paragraph{Case 2:} $B\le \frac{\lambda_i B^2}{2}$. We get
\begin{eqnarray*}
\frac{1}{2}b_o\phi^*
-(B+\frac{\lambda_i B^2}{2})\frac{4(c^{\prime})^2i}{(t_o+i)^{2}}\\
\ge
2\Delta
-\lambda_i B^2\frac{4(c^{\prime})^2i}{(t_o+i)^{2}}\\
=
2\Delta-\frac{1}{\Delta}\ln \frac{(1+i)^2}{\delta}\frac{4(c^{\prime})^2B^2i}{(t_o+i)^{2}}\\
\ge
2\Delta-\frac{1}{\Delta}\ln \frac{(t_o+i)^2}{\delta}\frac{4(c^{\prime})^2B^2(t_o+i)}{(t_o+i)^{2}}
\end{eqnarray*}
Now we show 
$$
\frac{1}{\Delta}\ln \frac{(t_o+i)^2}{\delta}\frac{4(c^{\prime})^2B^2}{t_o+i}
\le \Delta
$$
Since 
\begin{eqnarray*}
t_o+i\ge t_o
\ge
768(c^{\prime})^2(1+\frac{1}{b_o})^2n^2\ln^2 \frac{1}{\delta}\\
=
\frac{48(c^{\prime})^2B^2}{(\frac{1}{2}b_o\phi^*)^2}\ln^2 \frac{1}{\delta}
\end{eqnarray*}
$\ln \frac{1}{\delta}\ge 1$,
and 
$
\frac{16(c^{\prime})^2B^2}{(\frac{1}{2}b_o\phi^*)^2}\ge \frac{1}{3}
$,
we can apply Lemma \ref{lm:tech} with 
$b=2$,
$C:=\frac{16(c^{\prime})^2B^2}{(\frac{1}{2}b_o\phi^*)^2}$, 
$t := t_o+i \ge (\frac{3C}{b-1}\ln\frac{1}{\delta})^{\frac{2}{b-1}}$, and get 
$$
\frac{4(c^{\prime})^2B^2}{\Delta^2}\ln\frac{(t_o+i)^2}{\delta}
:=
2C\ln t +C\ln\frac{1}{\delta} <t^{b-1} = t_o+i
$$
That is,
$\frac{1}{\Delta}\ln \frac{(t_o+i)^2}{\delta}\frac{4(c^{\prime})^2B^2}{t_o+i}
\le \Delta
$.
Thus, for both cases, 
$$
2\Delta-(B+\frac{\lambda_i B^2}{2})\frac{4(c^{\prime})^2i}{(t_o+i)^{2}}
\ge =\Delta
$$
and hence,
$$
Pr(\omega \in \Omega_{i}\setminus\Omega_{i+1})
\le e^{-\frac{1}{\Delta}(\ln\frac{(1+i)^2}{\delta})\Delta}
= \frac{\delta}{(i+1)^2}
$$
Finally, we have
$$
Pr(\cup_{i\ge 1}\Omega_{i}\setminus \Omega_{i+1})
\le
\sum_{i=1}^{\infty}Pr(\omega \in \Omega_{i}\setminus\Omega_{i+1})
\le \delta
$$
\end{proof}
\begin{proof}[\textbf{Proof of Theorem \ref{thm:mbkm_local}}]
Since the conditions in Proposition \ref{prop:high_prob} holds for any $t>i$,
we apply it and get
$$
Pr(\Omega_t)\ge 1-Pr(\cup_{t>i}\Omega_{t}\setminus\Omega_{t+1})
\ge 1-\delta
$$ 
This proves the first statement.
Taking expectation over $\Omega_t$ conditioning on filtration $F_{t-1}$ with respect to the inequality derived in Lemma \ref{lm:iter_wise}, we get
\begin{eqnarray*}
E_t[\Delta^t|F_{t-1}]
\le 
\Delta^{t-1}(1-\frac{\beta}{t_o+t})
+[\frac{c^{\prime}}{t_o+t}]^2B
\end{eqnarray*}
since \eqref{noise3} is bounded by $B$ by Lemma \ref{lm:boundB}, and since $E_t\{\xi_r^t|F_{t-1}\}=0$, $\forall r\in [k]$.
Taking total expectation over $\Omega_t$, we get
\begin{eqnarray*}
E_t[\Delta^t]
\le 
E_{t}[\Delta^{t-1}](1-\frac{\beta}{t_o+t})+\frac{(c^{\prime})^2B}{(t+t_o)^2}\\
\le
E_{t-1}[\Delta^{t-1}](1-\frac{\beta}{t_o+t})+\frac{(c^{\prime})^2B}{(t+t_o)^2}\\
\end{eqnarray*}
We can apply Lemma \ref{lm:tech_2} by letting
$u_{t}\leftarrow E_{t+t_o}[\Delta^{t+t_o}]$
(we temporarily change the notation $E_t[\Delta^t]$ to $E_{t+t_o}[\Delta^{t+t_o}]$ to match the notation in Lemma \ref{lm:tech_2}),
$t_o\leftarrow t_o+i$,
$a \leftarrow \beta$,
and
$b \leftarrow (c^{\prime})^2B$
\begin{eqnarray*}
E_t[\Delta^t]\le
(\frac{t_o+i+1}{t_o+t+1})^{\beta}\Delta^i 
+ \frac{(c^{\prime})^2B}{\beta-1}
(\frac{t_o+i+2}{t_o+i+1})^{\beta+1}\frac{1}{t_o+t+1}
\end{eqnarray*}
\end{proof}
\section{Proofs of Theorem \ref{thm:km} and Theorem \ref{thm:solution}}
One subtlety we need to point out before the proofs is that, in Algorithm \ref{alg:MBKM}, the learning rate $\eta_r^t$ as well as the update rule:
$$
c_r^t \leftarrow (1-\eta^t_r)c_r^{t-1} + \eta^t_r\hat{c}_r^t
$$
 is only defined for a cluster $r$ that is ``sampled'' at the $t$-th iteration. However, even if the cluster is not ``sampled'', i.e., $c_r^t = c_r^{t-1}$, the same update rule
with $\hat{c}_r^t=c_r^{t-1}$ and and the same learning rate still holds for this case.
So in our analysis, we equivalently treat each cluster $r$ as updated with learning rate $\eta_r^t$, and differentiates between a sampled  and not-sampled cluster only through the definition of $\hat{c}_r^t$. 
\subsection*{Proof leading to Theorem \ref{thm:km}}
\begin{lm}\label{lm:kmeans}
Suppose $\forall r\in [k]$, $\eta_r^t\le \eta_{\max}^t$ w.p. $1$. 
Then,
$
E[\phi^{t+1}-\phi^t|F_t]
\le
-2\min_{r,t; p_r^{t+1}>0}\eta_r^{t+1}p_r^{t+1} (\phi^t - \tilde\phi^t) 
+(\eta_{\max}^{t+1})^26\phi^t
$,
where $\tilde{\phi}^t:=\sum_r\sum_{x\in A_r^{t+1}}\|x-m(A_r^{t+1})\|^2$.
\end{lm}
\begin{proof}[\textbf{Proof of Lemma \ref{lm:kmeans}}]
For simplicity, we denote $E[\cdot|F_t]$ by $E_t[\cdot]$ (the same notation is also used as a shorthand to $E[\cdot|\Omega_t]$ in the proof of Theorem \ref{thm:mbkm_local}; we abuse the notation here).
\begin{eqnarray*}
E_t[\phi^{t+1}]
= E_t[\sum_{r=1}^k\sum_{x\in A_r^{t+2}}\|x-c_r^{t+1}\|^2]\\
\le E_t[\sum_r\sum_{x\in A_r^{t+1}}\|x-c_r^{t+1}\|^2] \\
=E_t[\sum_r\sum_{x\in A_r^{t+1}}\|x-(1-\eta_r^{t+1})c_r^t-\eta_r^{t+1}\hat{c}_r^{t+1}\|^2]\\
=E_t[\sum_r\sum_{x\in A_r^{t+1}}(1-\eta_r^{t+1})^2\|x-c_r^t\|^2\\
+(\eta_r^{t+1})^2\|x-\hat{c}_r^{t+1}\|^2
+2\eta_r^{t+1}(1-\eta_r^{t+1})\langle x-c_r^t,x-\hat{c}_r^{t+1}\rangle]\\
\end{eqnarray*} 
where the inequality is due to the optimality of clustering $A^{t+2}$ for centroids $C^{t+1}$.
Since
$$
E_t[\hat{c}_r^{t+1}]=(1-p_r^{t+1})c_r^{t}+p_r^{t+1}m(A_r^{t+1})
$$ we have
\begin{eqnarray*}
&&\langle x-c_r^t,x-\hat{c}_r^{t+1}\rangle\\
&&=
(1-p_r^{t+1})\|x-c_r^t\|^2
+p_r^{t+1} \langle x-c_r^t, x-m(A_r^{t+1}) \rangle
\end{eqnarray*}
Plug this into the previous inequality, we get
\begin{eqnarray*}
E_t[\phi^{t+1}]
\le
\sum_r(1-2\eta_r^{t+1})\phi^t_r
+(\eta_r^{t+1})^2\phi^t_r \\
+(\eta_r^{t+1})^2\sum_{x\in A_r^{t+1}}\|x-\hat{c}_r^{t+1}\|^2\\
+2\eta_r^{t+1}\{(1-p_r^{t+1})\sum_{x\in A_r^{t+1}}\|x-c_r^t\|^2\\
+p_r^{t+1} \sum_{x\in A_r^{t+1}}\langle x-c_r^t, x-m(A_r^{t+1}) \rangle\}\\
=\phi^t-2\sum_r\eta_r^{t+1}p_r^{t+1}\phi^t_r\\
+2\sum_r\eta_r^{t+1}p_r^{t+1}\sum_{x\in A_r^{t+1}}\langle x-c_r^t, x-m(A_r^{t+1}) \rangle\}\\
+(\eta_r^{t+1})^2\phi^t_r
+(\eta_r^{t+1})^2\sum_{x\in A_r^{t+1}}\|x-\hat{c}_r^{t+1}\|^2
\end{eqnarray*}
Now,
\begin{eqnarray*}
\sum_{x\in A_r^{t+1}}\langle x-c_r^t, x-m(A_r^{t+1}) \rangle\\
=\sum_{x\in A_r^{t+1}}\langle x-m(A_r^{t+1})+m(A_r^{t+1})-c_r^t, x-m(A_r^{t+1}) \rangle\\
=\sum_{x\in A_r^{t+1}}\|x-m(A_r^{t+1})\|^2\\
+\sum_{x\in A_r^{t+1}}\langle m(A_r^{t+1})-c_r^t,x-m(A_r^{t+1})  \rangle
=\phi^t_r
\end{eqnarray*}
since $\sum_{x\in A_r^{t+1}}\langle m(A_r^{t+1})-c_r^t,x-m(A_r^{t+1})\rangle=0$, by property of the mean of a cluster.
Then
\begin{eqnarray*}
E_t[\phi^{t+1}]
\le
\phi^t+\sum_r 2\eta_r^{t+1}p_r^{t+1}(-\phi_r^t+\tilde\phi_r^t)\\
+(\eta_r^{t+1})^2 [\phi_r^t + E_t[\sum_{x\in A_r^{t+1}}\|x-\hat{c}_r^{t+1}\|^2] 
\end{eqnarray*}
Now a key observation is that $p_r^{t+1}=0$ if and only if cluster $A_r^{t+1}$ is empty, i.e.,
degenerate.
Since the degenerate clusters do not contribute to the $k$-means cost, 
we have 
$\sum_{r; p_r^{t+1}>0}\phi^t_r=\phi^t$, and similarly,
$\sum_{r; p_r^{t+1}>0}\tilde\phi^t_r=\tilde\phi^t$.
Therefore,
\begin{eqnarray*}
E_t[\phi^{t+1}]
\le
\phi^t
-2\min_{r,t; p_r^{t+1}>0}\eta_r^{t+1}p_r^{t+1} (\phi^t - \tilde\phi^t) \\
+(\eta_{\max}^{t+1})^2(E_t\sum_r\sum_{x\in A_r^{t+1}}\|x-\hat{c}_r^{t+1}\|^2+\phi^t)\\
=\phi^t
-2\min_{r,t; p_r^{t+1}>0}\eta_r^{t+1}p_r^{t+1} (\phi^t - \tilde\phi^t) 
+(\eta_{\max}^{t+1})^26\phi^t
\end{eqnarray*}
where the last inequality is by Lemma \ref{lm:boundB1}.
\end{proof}
 \begin{lm}
 \label{lm:limit_cluster}
Suppose Assumption (A) holds.
 If we run Algorithm \ref{alg:MBKM} on $X$ with $\eta^t=\frac{c^{\prime}}{t_o+t}$, 
and $t_o>1$, with any initial set of $k$ centroids $C^0\in conv(X)$.
Then for any $\delta>0$, $\exists t$ s.t. $\Delta(C^t,C^*)\le \delta$ with $C^*:=m(A^*)$ for some $A^*\in \{A^*\}_{[k]}$.
\end{lm}
\begin{proof}[\textbf{Proof of Lemma \ref{lm:limit_cluster}}]
First note that since $\{C^*\}_{[k]}$ includes all stationary points with $1\le k^{\prime}\le k$
non-degenerate centroids, and at any time $t$, $C^t$ must have $k^{t}\in [k]$ non-degenerate centroids, so there exists $C^*\in\{C^*\}_{k^{t}}\in\{C^*\}_{[k]}$ such that
$
\Delta(C^t, C^*)
$ is well defined.
For a contradiction, suppose $\forall t\ge 1$,
$\Delta(C^t,C^*)>\delta$, for all $C^*\in\{C^*\}_{k^{t}}$. 
Then
\paragraph{Case 1:}
$m(A^{t+1})\in\{C^*\}_{k^{t}}$ 

Then 
$$
\Delta(C^t, m(A^{t+1}))>\delta
$$ 
by our assumption.
\paragraph{Case 2:}
$m(A^{t+1})\notin \{C^*\}_{k^{t}}$ 

Since $C^{t}\in Cl(v^{-1}(A^{t+1}))$ by our definition, applying Lemma \ref{lm:stat_stab2}, 
$$
\Delta(C^t,m(A^{t+1}))\ge r_{\min}\phi(m(A^{t+1}))
$$
So for both cases, 
$$
\Delta(C^t, m(A^{t+1}))\ge \min\{\delta, r_{\min}\phi_{opt}\}
$$
Let denote $\delta_o:=\min\{\delta, r_{\min}\phi(m(A^{t+1}))\}$, then
by Lemma \ref{lm:kmeans}, 
\begin{eqnarray*}
E[\phi^{t+1}-\phi^t|F_t]\\
\le 
-\frac{2c^{\prime}\min_{\substack{r\in [k];p_r^{t+1}(m)>0}}p_r^{t+1}(m)}{t+1+t_o}\phi^t(1-\frac{\tilde\phi^t}{\phi^t})\\
+(\frac{c^{\prime}}{t+1+t_o})^26\phi_{\max}
\end{eqnarray*}
Note for $p_r^{t+1}(m)>0$, by the discrete nature of the dataset, 
$
\frac{n_r^{t+1}}{n}\ge  \frac{1}{n}
$, therefore,
$$
\min_{\substack{r\in [k];p_r^t(m)>0}}p_r^t(m)
\ge 1-(1-\frac{1}{n})^m
\ge 1-e^{-\frac{m}{n}}
$$
Also note
\begin{eqnarray*}
\phi^t- \tilde\phi^t 
=\sum_{r\in [k^{\prime}]} \sum_{x\in A^{t+1}_r}\|x-C^t\|^2-\|x-m(A_r^{t+1})\|^2\\
= \sum_r \|c_r^t-m(A_r^{t+1})\|^2n_r^{t+1}
= \Delta(C^t, m(A^{t+1}))\ge \delta_o
\end{eqnarray*}
Then $\forall t\ge 1$,
\begin{eqnarray*}
E[\phi^{t+1}]-E[\phi^t] \\
\le -\frac{2c^{\prime}(1-e^{-\frac{m}{n}})}{t+1+t_o}\delta_o
+\frac{6\phi_{\max}(c^{\prime})^2}{(t+1+t_o)^2}
\end{eqnarray*} 
Summing up all inequalities,
\begin{eqnarray*}
E[\phi^{t+1}]-E[\phi^0]\\
\le -2c^{\prime}(1-e^{-\frac{m}{n}})\delta_o\ln \frac{t+t_o+1}{t_o}
+\frac{6\phi_{\max}(c^{\prime})^2}{t_o-1}
\end{eqnarray*}
 Since $t$ is unbounded and $\ln \frac{t+t_o+1}{t_o}$ increases with $t$ while
 $\frac{6\phi_{\max}(c^{\prime})^2}{t_o-1}$ is a constant, 
 $\exists T$ such that for all $t\ge T$,
 $E\phi^t - \phi^0 \le -\phi^0$, which means
 $E[\phi^{t}]\le 0$, for all $t$ large enough. 
 This implies the $k$-means cost of some clusterings is negative, which is impossible.
 So we have a contradiction. 
 \end{proof}
\paragraph{Proof setup of Theorem \ref{thm:km}}
The goal of the proof is to show that first, with high probability, the algorithm converges to some stationary clustering, $A^{*}\in \{A^*\}_{[k]}$. We call this event $G$; formally,
$$
G := \{\exists T\ge 1, \exists A^{*}\in \{A^*\}_{[k]}, \mbox{~s.t.~} A^t = A^{*}, \forall t\ge T\}
$$
Second, we want to establish the $O(\frac{1}{t})$ expected convergence rate of the algorithm to this stationary clustering $A^{*}$.

To prove that the event $G$ has high probability, we first consider random variable $\tau$:
$$
\tau:=\min\{t\ge 0~|\min_{A^*\in \{A^*\}_{[k]}}\Delta(C^t, m(A^*)) \le \frac{1}{2}r_{\min}\phi^*\}
$$
That is, $\tau$ is the first time the algorithm ``hits'' a stationary clustering; $\tau$ is a stopping time since $\forall t\ge 0$, $\{\tau\le t\}$ is $F_t$-measurable. 
By Lemma \ref{lm:limit_cluster}
\begin{eqnarray}\label{eqn:sum_one}
Pr(\{\tau <\infty\})=Pr(\{\tau \in \mathbb{N}\})
=Pr(\cup_{T\ge 0}\{\tau=T\})
=1
\end{eqnarray}
Fixing $\tau$, we denote the stationary clustering that the algorithm ``hits'' by
$$
A^{*}(\tau):=\arg\min_{A^*\in \{A^*\}_{[k]}}\Delta(C^{\tau}, m(A^*)) 
$$
$A^{*}(\tau)$ is well defined; the reason is that when $\Delta(C^{\tau}, m(A^*)) \le \frac{1}{2}r_{\min}\phi^*$, $A^{\tau}=A^*$, so there can be only one minimizer. 

We will prove a subset $G_o\subset G$ holds with high probability.
To do this, we construct $G_o$ as a union of disjoint events determined by the realization of $\tau$ and $A^{*}(\tau)$: we define events
$$
G_{T}(A^{*}):= \{\tau =T\}\cap\{A^{*}(\tau)=A^*\}\cap\{\forall t\ge T, \Delta^t\le r_{\min}\phi^*\} 
$$
Then we can represent the event where the algorithm's iterate converges to a particular stationary clustering $A^*$ as
$$
G(A^{*}):=\cup_{T\ge 0}G_{T}(A^{*})
$$
Finally, we define
$$
G_o:= \cup_{A^*\in \{A^*\}_{[k]}}G(A^*)
$$
$G_o\subset G$ since the event $\Delta^t\le r_{\min}\phi^*$ implies 
$A^t=A^*$.
\begin{proof}[\textbf{Proof of Theorem \ref{thm:km}}]
Fix any $(T, A^*)$, conditioning on $\{\tau=T\}\cap\{A^*(\tau)=A^*\}$, since we have
\begin{eqnarray*}
c^{\prime}
>
\frac{\phi_{\max}}{(1-e^{-\frac{m}{n}})r_{\min}\phi_{opt}}
\end{eqnarray*}
We can envoke Lemma \ref{lm:thm_part1} to get $\forall t < T$,
\begin{eqnarray}\label{eqn:globconv}
E\{\phi^t-\phi(A^{*})|G_T(A^*)\} = O(\frac{1}{t})
\end{eqnarray}
Now let's consider the case $t\ge T$.
Since by Lemma \ref{lm:stat_stab2}, $A^{*}$ is $(r_{\min},0)$-stable, we can apply Theorem \ref{thm:mbkm_local}:
in this context, the parameters in the statement of Theorem \ref{thm:mbkm_local} are
$b_o=r_{\min}$, $\alpha=0$, $p^*_{\min}\ge \frac{1}{n}$.
Thus, for any
$$
m\ge 1
$$
$$
c^{\prime}>\frac{\beta}{2(1-e^{\frac{4m}{5n}})} \mbox{~~with~~}\beta\ge 2
$$
and
$$
t_o\ge 768(c^{\prime})^2(1+\frac{1}{r_{\min}})^2n^2\ln^2\frac{1}{\delta}
$$
the conditions required by Theorem \ref{thm:mbkm_local} are satisfied. 
Then by the first statement of Theorem \ref{thm:mbkm_local},
\begin{eqnarray}\label{eqn:conditional}
Pr(\{\forall t\ge T, \Delta^t\le r_{\min}\phi^*\}|\{\tau=T\}\cap\{A^*(\tau)=A^*\})\nonumber\\
= P(\Omega_{\infty}|\{\tau=T\}\cap\{A^*(\tau)=A^*\})
\ge 1-\delta \label{eqn:conditional}
\end{eqnarray}
and by the second statement of Theorem \ref{thm:mbkm_local}, $\forall t> T$,
\begin{eqnarray*}
E\{\phi^t - \phi(A^{*}) |\Omega_t, \{\tau=T\}\cap\{A^*(\tau)=A^*\}\}\\
\le E\{\Delta(C^t,C^{*})|\Omega_t, \{\tau=T\}\cap\{A^*(\tau)=A^*\}\} 
=O(\frac{1}{t})
\end{eqnarray*} 
where
the first inequality is by Lemma \ref{lm:kmdist_cdist}. 
Since $\Omega_{\infty}\subset \Omega_t$, $\forall t\ge 0$, this implies
\begin{eqnarray} 
E\{\Delta(C^t,C^{*})|\Omega_{\infty}, \{\tau=T\}\cap\{A^*(\tau)=A^*\}\}\nonumber\\
=E\{\Delta(C^t,C^{*})|G_T(A^*)\}
=O(\frac{1}{t}) \label{eqn:localconv}
\end{eqnarray} 
Finally, we turn to prove $Pr(G)$ is large.
Recall
\begin{eqnarray*}
Pr\{G\}\ge Pr\{G_o\}
= Pr\{\cup_{T\ge 0}\cup_{A^{*}\in \{A^*\}_{[k]}} G_T(A^*)\}\\
= \sum_{T\ge 0,A^{*}\in \{A^*\}_{[k]}} Pr\{G_T(A^*)\}
\end{eqnarray*}
where the second equality holds because the events $G_{T}(A^*)$ are disjoint for different pairs of $(T, A^*)$, since the stopping time $\tau$ and the minimizer $A^*(\tau)$ are unique for each experiment.
Since
\begin{eqnarray*}
\sum_{T\ge 0,A^{*}\in \{A^*\}_{[k]}} Pr\{G_T(A^*)\}\\
=\sum_{T,A^{*}} Pr\{\Omega_{\infty}|\{\tau=T\}\cap\{A^*(\tau)=A^*\}\}\\
Pr(\{\tau=T\}\cap\{A^*(\tau)=A^*\})\\
 \ge (1-\delta)\sum_{T,A^{*}}Pr(\{\tau=T\}\cap\{A^*(\tau)=A^*\})\\
 =(1-\delta)Pr\{\cup_{T}\cup_{A^{*}}\{\tau=T\}\cap\{A^*(\tau)=A^*\}\}\\
 =(1-\delta)Pr\{\cup_{T\ge 0}\{\tau=T\}\}
 =1-\delta
\end{eqnarray*}
where the inequality is by \eqref{eqn:conditional}, and the last two equalities are due to the finiteness of $\{A^*\}_{[k]}$ and by \eqref{eqn:sum_one}, respectively.
Therefore,
$
Pr\{G\}\ge 1-\delta
$,
which completes the proof of the first statement.
In addition,
\begin{eqnarray*}
Pr\{\cup_{A^{*}\in\{A^*\}_{[k]}}G(A^{*})\}\\
= Pr\{\cup_{T\ge 0,A^{*}\in \{A^*\}_{[k]}}\Omega_{\infty}\cap\{\tau=T\}\cap\{A^*(\tau)=A^*\}\}\\
\ge 1-\delta
\end{eqnarray*} which proves the second statement.
Finally, combining inequalities \eqref{eqn:globconv} and \eqref{eqn:localconv}, we have $\forall \ge 1$ and $\forall t\ge 1$,
$$
E\{\phi^t - \phi(A^{*}) |G_T(A^*)\} = O(\frac{1}{t})
$$
Since the quantity $\phi^t - \phi(A^{**})$ is independent of $T$, we reach the conclusion
$$
E\{\phi^t - \phi(A^{*}) |G(A^{*})\} = O(\frac{1}{t})
$$
\end{proof}
\begin{lm}
\label{lm:thm_part1}
Suppose the assumptions and settings in Theorem \ref{thm:km} hold, 
conditioning on any $G_{T}(A^{*})$, we have $\forall 1\le t < T$,
$$
E\{\phi^t-\phi(A^{*})|G_{T}(A^{*})\} = O(\frac{1}{t})
$$
\end{lm}
\begin{proof}
First observe that conditioning on the event $G_{T}(A^{*})$,
$
\Delta(C^t, C^{*})>\frac{1}{2}r_{\min}\phi^{*}
$,
$\forall t<T$.
Now we are in a setup similar to that in the proof Lemma \ref{lm:limit_cluster}, and the argument therein will lead us to the conclusion that
$$
\phi^t-\tilde\phi^t> \min\{\frac{1}{2}r_{\min}, r_{\min}\}\tilde\phi^{t}=\frac{1}{2}r_{\min}\tilde\phi^{t}
$$
Proceeding as in Lemma \ref{lm:limit_cluster}, we have conditioning on $G_{T}(A^{*})$,
\begin{align*}
\label{eqn:kmeans}
&E[\phi^t|G_{T}(A^{*})]\\
&\le E[\phi^{t-1}|G_{T}(A^{*})]
\{1-\frac{2c^{\prime}\min_{\substack{r\in [k];p_r^{t}(m)>0}}p_r^{t}(m)}{t+t_o}\frac{r_{\min}\phi_{opt}}{2\phi_{\max}}\}
+(\frac{c^{\prime}}{t+t_o})^26\phi_{\max}
\end{align*} 
since $\forall t\ge 1$,
$$
1-\frac{\tilde\phi^t}{\phi^t}\ge \frac{r_{\min}}{2}\frac{\tilde\phi^t}{\phi^t}
\ge\frac{r_{\min}}{2}\frac{\phi_{opt}}{\phi_{\max}}
$$
Now, since we set
\begin{eqnarray*}
c^{\prime}
>
\frac{\phi_{\max}}{(1-e^{-\frac{m}{n}})r_{\min}\phi_{opt}}
\end{eqnarray*}
we have
\begin{eqnarray*}
2c^{\prime}\min_{\substack{r\in [k];p_r^{t}(m)>0}}p_r^{t}(m)\frac{r_{\min}\phi_{opt}}{2\phi_{\max}}\\
\ge
2c^{\prime}(1-(1-\frac{1}{n})^m)\frac{r_{\min}\phi_{opt}}{2\phi_{\max}}\\
\ge
2c^{\prime}(1-e^{-\frac{m}{n}})\frac{r_{\min}\phi_{opt}}{2\phi_{\max}}\\
>
2\frac{\phi_{\max}}{(1-e^{-\frac{m}{n}})r_{\min}\phi_{opt}}(1-e^{-\frac{m}{n}})\frac{r_{\min}\phi_{opt}}{2\phi_{\max}}
>1
\end{eqnarray*}
Applying Lemma \ref{lm:tech_2} with 
$$
a:=2c^{\prime}\min_{\substack{r\in [k];p_r^{t}(m)>0}}p_r^{t}(m)\frac{r_{\min}\phi_{opt}}{2\phi_{\max}}>1
$$ 
$$
b:=\frac{6(c^{\prime})^2\phi_{\max}}{(t_o+t)^2}
$$
We conclude that $\forall 1\le t<T$,
\begin{eqnarray*}
E[\phi^t|G_{T}(A^{*})]
\le \frac{t_o+1}{t_o+t+1}\phi^o
+\frac{b}{a-1}(\frac{t_o+2}{t_o+1})^{a+1}\frac{1}{t_o+t+1}
\end{eqnarray*}
Subtracting $\phi(A^{*})$ from both sides of the equation, we get
\begin{eqnarray*}
E[\phi^t-\phi(A^{*})|G_{T}(A^{*})]
\le \frac{t_o+1}{t_o+t+1}(\phi^o-\phi(A^{*}))\nonumber\\
+\frac{b}{a-1}(\frac{t_o+2}{t_o+1})^{a+1}\frac{1}{t_o+t+1}
=O(\frac{1}{t})\label{eqn:less}
\end{eqnarray*}
\end{proof}
\subsection*{Proofs leading to Theorem \ref{thm:solution}}
Here, we additionally define two quantities that characterizes $C^*$:
Let $A^*=v(C^*)$,
we use $p^*_{\min}:=\min_{r\in [k]}\frac{n^*_r}{n}$ to characterize the fraction of the smallest cluster in $A_{*}$ to the entire dataset.
We use $w_r:=\frac{\frac{\phi_{*}^r}{n^*_r}}{\max_{x\in A^*_r} \|x-c^*_r\|^2}$ to characterize the ratio between average and maximal ``spread'' of cluster $A_r^*$, and we let $w_{\min}:=\min_{r\in [k]}w_r$.
\subsection{Existence of stable stationary point under geometric assumptions on the dataset}
First, we observe that our Assumption (B) implies two lower bounds on $\|c_r^*-c_s^*\|$, 
$\forall r, s\ne r$. 
Let $x\in A^{*}_r\cap A^{t}_s$. Split $x$ into its projection on the line joining $c^{*}_r$ and $c^{*}_s$, and its orthogonal component:
\begin{eqnarray}\label{eqn:xsplit}
x = \frac{1}{2}(c^{*}_r+c^{*}_s)+\lambda (c^{*}_r-c^{*}_s)+u
\end{eqnarray}
with $u\perp c^{*}_r-c^{*}_s$. 
Note $\lambda$ measures the ratio between departure of the projected point from the mid-point of $c^{*}_r$ and $c^{*}_s$ and the norm $\|c_r^*-c_s^*\|$. 
By minimality of our definition of margin $\Delta_{rs}$,
\begin{eqnarray}
\label{eqn:margin}
\|\bar{x}-\frac{1}{2}(c^{*}_r+c^{*}_s)\|=\lambda\|c^{*}_r-c^{*}_s\|
\ge \frac{1}{2}\Delta_{rs} 
\end{eqnarray}
In addition, since $c_r^*$ is the mean of $A_r^*$, we know there exists $x\in A_r^*$ such that $\bar x$ falls outside of the line segment $c_r^*-c_s^*$ (or exactly on $c_r^*$ in the special case where all points projects on $c_r^*$). Similar holds for $c_s^*$. Thus,
\begin{eqnarray}\label{eqn:mean_sep}
\|c_r^*-c_s^*\|\ge \Delta_{rs}\ge f(\alpha)\sqrt{\phi^*}(\frac{1}{\sqrt{n_r^*}}+\frac{1}{\sqrt{n_s^*}})
\end{eqnarray}
\begin{lm}[Theorem 5.4 of \cite{kumar}]
\label{lm:kumar}
Suppose $(X,C^*)$ satisfies (B).
If $\forall r\in [k], s\ne r$, $\Delta_r^t+\Delta_s^t\le \frac{\Delta_{rs}}{16}$. Then for any $s\ne r$,
$
|A_r^*\cap A_s^t| \le \frac{b^2}{f(\alpha)}
$, where $b\ge \max_{r,s}\frac{\Delta_r^t+\Delta_s^t}{\Delta_{rs}}$.
\end{lm}
The proof is almost verbatim of Theorem 5.4 of \cite{kumar}; we include it here for completeness.
\begin{proof}
Since the projection of $x$ on the line joining $c^t_r, c^t_s$ is closer to $s$, we have
\begin{eqnarray*}
x(c^t_s -c^t_r)\ge \frac{1}{2}(c^t_s-c^t_r)(c^t_s+c^t_r)
\end{eqnarray*}
Substituting \eqref{eqn:xsplit} into the inequality above,
\begin{eqnarray}
\frac{1}{2}(c^{*}_r+c^{*}_s)(c^t_s-c^t_r)
+\lambda(c^{*}_r-c^{*}_s)(c^t_s-c^t_r)\nonumber\\
+u(c^t_s-c^t_r)\ge \frac{1}{2}(c^t_s-c^t_r)(c^t_s+c^t_r)\label{eqn:goodpoints}
\end{eqnarray}
Since $u\perp c^{*}_r-c^{*}_s$, let $\Delta=\Delta^t_s+\Delta^t_r$. We have
$$
u(c^t_s-c^t_r) = u(c^t_s-c^{*}_s-(c^t_r-c^{*}_r))
\le \|u\|\Delta
$$
Rearranging \eqref{eqn:goodpoints}, we have
\begin{eqnarray*}
 \frac{1}{2}(c^{*}_r+c^{*}_s-c^t_s-c^t_r)(c^t_s-c^t_r)\nonumber\\
 +\lambda(c^{*}_r-c^{*}_s)(c^t_s-c^t_r)
+u(c^t_s-c^t_r)\ge 0\nonumber\\
\equiv
\frac{\Delta^2}{2}+\frac{\Delta}{2}\|c^{*}_r-c^{*}_s\|
-\lambda\|c^{*}_r-c^{*}_s\|^2\nonumber\\
+\lambda\Delta \|c^{*}_r-c^{*}_s\|+\|u\|\Delta\ge 0
\end{eqnarray*}
Therefore,
\begin{eqnarray*}
\|x - c^{*}_r\| = \|(\frac{1}{2}-\lambda)(c^{*}_s-c^{*}_r)+u\|\ge \|u\| \\
\ge \frac{\lambda}{\Delta}\|c^{*}_r-c^{*}_s\|^2-\frac{\Delta}{2}\\
-\frac{1}{2}\|c^{*}_r-c^{*}_s\|
-\lambda \|c^{*}_r-c^{*}_s\| \label{ineq:1}
\ge \frac{\Delta_{rs}\|c^*_r-c^*_s\|}{64\Delta}
\end{eqnarray*}
where the last inequality is by our assumption that $\Delta\le\frac{\Delta_{rs}}{16}$, 
and $\lambda \ge \frac{\Delta_{rs}}{2\|c^{*}_r-c^{*}_s\|}$ by \eqref{eqn:margin}.
By previous inequality and our assumption on $f$,
\footnote{We use $f$ as a shorthand for $f(\alpha)$ in the subsequent proof.} for all $s\ne r$
\begin{eqnarray*}
|A^{*}_r\cap A^{t}_s|\frac{\Delta^2_{rs}\|c_r^*-c_s^*\|^2}{f\Delta^2} 
\le \sum_{x\in A^{*}_r\cap A^t_s}\|x-c^{*}_r\|^2
\end{eqnarray*}
So
$
|A^{*}_r\cap A^{t}_s|\le \sum_{x\in A^{*}_r\cap A^t_s}\|x-c^{*}_r\|^2
\frac{f(\Delta^t_r+\Delta^t_s)^2}{\Delta_{rs}^2\|c_r^*-c_s^*\|^2}
\le\frac{f b^2}{f^2\phi^*(\frac{1}{n^*_r})}( \sum_{A^{*}_r\cap A^t_s}\|x-c^{*}_r\|^2)
$, where the second inequality is by \eqref{eqn:mean_sep}. 
That is,
$
\frac{|A^{*}_r\cap A^{t}_s|}{n^*_r}
\le \frac{b^2}{f\phi^*} \sum_{A^{*}_r\cap A^t_s}\|x-c^{*}_r\|^2
$.
Similarly, for all $s\ne r$,
$
\frac{|A^{*}_s\cap A^{t}_r|}{n^*_r}
\le \frac{b^2}{f\phi^*} \sum_{A^{*}_s\cap A^t_r}\|x-c^{*}_s\|^2
$
Summing over all $s\ne r$,
$\frac{|A_{r}\triangle A_r^*|}{n_r^*}
=\rho_{out}+\rho_{in}\le \frac{b^2}{f\phi^*}\phi^* = \frac{b^2}{f}$.
\end{proof}
\begin{lm}
\label{lm:kmdist_cdist}
Fix a stationary point $C^*$ with $k$ centroids, and any other set of $k^{\prime}$-centroids, $C$, with $k^{\prime}\ge k$ so
that $C$ has exactly $k$ non-degenerate centroids.
We have
$$
\phi(C)-\phi^*\le \min_{\pi}\sum_{r}n_r^*\|c_{\pi(r)}-c_r^*\|^2 = \Delta(C,C^*)
$$
\end{lm}
\begin{proof}
Since degenerate centroids do not contribute to $k$-means cost, in the following we only consider
the sets of non-degenerate centroids $\{c_s, s\in [k]\}\subset C$ and $\{c_r^*, r\in [k]\}\subset C^*$.
We have for any permutation $\pi$,
\begin{eqnarray*}
\phi(C)-\phi^*
=\sum_s\sum_{x\in A_s}\|x-c_s\|^2-\sum_r\sum_{x\in A^*_r}\|x-c^*_r\|^2\\
\le
\sum_r\sum_{x\in A^*_r}\|x-c_{\pi(r)}\|^2-\sum_r\sum_{x\in A^*_r}\|x-c^*_r\|^2\\
= \sum_r n_r^*\|c_{\pi(r)}-c_r^*\|^2
\end{eqnarray*}
where the last inequality is by optimality of clustering assignment based on Voronoi diagram,
and the second inequality is by applying the centroidal property in Lemma \ref{centroidal} to each centroid in $C^*$.
Since the inequality holds for any $\pi$, it must holds for 
$\min_{\pi}\sum_{r}n_r^*\|c_{\pi(r)}-c_r^*\|^2$, which completes the proof.
\end{proof}
\subsubsection*{Proofs regarding seeding guarantee}
\begin{lm}[Theorem 4 of \cite{tang_montel:aistats16}]
\label{lm:adapt_tang}
Suppose $(X, C^*)$ satisfies (B). If we obtain seeds from Algorithm \ref{alg:seeding}, then
$$
\Delta(C^0,C^*)\le \frac{1}{2}\frac{f(\alpha)^2}{16^2}\phi^*
$$
with probability at least
$
1-m_o\exp(-2(\frac{f(\alpha)}{4}-1)^2w_{\min}^2)-k \exp (-m_op^*_{\min})
$.
\end{lm}
\begin{proof}
First recall that, as in \eqref{eqn:mean_sep}, assumption (B) implies center-separability assumption in Definition 1 of \cite{tang_montel:aistats16}, i.e.
$$
\forall r\in [k], s\ne r,
\|c_r^*-c_s^*\|\ge f(\alpha)\sqrt{\phi^*}(\frac{1}{\sqrt{n_r^*}}+\frac{1}{\sqrt{n_s^*}})
$$
with $f(\alpha)\ge \max_{r\in [k],s\ne r}\frac{n_r^*}{n_s^*}$.
\footnote{note: ``$\alpha$'' in \cite{tang_montel:aistats16} is defined as $\min_{r\in [k],s\ne r}\frac{n_r^*}{n_s^*}$,which is not to be confused with our ``$\alpha$''.}
Applying Theorem 4 of \cite{tang_montel:aistats16} with $\mu_r=c_r^*$ and $\nu_r=c_r^0$, we get
$\forall r\in [k]$,
$
\|c_r^0-c_r^*\|\le \frac{\sqrt{f(\alpha)}}{2}\sqrt{\frac{\phi_r^*}{n_r^*}}
$
with probability at least
$
1-m_o\exp(-2(\frac{f(\alpha)}{4}-1)^2w_{\min}^2)-k \exp (-m_op^*_{\min})
$.
Summing over all $r$, the previous event implies
$
\sum_{r}n_r^*\|c_r^0-c_r^*\|^2\le \frac{f(\alpha)}{4}\phi^*
\le 
 \frac{1}{2}\frac{f(\alpha)^2}{16^2}\phi^*
$, where the last inequality is by the assumption that
$f\ge 64^2$ in (B).
\end{proof}
\begin{lm}
\label{lm:augmented}
Assume the conditions Lemma \ref{lm:adapt_tang} hold. For any $\xi>0$, if 
in addition, 
$$
f(\alpha)\ge 5\sqrt{\frac{1}{2w_{\min}}\ln (\frac{2}{\xi p^*_{\min}}\ln\frac{2k}{\xi})}
$$
If we obtain seeds from Algorithm \ref{alg:seeding} choosing 
$$
\frac{\ln \frac{2k}{\xi}}{p^*_{\min}}<m_o<\frac{\xi}{2}\exp\{2(\frac{f(\alpha)}{4}-1)^2w_{\min}^2\}
$$ 
Then
$\Delta(C^0,C^*)\le \frac{1}{2}\frac{f(\alpha)^2}{16^2}\phi^*$
with probability at least
$
1-\xi
$.
\end{lm}
\begin{proof}
By Lemma \ref{lm:adapt_tang},
a sufficient condition for the success probability to be at least $1-\xi$ is:
$$
m_o\exp(-2(\frac{f(\alpha)}{4}-1)^2w^2_{\min})\le \frac{\xi}{2}
$$ and 
$$
k\exp(-m_o p^*_{\min})\le \frac{\xi}{2}
$$ 
This translates to requiring
$$
\frac{1}{p^*_{\min}}\ln \frac{2k}{\xi}\le m_o\le  \frac{\xi}{2}\exp(2(\frac{f(\alpha)}{4}-1)^2w^2_{\min})
$$
Note for this inequality to be possible, we also need
$
\frac{1}{p^*_{\min}}\ln \frac{2k}{\xi}\le  \frac{\xi}{2}\exp(2(\frac{f(\alpha)}{4}-1)^2w^2_{\min})
$, 
imposing a constraint on $f(\alpha)$. 
Taking logarithm on both sides and rearrange, we get
$$
(\frac{f(\alpha)}{4}-1)^2\ge \frac{1}{2w_{\min}}\ln (\frac{2}{\xi p^*_{\min}}\ln\frac{2k}{\xi})
$$
This satisfied since
$
f(\alpha)\ge 5\sqrt{\frac{1}{2w_{\min}}\ln (\frac{2}{\xi p^*_{\min}}\ln\frac{2k}{\xi})}
$.
\end{proof}
\begin{proof}[\textbf{Proof of Theorem \ref{thm:solution}}]
By Proposition \ref{prop:geom},
$(X, C^{*})$ satisfying (B) implies $C^*$ is $(\frac{f(\alpha)^2}{16^2},\alpha)$-stable.
Let $b_0:=\frac{f(\alpha)^2}{16^2}$,
and we denote event
$
F:=\{\Delta(C^0,C^{opt})\le \frac{1}{2}b_0\phi^{*}\}
$.
Since 
$
f(\alpha)\ge 5\sqrt{\frac{1}{2w_{\min}}\ln (\frac{2}{\xi p^*_{\min}}\ln\frac{2k}{\xi})}
$, 
and 
$
\frac{\log\frac{2k}{\xi}}{p^*_{\min}}
<m_o
<\frac{\xi}{2}\exp\{2(\frac{f(\alpha)}{4}-1)^2w_{\min}^2\}
$, we can apply
Lemma \ref{lm:augmented} to get
$$
Pr\{F\}\ge 1-\xi
$$ 
Conditioning on $F$, we can invoke Theorem \ref{thm:mbkm_local},
since (A1) is satisfied implicitly by (B), $C^o\subset conv(X)$ by the sampling method used in Algorithm \ref{alg:seeding}, and we can guarantee that the setting of our parameters, 
$m$, $c^{\prime}$, and $t_o$, satisfies the condition required in Theorem \ref{thm:mbkm_local}.
Let $\Omega_t$ be as defined in the main paper, by Theorem \ref{thm:mbkm_local}, 
$\forall t\ge 1$,
$$
E\{\Delta^t|\Omega_t,F\}=O(\frac{1}{t})
\mbox{~~and~~} 
Pr\{\Omega_t|F\}\ge 1-\delta
$$
So
$$
Pr\{\Omega_t\cap F\}
=Pr\{\Omega_t|F\}Pr\{F\}
\ge (1-\delta)(1-\xi)
$$
Finally, using Lemma \ref{lm:kmdist_cdist}, and letting $G_t:=\Omega_t\cap F$, we get
the desired result.
\end{proof}
\section{Appendix D: auxiliary lemmas}
\paragraph{Equivalence of Algorithm \ref{alg:MBKM} to stochastic $k$-means}
Here, we formally show that Algorithm~\ref{alg:MBKM} with specific instantiation of sample size $m$ and learning rates $\eta^t_r$ is equivalent to online k-means \cite{BottouBengio} and mini-batch k-means \cite{Sculley}. 
\begin{claim}\label{claim:equivAlg}
In Algorithm~\ref{alg:MBKM}, if we set a counter for $\hat{N}^t_r:=\sum_{i=1}^t\hat{n}_r^i$ and
if we set the learning rate 
$\eta_r^t:=\frac{\hat{n}_r^t}{\hat{N}^t_r}$,
then provided the same random sampling scheme is used,
\begin{enumerate}
\item
When mini-batch size $m=1$, the update of Algorithm~\ref{alg:MBKM} is equivalent to that described in [Section 3.3, \cite{BottouBengio}].
\item
When $m>1$, the update of Algorithm~\ref{alg:MBKM} is equivalent to that described from line 3 to line 14 in [Algorithm 1, \cite{Sculley}] with mini-batch size $m$.
\end{enumerate}
\end{claim}
\begin{proof}
For the first claim, we first re-define the variables used in [Section 3.3, \cite{BottouBengio}]. We substitute index $k$ in \cite{BottouBengio} with $r$ used in Algorithm~\ref{alg:MBKM}. For any iteration $t$, we define the equivalence of definitions: $s \leftarrow x_i$, 
$c^t_r \leftarrow w_k$, $\hat{n}^t_r \leftarrow \Delta n_k$, $\hat{N}^t_r \leftarrow n_k$. 
According to the update rule in \cite{BottouBengio}, $\Delta n_k = 1$ if the sampled point $x_i$ is assigned to cluster with center $w_k$. Therefore, the update of the k-th centroid according to online $k$-means in \cite{BottouBengio} is:
$$
w_k\leftarrow w_k + \frac{1}{n_k}(x_i - w_k)1_{\{\Delta n_k =1\}}
$$
Using the re-defined variables, at iteration $t$, this is equivalent to 
$$
c^{t}_r = c^{t-1}_r + \frac{1}{\hat{N}^{t}_r}(s - c^{t-1}_r)1_{\{\hat{n}^t_r=1\}}
$$
Now the update defined by Algorithm~\ref{alg:MBKM} with $m=1$ and $\eta_r^t=\frac{\hat{n}_r^t}{\hat{N}_r^t}$ is:
\begin{align*}
c^{t}_r &= c^{t-1}_r + \eta_r^t(\hat{c}_r^{t} - c^{t-1}_r)1_{\{\hat{n}^t_r\ne 0\}}\\
&=c^{t-1}_r + \frac{\hat{n}_r^t}{\hat{N}_r^t}(s - c^{t-1}_r)1_{\{\hat{n}^t_r=1\}}\\
&=c^{t-1}_r + \frac{1}{\hat{N}_r^t}(s - c^{t-1}_r)1_{\{\hat{n}^t_r=1\}}
\end{align*}
since $\hat{n}_r^t$ can only take value from $\{0,1\}$. This completes the first claim.

For the second claim, consider line 4 to line 14 in [Algorithm 1, \cite{Sculley}]. We substitute their index of time $i$ with $t$ in Algorithm~\ref{alg:MBKM}. We define the equivalence of definitions: $m\leftarrow b$, $S^t\leftarrow M$, $s \leftarrow x$, $c^{t-1}_{I(s)}\leftarrow d[x]$, $c^{t-1}_r\leftarrow c$. 

At iteration $t$, we let $v[c_r^{t-1}]_t$ denote the value of counter $v[c]$ upon completion of the loop from line 9 to line 14 for each center $c$, then $\hat{N}^t_{r}\leftarrow v[c_r^{t-1}]_t$. Since according to Lemma~\ref{techlm:avgAlg}, from line 9 to line 14, the updated centroid $c_r^{t}$ after iteration $t$ is
\begin{eqnarray*}
c_r^t = \frac{1}{v[c_r^{t-1}]_t}\sum_{s\in \cup_{i=1}^t S_r^{i}}s
= \frac{1}{\hat{N}^t_{r}}\sum_{s\in \cup_{i=1}^t S_r^{i}}s
\end{eqnarray*} 
This implies
\begin{eqnarray*}
c^t_r-c^{t-1}_r = \frac{1}{\hat{N}^t_r}\sum_{s\in\cup_{i=1}^t S_r^{i}}s -c^{t-1}_r\\
=\frac{1}{\hat{N}^t_r}[\sum_{s\in S^t_r}s+\sum_{s^{\prime}\in \cup_{i=1}^{t-1}S_r^{i}}s^{\prime} ]-c^{t-1}_r\\
=\frac{1}{\hat{N}^t_r}[\sum_{s\in S^t_r}s+\hat{N}^{t-1}_{r}c_r^{t-1}]-c^{t-1}_r\\
=-\frac{\hat{n}^t_r}{\hat{N}^{t}_r}c^{t-1}_r 
+ \frac{\hat{n}^t_r}{\hat{N}^{t}_r}\frac{\sum_{s\in S^t_r}s}{\hat{n}^t_r}
=-\eta_r^t c_r^{t-1}+\eta_r^t\hat{c}_r^t
\end{eqnarray*}
Hence, the updates in Algorithm~\ref{alg:MBKM} and line 4 to line 14 in [Algorithm 1, \cite{Sculley}] are equivalent.
\end{proof}
\begin{lm}[Centroidal property, Lemma 2.1 of \cite{kanungo}]
\label{centroidal}
For any point set $Y$ and any point $c$ in $\mathbb{R}^d$, 
$$
\sum_{x\in Y}\|x-c\|^2= \sum_{x\in Y}\|x-m(Y)\|^2+|Y|\|m(Y)-c\|^2
$$
\end{lm}
\begin{lm}\label{techlm:avgAlg}
Let $w_t, g_t$ denote vectors of dimension $\mathbb{R}^{d}$ at time $t$. If we choose $w_0 $ arbitrarily, and for $t=1\dots T$, we repeatdly apply the following update
$$
w_t = (1-\frac{1}{t})w_{t-1}+\frac{1}{t}g_t
$$
Then
$$
w_T = \frac{1}{T}\sum_{t=1}^{T}g_t
$$ 
\end{lm}
\begin{proof}
We prove by induction on $T$. 
For $T=1$, $w_1 = (1-1)w_0+g_1 = \frac{1}{1}\sum_{t=1}^{1}g_t$. So the claim holds for $T=1$.

Suppose the claim holds for $T$, then for $T+1$, by the update rule
\begin{eqnarray*}
w_{T+1}=(1-\frac{1}{T+1})w_T + \frac{1}{T+1}g_{T+1}\\
=(1-\frac{1}{T+1}) \frac{1}{T}\sum_{t=1}^{T}g_t + \frac{1}{T+1}g_{T+1}\\
=\frac{T}{T+1}\frac{1}{T}\sum_{t=1}^{T}g_t+ \frac{1}{T+1}g_{T+1}\\
=\frac{1}{T+1}\sum_{t=1}^{T+1}g_t
\end{eqnarray*}
So the claim holds for any $T\ge 1$.
\end{proof}
\begin{lm}
\label{lm:boundB1}
$\forall t\ge 1$, conditioning on $F_t$, the noise term
\eqref{noise1} is upper bounded by
$B_1:=5\phi^t$. 
\end{lm}
\begin{proof}
Since
$$
\|x-\hat{c}_r^{t+1}\|^2
\le 2\|x-c_r^{t}\|^2+2\|c_r^{t}-\hat{c}_r^{t+1}\|^2
$$
We have
\begin{eqnarray*}
E[\sum_r\sum_{x\in A_r^{t+1}}\|x-\hat{c}_r^{t+1}\|^2+\phi^t|F_t]\\
\le
2\sum_r\sum_{x\in A_r^{t+1}}\|x-c_r^{t}\|^2\\
+2\sum_r\sum_{x\in A_r^{t+1}}E[\|c_r^{t}-\hat{c}_r^{t+1}\|^2|F_t]
+\phi^t
\end{eqnarray*}
Now,
$$
E[\|c_r^{t}-\hat{c}_r^{t+1}\|^2|F_t]
\le E\frac{\sum_{s\in S_r^t}\|c_r^t-s\|^2}{|S_r^t|}
=\frac{\phi_r^t}{n_r^t}
$$
where $S_r^t$ is the sampled from $A_r^t$ in Algorithm \ref{alg:MBKM}, and the inequality is by convexity of $l_2$-norm. 
Substituting this into the previous inequality completes the proof.
\end{proof}
\begin{lm}
\label{lm:boundB}
Suppose $C^*$ is $(b_o,\alpha)$-stable.
Conditioning on $\Omega_i$, we have,
The terms 
\eqref{noise2}, and \eqref{noise3}, for $t=i$, are upper bounded by
$B:=4(b_o+1)n\phi^*$. 
\end{lm}
\begin{proof}
Conditioning on $\Omega_i$,
$$
\Delta^{i-1}\le b_o\phi^*
$$
By Lemma \ref{lm:kmdist_cdist}, we also have
$$
\phi^{i-1}-\phi^*\le \Delta^{i-1}\le b_o\phi^*
$$ 
By Cauchy-Schwarz,
\begin{eqnarray*}
\sum_r n_r^*\langle c_r^{i-1}-c_r^*,\hat{c}_r^i-c_r^{i-1}\rangle\\
\le
\sqrt{\sum_rn_r^*\|c_r^{i-1}-c_r^*\|^2}\sqrt{\sum_rn_r^*\|\hat{c}_r^i-c_r^{i-1}\|^2}
\end{eqnarray*}
Now, since 
$\hat{c}_r^i$ is the mean of a subset of $A_r^{i}$,
$$
\|\hat{c}_r^i-c_r^{i-1}\|^2\le \phi_r^{i-1}
$$
Hence
$$
\sum_rn_r^*\|\hat{c}_r^i-c_r^{i-1}\|^2
\le
n\phi^{i-1}
$$
On the other hand,
\begin{eqnarray*}
\sum_rn_r^*\|c_r^{i-1}-E[\hat{c}_r^i|F_{i-1}]\|^2
=\sum_rn_r^*\|c_r^{i-1}-m(A_r^i)\|^2 \\
\le n\sum_r\phi(c_r^{i-1})-\phi(m(A_r^i))\\
= n [\phi^{i-1}-\phi(m(A^i))]
\le n(\phi^{i-1}-\phi^*)
\end{eqnarray*}
Now we first bound \eqref{noise2}:
\begin{eqnarray*}
\sum_rn_r^*\langle c_r^{i-1}-c_r^*,\hat{c}_r^i-E[\hat{c}_r^i|F_{i-1}]\rangle\\
=\sum_r n_r^*\langle c_r^{i-1}-c_r^*,\hat{c}_r^i-c_r^{i-1}\rangle\\
+\sum_r n_r^*\langle c_r^{i-1}-c_r^*,c_r^{i-1}-E[\hat{c}_r^i|F_{i-1}]\rangle\\
\le
\sqrt{\Delta^{i-1}}\sqrt{n\phi^{i-1}}+\sqrt{\Delta^{i-1}}\sqrt{n(\phi^{i-1}-\phi^*)}\\
\le
\sqrt{b_o\phi^*}\sqrt{n(b_o+1)\phi^*}+\sqrt{n}b_o\phi^*
\le
2(b_o+1)\sqrt{n}\phi^*
\end{eqnarray*}
To bound \eqref{noise3},
\begin{eqnarray*}
\sum_rn_r^*\|\hat{c}_r^i-c_r^*\|^2\\
\le 2\sum_rn_r^*\|\hat{c}_r^i-c_r^{i-1}\|^2+2\sum_rn_r^*\|c_r^{i-1}-c_r^*\|^2\\
\le
2n\phi^{i-1} + 2\Delta^{i-1}
\le 2n(b_o+1)\phi^*+2b_o\phi^*
\le 4n(b_o+1)\phi^*
\end{eqnarray*}
\end{proof}
\begin{claim}
\label{claim:conv_hull}
In the context of Algorithm \ref{alg:MBKM}, if $\forall c_r^t\in C^t, c_r^t\in conv(X)$, 
then $\forall c_r^{t+1}\in C^{t+1}$, $c_r^{t+1}\in conv(X)$.
\end{claim}
\begin{proof}[Proof of Claim]
By the update rule in Algorithm \ref{alg:MBKM},
$c_r^{t+1}$ is a convex combination of $c^t_r$ and $\hat{c}_r^{t+1}$,
where $\hat{c}_r^{t+1}$ is the mean of a subset of $X$, and hence $\hat{c}_r^{t+1}\in conv(X)$.
Since both $c_r^t$ and $\hat{c}_r^{t+1}$ are in $conv(X)$, 
$c_r^{t+1}\in conv(X)$.
\end{proof}
\begin{lm}[technical lemma]
\label{lm:tech}
For any fixed $b\in (1,2]$.
If $C\ge\frac{b-1}{3},\delta\le\frac{1}{e}$, and $t\ge (\frac{3C}{b-1}\ln\frac{1}{\delta})^{\frac{2}{b-1}}$, then
$
t^{b-1}-2C\ln t - C\ln\frac{1}{\delta}> 0
$. 
\end{lm}
\begin{proof}
Let $f(t):=t^{b-1}-2C\ln t - C\ln\frac{1}{\delta}$.
Taking derivative, we get
$f^{\prime}(t)=(b-1)t^{b-2}-\frac{2C}{t}\ge0$ when $t\ge (\frac{2C}{b-1})^{\frac{1}{b-1}}$. 
Since $\ln\frac{1}{\delta}\frac{3C}{b-1}\ge\frac{3C}{b-1}\ge1$,
$(\ln\frac{1}{\delta}\frac{3C}{b-1})^{\frac{2}{b-1}}\ge(\frac{2C}{b-1})^{\frac{1}{b-1}}$,
it suffices to show $f((\ln\frac{1}{\delta}\frac{3C}{b-1})^{\frac{2}{b-1}})> 0$ for our statement to hold.
$f((\ln\frac{1}{\delta}\frac{3C}{b-1})^{\frac{2}{b-1}})
=(\ln\frac{1}{\delta}\frac{3C}{b-1})^{2}
-2C\ln\{(\ln\frac{1}{\delta}\frac{3C}{b-1})^{\frac{2}{b-1}}\}-C\ln\frac{1}{\delta}
=(\ln\frac{1}{\delta})^{2}\frac{9C^2}{(b-1)^2}
-\frac{4C}{b-1}\ln(\ln\frac{1}{\delta}\frac{3C}{b-1})
-C\ln\frac{1}{\delta}
=\frac{4C}{b-1}[\frac{\frac{3}{2}C}{b-1}\ln\frac{1}{\delta}
-\ln(\frac{3C}{b-1}\ln\frac{1}{\delta})]
+C\ln\frac{1}{\delta}[\frac{3C}{(b-1)^2}-1]
>0
$,
where the first term is greater than zero because $x-\ln(2x)>0$ for $x>0$,
and the second term is greater than zero by our assumption on $C$.
\end{proof}
\begin{lm}[Lemma D1 of \cite{balsubramani13}]
\label{lm:tech_2}
Consider a nonnegative sequence $(u_t: t\ge t_o)$, such that for some constants $a,b>0$
and for all $t>t_{o}\ge 0$,
$u_t\le (1-\frac{a}{t})u_{t-1}+\frac{b}{t^2}$.
Then, if $a>1$,
$$
u_t\le (\frac{t_o+1}{t+1})^a u_{t_o} + \frac{b}{a-1}(1+\frac{1}{t_o+1})^{a+1}\frac{1}{t+1}
$$
\end{lm}

\end{document}